\documentclass{article}

\def\includehome{./include}
\def\bibhome{./include}
\def\fighome{./figures}

\usepackage{dsfont}
\usepackage{xspace}

\usepackage{tikz}
\usetikzlibrary{fit}
\usetikzlibrary{calc,shapes}
\usetikzlibrary{decorations.pathmorphing} 
\usetikzlibrary{fit}					
\usetikzlibrary{backgrounds}	
\usepackage{wrapfig}

\usepackage{stmaryrd}
\usepackage[english]{babel}
\usepackage{diagbox}

\usepackage[utf8]{inputenc}
\usepackage[T1]{fontenc}    
\usepackage{pgfplots}
\pgfplotsset{compat=newest}
\usepgfplotslibrary{groupplots}
\usepgfplotslibrary{dateplot}

\usepackage{hyperref}       
\usepackage{url}            
\usepackage{booktabs}       
\usepackage{multirow}
\usepackage{array}
\usepackage{makecell}
\usepackage{nicefrac}       
\usepackage{microtype}      
\usepackage{lipsum}
\usepackage{graphicx}
\usepackage{float}
\usepackage{nomencl}
\usepackage{enumerate}
\usepackage{enumitem}

\usepackage{xcolor}
\usepackage{color}

\usepackage{tablefootnote}
\usepackage{subcaption}
\usepackage[font=normal,labelfont=bf]{caption}

\usepackage{amsfonts}       
\usepackage{amsthm}
\usepackage{amssymb}
\usepackage{amsmath}
\usepackage{bbm}
\usepackage{bm}

\usepackage{mathtools}
\usepackage{mathrsfs}

\usepackage{soul}
\usepackage[capitalize]{cleveref}

\usepackage{cite}

\newtheorem{theorem}{Theorem}
\newtheorem{lemma}[theorem]{Lemma}
\newtheorem{definition}[theorem]{Definition}

\newcommand{\mytensor}[1]{\mathcal{#1}}
\newcommand{\mymatrix}[1]{\bm{#1}}
\newcommand{\myvector}[1]{\bm{#1}}

\newcommand{\scalarSup}[2]{{#1}^{(#2)}}
\newcommand{\scalarSub}[2]{#1_{#2}}

\newcommand{\vectorSup}[2]{\bm{#1}^{(#2)}}
\newcommand{\vectorSub}[2]{#1_{#2}}	
\newcommand{\vectorInd}[3]{#1^{(#2)}_{#3}}

\newcommand{\matrixSub}[2]{#1_{#2}}	
\newcommand{\matrixInd}[3]{#1^{(#2)}_{#3}}

\newcommand{\tensorSup}[2]{\mathcal{#1}^{(#2)}}
\newcommand{\tensorSub}[2]{\mathcal{#1}_{#2}}
\newcommand{\tensorInd}[3]{\mathcal{#1}^{(#2)}_{#3}}

\newcommand{\sumIndex}[2]{\sum_{#1 = 1}^{#2}}

\newcommand{\loss}{\mathcal{L}}
\newcommand{\gradient}[1]{\frac{\partial \loss}{\partial {#1}}}
\newcommand{\derivative}[2]{\frac{\partial {#1}}{\partial {#2}}}
\newcommand{\gradientInline}[1]{{\partial \loss} / {\partial {#1}}}
\newcommand{\derivativeInline}[2]{{\partial {#1}} / {\partial {#2}}}

\newcommand{\rootunity}[2]{\omega_{#1}^{#2}} 
\newcommand{\Fourier}[1]{\widetilde{#1}}
\newcommand{\adjoint}[1]{#1^\top}
\newcommand{\inverse}[1]{\overline{#1}}

\def\P{\mathbb{P}} 
\def\E{\mathbb{E}} 
\def\V{\mathbb{V}} 

\def\R{\mathbb{R}} 
\def\C{\mathbb{C}} 
\def\Z{\mathbb{Z}} 

\def\conv{\ast}
\newcommand{\im}{\mathsf{j}} 

\newcommand{\ARMA}{ARMA\xspace}

\newcommand{\ARMAlong}{autoregressive-moving-average\xspace}

\newcommand{\MA}{MA\xspace}

\newcommand{\MAlong}{moving-average\xspace}

\newcommand{\AR}{AR\xspace}
\newcommand{\ARLong}{Autoregressive\xspace}
\newcommand{\ARlong}{autoregressive\xspace}

\newcommand{\DTFT}{DTFT\xspace}

\newcommand{\DTFTlong}{discrete-time Fourier transform\xspace}

\newcommand{\DFT}{DFT\xspace}
\newcommand{\DFTLONG}{Discrete Fourier Transform\xspace}

\newcommand{\DFTlong}{discrete Fourier transform\xspace}

\newcommand{\IDFT}{IDFT\xspace}

\newcommand{\FFT}{FFT\xspace}
\newcommand{\FFTLONG}{Fast Fourier Transform\xspace}

\newcommand{\NNlong}{neural network\xspace}

\newcommand{\CNN}{CNN\xspace}

\newcommand{\CNNLong}{Convolutional neural network\xspace}

\newcommand{\ERF}{ERF\xspace}
\newcommand{\ERFLONG}{Effective Receptive Field\xspace}
\newcommand{\ERFLong}{Effective receptive field\xspace}
\newcommand{\ERFlong}{effective receptive field\xspace}

\newcommand{\BIBO}{BIBO\xspace}
\newcommand{\BIBOLONG}{Bounded-Input Bounded-Output\xspace}

\newcommand{\ROC}{ROC\xspace}
\newcommand{\ROCLONG}{Region of Convergence\xspace}

\newcommand{\MGF}{MGF\xspace}
\newcommand{\MGFLONG}{Moment Generating Function\xspace}

\newcommand{\MGFlong}{moment generating function\xspace}

\newcommand{\mytitle}{ARMA Nets: \\ Expanding Receptive Field for Dense Prediction}

     \PassOptionsToPackage{numbers, compress}{natbib}

\usepackage[preprint]{\includehome/neurips_2020}

\title{\mytitle}

\author{
	\begin{tabular}{c} Jiahao Su$^{1}$ \\ \texttt{jiahaosu@umd.edu} \end{tabular} \begin{tabular}{c} Shiqi Wang$^{2}$ \\ \texttt{161170041@smail.nju.edu.cn} \end{tabular} \begin{tabular}{c} Furong Huang$^{1}$ \\ \texttt{furongh@cs.umd.edu} \end{tabular} \\
  	$^{1}$University of Maryland, College Park, MD USA \hskip0.5em $^{2}$Nanjing University, Nanjing, China \\
}

\begin{document}
\maketitle

\begin{abstract}
Global information is essential for dense prediction problems,
whose goal is to compute a discrete or continuous label for each pixel in the images. 
Traditional convolutional layers in neural networks, initially designed for image classification, are restrictive in these problems since the filter size limits their receptive fields.
In this work, we propose to replace any traditional convolutional layer 
with an autoregressive moving-average (ARMA) layer, 
a novel module with an adjustable receptive field controlled by the learnable autoregressive coefficients.
Compared with traditional convolutional layers, 
our ARMA layer enables explicit interconnections of the output neurons, 
and learns its receptive field by adapting the autoregressive coefficients of the interconnections.
ARMA layer is adjustable to different types of tasks: 
for tasks where global information is crucial, 
it is capable of learning relatively large autoregressive coefficients 
to allow for an output neuron's receptive field covering the entire input; 
for tasks where only local information is required, 
It can learn small or near zero autoregressive coefficients 
and automatically reduces to a traditional convolutional layer.
We show both theoretically and empirically that the effective receptive field of 
networks with ARMA layers (named as ARMA networks) expands 
with larger autoregressive coefficients.
We also provably solve the instability problem of learning and prediction in the ARMA layer
through a re-parameterization mechanism. 
Additionally, we demonstrate that ARMA networks 
substantially improve their baselines on challenging dense prediction tasks 
including video prediction and semantic segmentation.
\end{abstract}

\section{Introduction}
\label{sec:definitions}

Convolutional layers in neural networks 
have many successful applications for machine learning tasks.
Each output neuron encodes an input region of the network measured by the
{\em \ERFlong} (\ERF)~\citep{luo2016understanding}.
A large \ERF that allows for sufficient global information is needed to make accurate predictions; 
however, a simple stack of convolutional layers does not effectively expand \ERF.
{\CNNLong}s ({\CNN}s) typically encode global information by adding 
downsampling (pooling) layers, which coarsely aggregate global information.
A fully-connected classification layer subsequently reduces the entire feature map to an output label.
Downsampling and fully-connected layers are suitable for image classification tasks 
where only a single prediction is needed.
But they are less effective, due to potential loss of information, in dense prediction tasks 
such as semantic segmentation and video prediction, where each pixel requests a prediction.
Therefore, it is crucial to introduce mechanisms that enlarge \ERF without too much information loss.

Naive approaches to expanding \ERF, such as deepening the network or enlarging the filter size, 
drastically increase the model complexity,
which results in expensive computation, difficulty in optimization, and susceptibility to overfitting.
Recently advanced architectures have been proposed to expand \ERF,
including encoder-decoder structured networks~\citep{ronneberger2015u}, 
dilated convolutional networks~\citep{yu2015multi, yu2017dilated}, 
and non-local attention networks~\citep{wang2018non}. 
However, encoder-decoder structured networks could lose 
high-frequency information due to the downsampling layers.
Dilated convolutional networks could suffer from the gridding effect 
while the \ERF expansion is limited,
and non-local attention networks are expensive in training and inference.

We introduce a novel {\em \ARMAlong} (\ARMA) layer
that enables adaptive receptive field by explicit interconnections among its output neurons.
Our \ARMA layer realizes these interconnections via extra convolutions on output neurons,
on top of the convolutions on input neurons as in a traditional convolutional layer.
We provably show that an \ARMA network can have arbitrarily large \ERF,
thus encoding global information, with minimal extra parameters at each layer.
Consequently, an \ARMA network can flexibly enlarge its \ERF
to leverage global knowledge for dense prediction without reducing spatial resolution.
Moreover, the \ARMA networks are independent of the architectures above
including encoder-decoder structured networks, dilated convolutional networks and non-local attention networks.

A significant challenge in \ARMA networks lies in the complex computations 
needed in both forward and backward propagations — 
simple convolution operations are not applicable 
since the output neurons are influenced by their neighbors and thus interrelated.
Another challenge in ARMA networks is instability — 
the additional interconnections among the output neurons could recursively amplify the outputs 
and lead them to infinity. We address both challenges in this paper. 
   
\textbf{Summary of Contributions}
\begin{itemize}[leftmargin=*, itemsep=0pt, topsep=0pt]
\item We introduce a novel \ARMA layer that is a plug-and-play module 
substituting convolution layers in neural networks to allow flexible tuning of their \ERF, 
adapting to the task requirements and improving performance in dense prediction problems.
\item We recognize and address the problems of {\em computation} and {\em instability} in \ARMA layers.
{\bf (1)} To reduce computational complexity, we develop \FFT-based algorithms for both forward and backward passes;
{\bf (2)} To guarantee stable learning and prediction, we propose a {\em separable \ARMA layer}
and a re-parameterization mechanism that ensures the layer to operate in a stable region.
\item We successfully apply \ARMA layers 
in ConvLSTM network~\citep{xingjian2015convolutional} for pixel-level multi-frame video prediction 
and U-Net model~\citep{ronneberger2015u} for medical image segmentation.
\ARMA networks substantially outperform the corresponding baselines on both tasks, 
suggesting that our proposed \ARMA layer is a general and useful building block for dense prediction problems.
\end{itemize}

\section{Related Works}
\label{sec:related}

{\bf Dilated convolution~\citep{holschneider1990real}} 
enlarges the receptive field by upsampling the filter coefficients with zeros.
 Unlike encoder-decoder structure, dilated convolution preserves the spatial resolution
 and is thus widely used in dense prediction problems,
including semantic segmentation~\citep{long2015fully, yu2015multi, chen2017deeplab}, 
and objection detection~\citep{dai2016r, li2019scale}.
However, dilated convolution by itself creates gridding artifacts 
if its input contains higher frequency than the upsampling rate~\citep{yu2017dilated}, 
and the inconsistency of local information hampers the performance of the dilated convolutional networks~\citep{wang2018smoothed}.
Such artifacts can be alleviated by extra anti-aliasing layer~\citep{yu2017dilated}, group interacting layer~\citep{wang2018smoothed} or spatial pyramid pooling~\citep{chen2017rethinking}.

{\bf Deformable convolution} allows the filter shape (i.e.\ locations of the incoming pixels) to be learnable~\citep{dai2017deformable, jeon2017active, zhu2019deformable}.
While deformable convolution focuses on adjusting the filter {\em shape}, 
our \ARMA layer aims to expand the filter {\em size} adaptively.

{\bf Non-local attention network~\citep{wang2018non}} 
inserts non-local attention blocks between the convolutional layers.
A non-local attention block computes a weighted sum of all input neurons for each output neuron, 
similar to attention mechanism~\citep{vaswani2017attention}.
In practice, non-local attention blocks are computationally expensive, 
thus they are typically inserted in the upper part of the network (with lower resolution). 
In contrast, our \ARMA layers are economical (see \autoref{sec:arma-computation}), 
and can be used throughout the network.

{\bf Encoder-decoder structured network} pairs each downsampling layer with 
another upsampling layer to maintain the resolution, 
and introduces skip-connection between the pair 
to preserve the high-frequency information~\citep{ronneberger2015u, long2015fully}.
Since the shortcut bypasses the downsampling/upsampling layers,
the network has a small receptive field for the high-frequency components.
A potential solution is to augment upsampling with non-local attention block~\citep{oktay2018attention}
or \ARMA layer (\autoref{sec:experiments}).

{\bf Spatial recurrent neural networks} apply recurrent propagations over the spatial domain
~\citep{byeon2015scene, oord2016pixel, kalchbrenner2015grid, stollenga2015parallel, liu2016learning}, 
and learns the affinity between neighboring pixels~\citep{liu2017learning}.
Most of these prior works consider nonlinear recurrent neural networks, 
where the activation between recursions prohibits an efficient \FFT-based algorithm.
In contrast, our proposed \ARMA layer is equivalent to a linear recurrent neural network.
where the spatial recurrences in \ARMA layer can be efficiently evaluated using \FFT.

\section{\ARMA Neural Networks}
\label{sec:arma-nn}

In this section, we introduce a novel {\em \ARMAlong} (\ARMA) layer,
and analyze its ability to expand {\em \ERFLONG} (\ERF) in {\NNlong}s.
The analysis is further verified by visualizing the \ERF
with varying network depth and strength of \ARlong coefficients.

\subsection{\ARMA Layer}
\label{sub:traditional-to-arma}

A traditional convolutional layer
is essentially a {\em \MAlong} model~\citep{box2015time},
$\tensorSub{Y}{\bm{:}, \bm{:}, t} = \sum_{s = 1}^{S}
\tensorSub{W}{\bm{:}, \bm{:}, t, s} \conv \tensorSub{X}{\bm{:}, \bm{:}, s}$,
where the {\em \MAlong coefficients} $\mytensor{W} \in \R^{K_m \times K_m \times T \times S}$ 
is parameterized by a \(4^\text{th}\)-order kernel 
($K_m$ is the filter size, and $S, T$ are input/output channels), 
$\bm{:}$ denotes all elements from the specified coordinate,
and $\conv$ denotes convolution between an input feature and a filter.

\begin{wrapfigure}{r}{0.5\textwidth}
	\begin{minipage}{0.245\textwidth}
		\includegraphics[width=\textwidth]{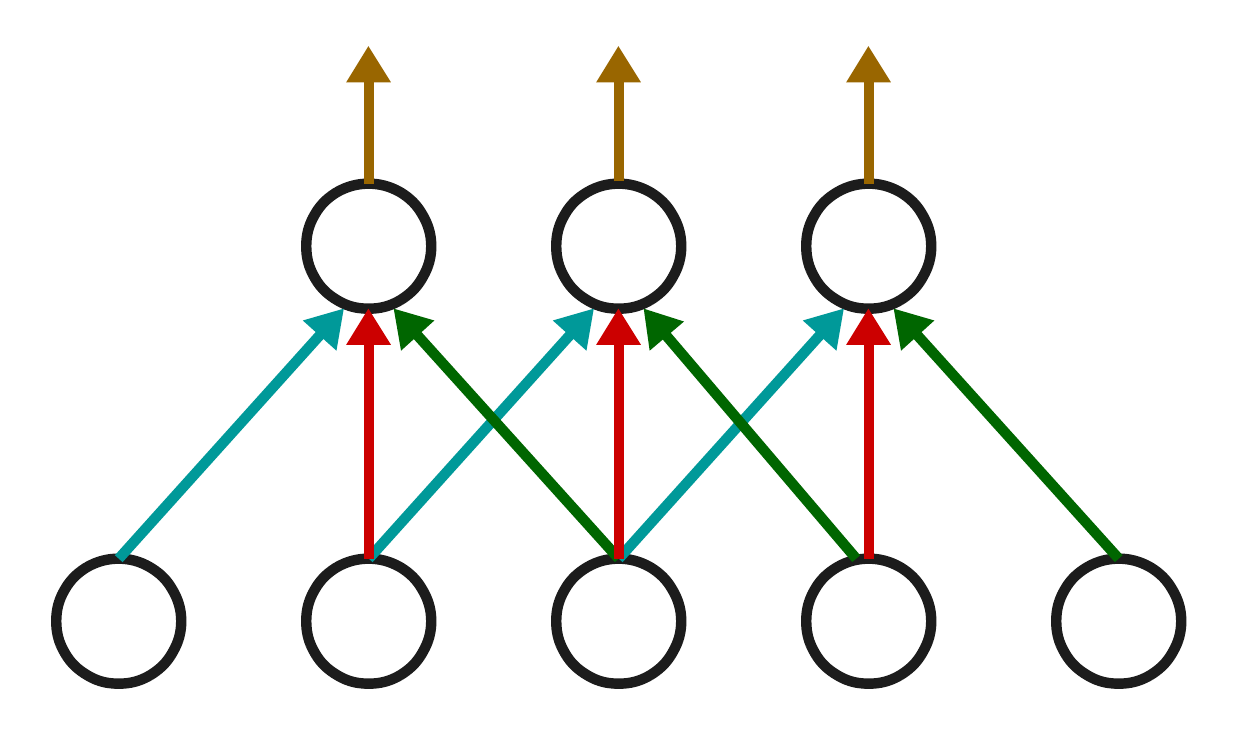}
	\label{fig:conv-1d}
	\captionof*{subfigure}{{ {\bf(a)} Convolution}}
	\end{minipage}
	\hfill
	\begin{minipage}{0.245\textwidth}
		\includegraphics[width=\textwidth]{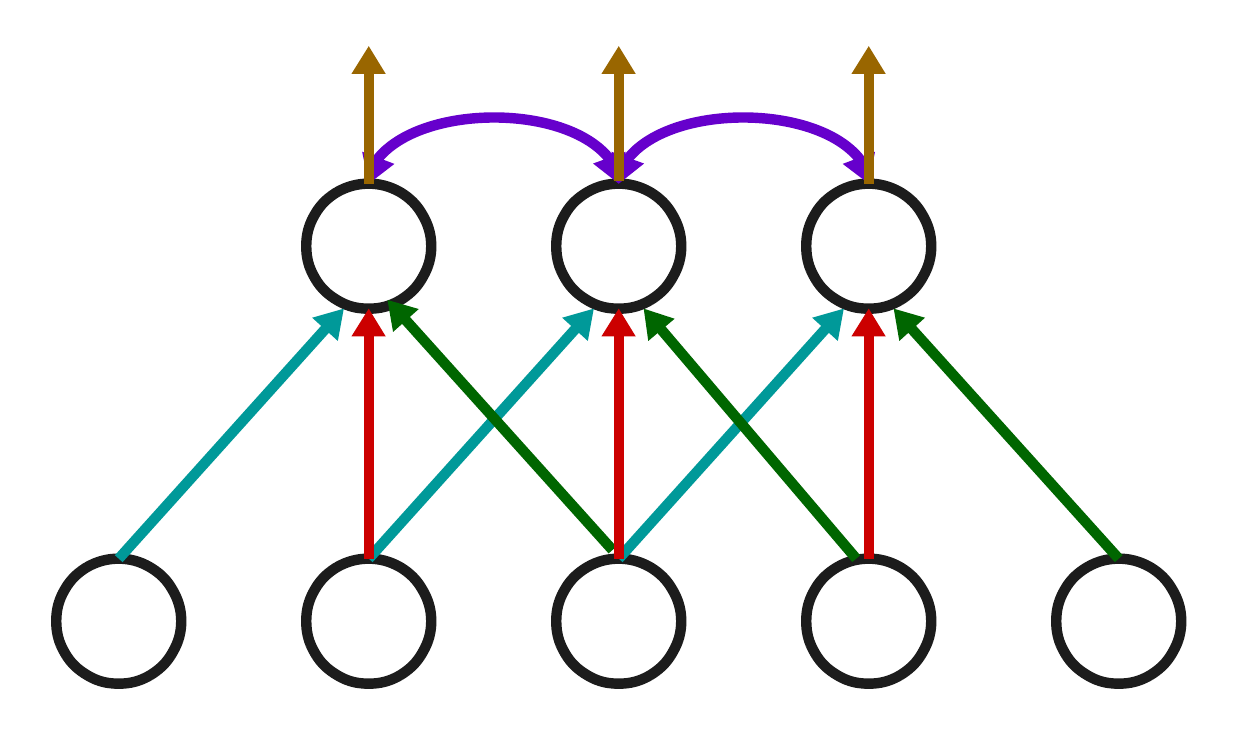}
		\label{fig:arma-1d}
		\captionof*{subfigure}{{\bf(b)} \ARMA}
	\end{minipage}
	\captionof{figure}{The \ARMA layer introduces interconnections among output neurons explicitly.}
	\label{fig:cmp-1d}
\end{wrapfigure}

As motivated in the introduction, we introduce a novel \ARMA layer, 
that enables adaptive receptive field by introducing explicit interconnections among its output neurons 
as illustrated in \autoref{fig:cmp-1d}.
Our \ARMA layer realizes these interconnections by introducing extra convolutions on the outputs, 
in addition to the convolutions on the inputs as in a traditional convolutional layer.
As a result, in an \ARMA layer, each output neuron can be affected by an input pixel faraway
through interconnections among the output neurons, thus receives global information.
Formally, we define \ARMA layer in Definition~\ref{def:arma-cnn}.

\begin{definition} [{\bf \ARMA layer}] 
\label{def:arma-cnn}
An \ARMA layer is parameterized by a moving-average kernel (coefficients)
$\mytensor{W} \in \R^{K_m \times K_m \times S \times T}$ and
an \ARlong kernel (coefficients) $\mytensor{A} \in \R^{K_a \times K_a \times T}$.
It receives an input $\mytensor{X} \in \R^{I_1 \times I_2 \times S}$ and 
returns an output $\mytensor{Y} \in \R^{I^{\prime}_1 \times I^{\prime}_2 \times T}$ 
with an \ARMA model:
{\small
\begingroup
\setlength{\abovedisplayskip}{2pt}
\setlength{\belowdisplayskip}{0pt}
\begin{equation}
\tensorSub{A}{\bm{:}, \bm{:}, t} \conv \tensorSub{Y}{\bm{:}, \bm{:}, t} = 
\sum_{s = 1}^{S} \tensorSub{W}{\bm{:}, \bm{:}, t, s} \conv \tensorSub{X}{\bm{:}, \bm{:}, s}
\label{eq:arma-cnn}
\end{equation}
\endgroup}%
\end{definition}

{\it Remarks:}
{\bf (1)} Since the output interconnections are realized by convolutions, 
the \ARMA layer maintains the {\em shift-invariant} property.
{\bf (2)} The \ARMA layer {\em reduces} to a traditional layer 
if the \ARlong kernel $\mytensor{A}$ represents an identical mapping.
{\bf (3)} The \ARMA layer is a {\em plug-and-play} module that 
can replace {\em any} convolutional layer, adding $K_a^2 T$ extra parameters 
negligible compared to $K_w^2 S T$ parameters in a traditional convolution layer.  
{\bf (4)} Different from traditional layer, computing \autoref{eq:arma-cnn} 
and its backpropagation is nontrivial, 
studied in \autoref{sec:arma-computation}.  

\vspace{-0.5em}
\begin{figure}[!htbp]
\centering
	\begin{subfigure}[b]{0.26\textwidth}
	\centering
		\includegraphics[width=\linewidth]{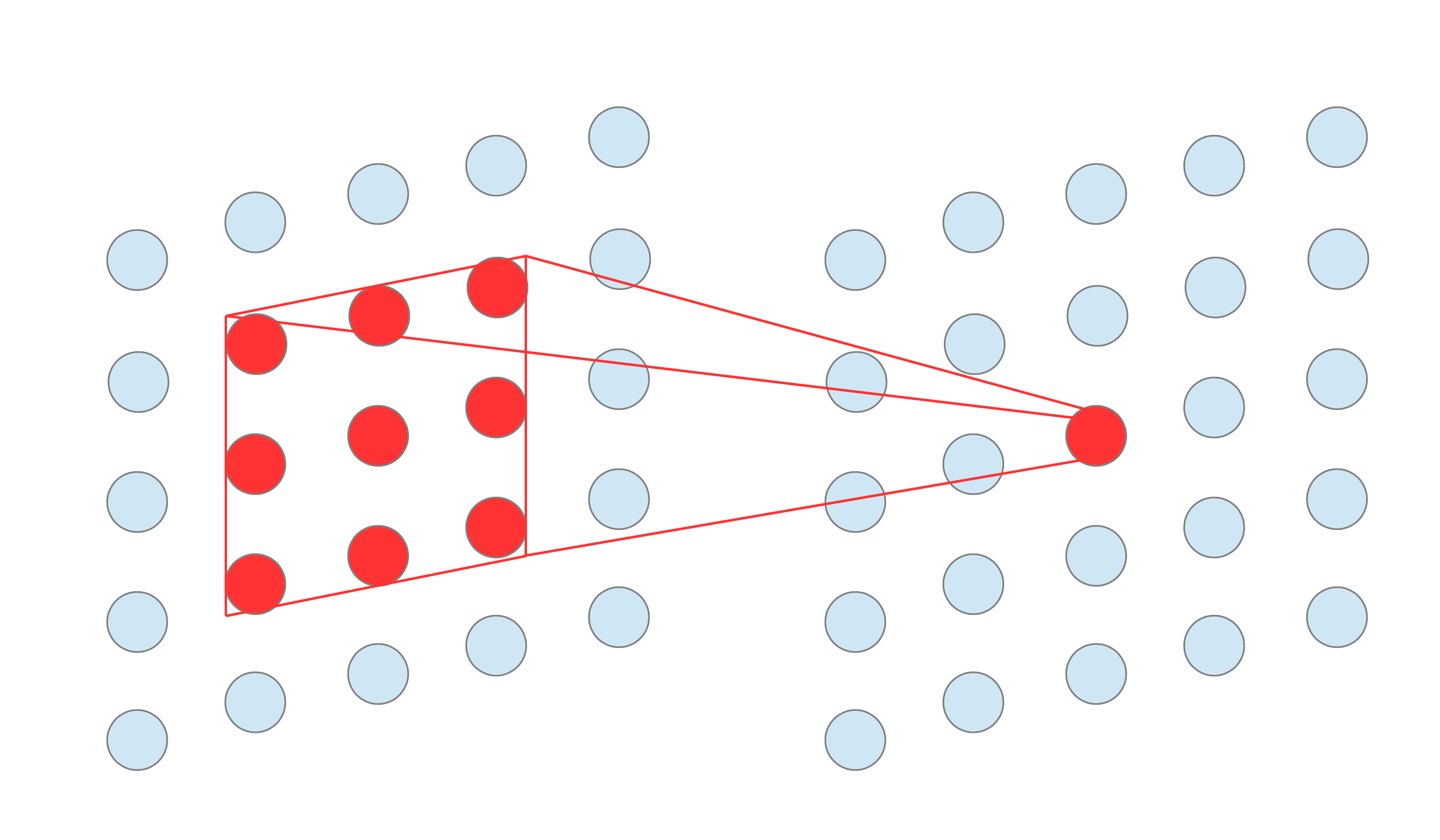}
		\caption{Convolution}
		\label{fig:conv-2d}
	\end{subfigure}
	\hfill
	\begin{subfigure}[b]{0.26\textwidth}
	\centering
		\includegraphics[width=\linewidth]{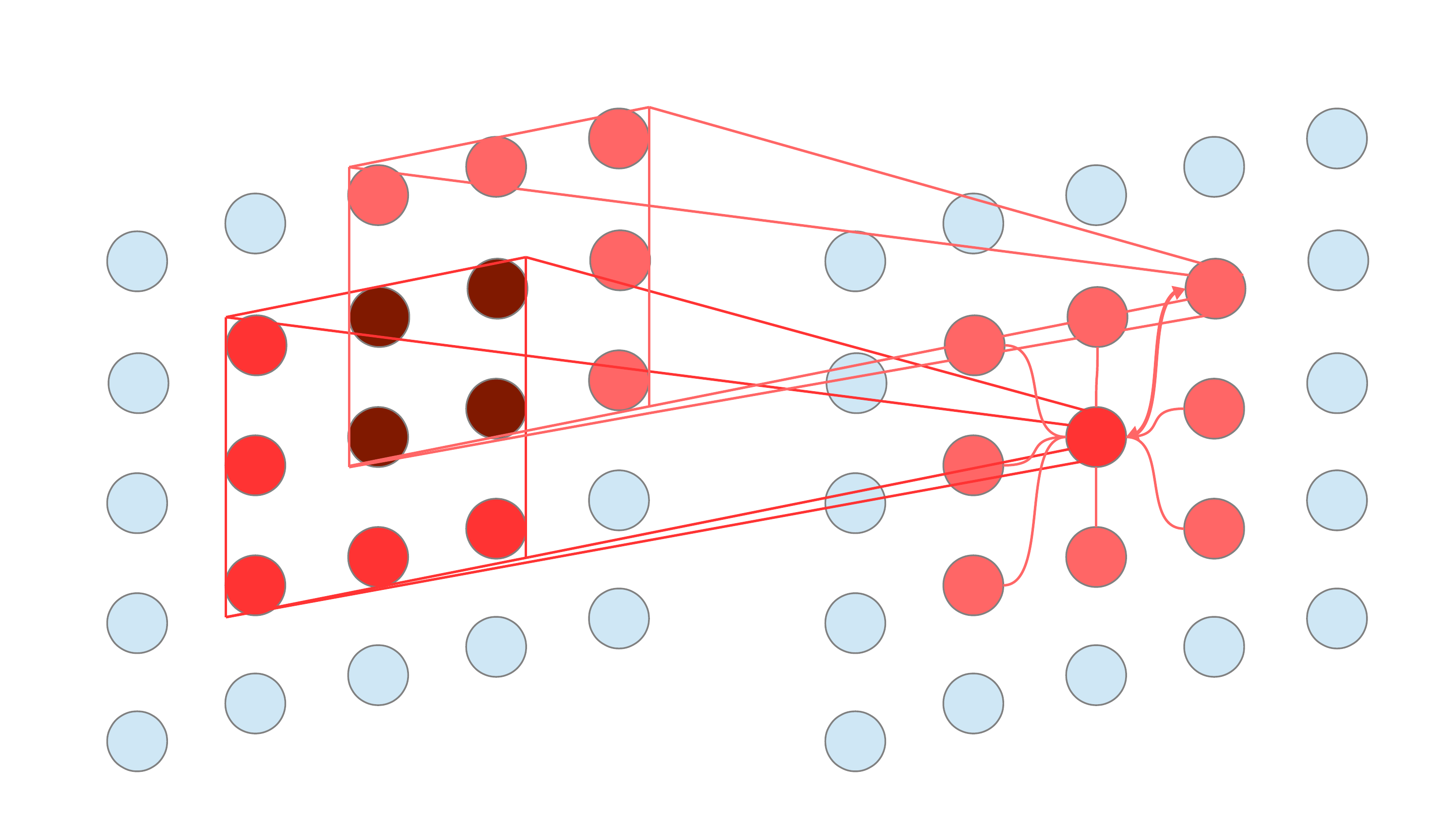}
		\caption{\ARMA}
		\label{fig:arma-2d}
	\end{subfigure}
	\hfill
	\begin{subfigure}[b]{0.23\textwidth}
	\centering
		\includegraphics[width=\linewidth]{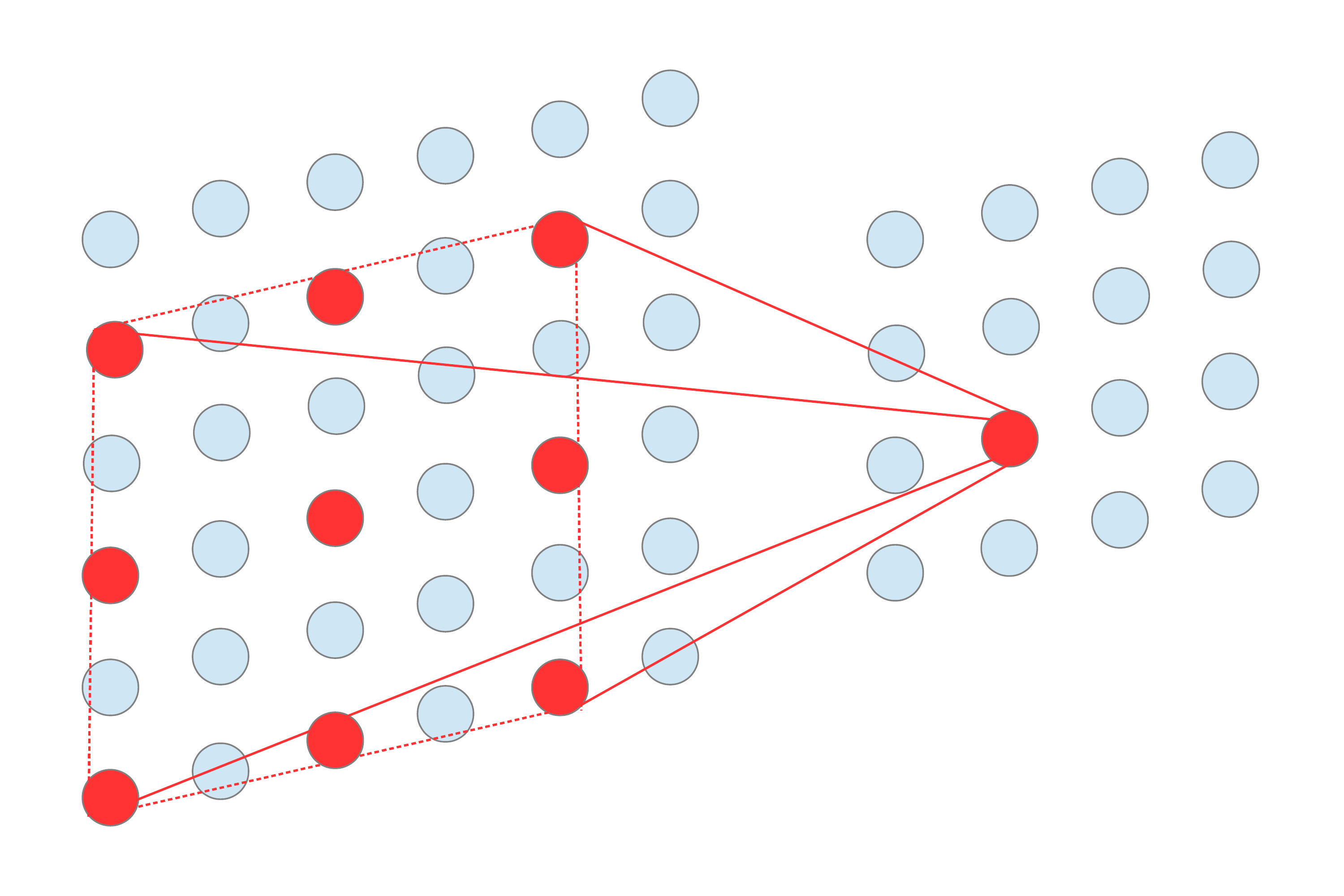}
		\caption{Dilated convolution}
		\label{fig:dilated-conv-2d}
	\end{subfigure}
	\hfill
	\begin{subfigure}[b]{0.23\textwidth}
	\centering
		\includegraphics[width=\linewidth]{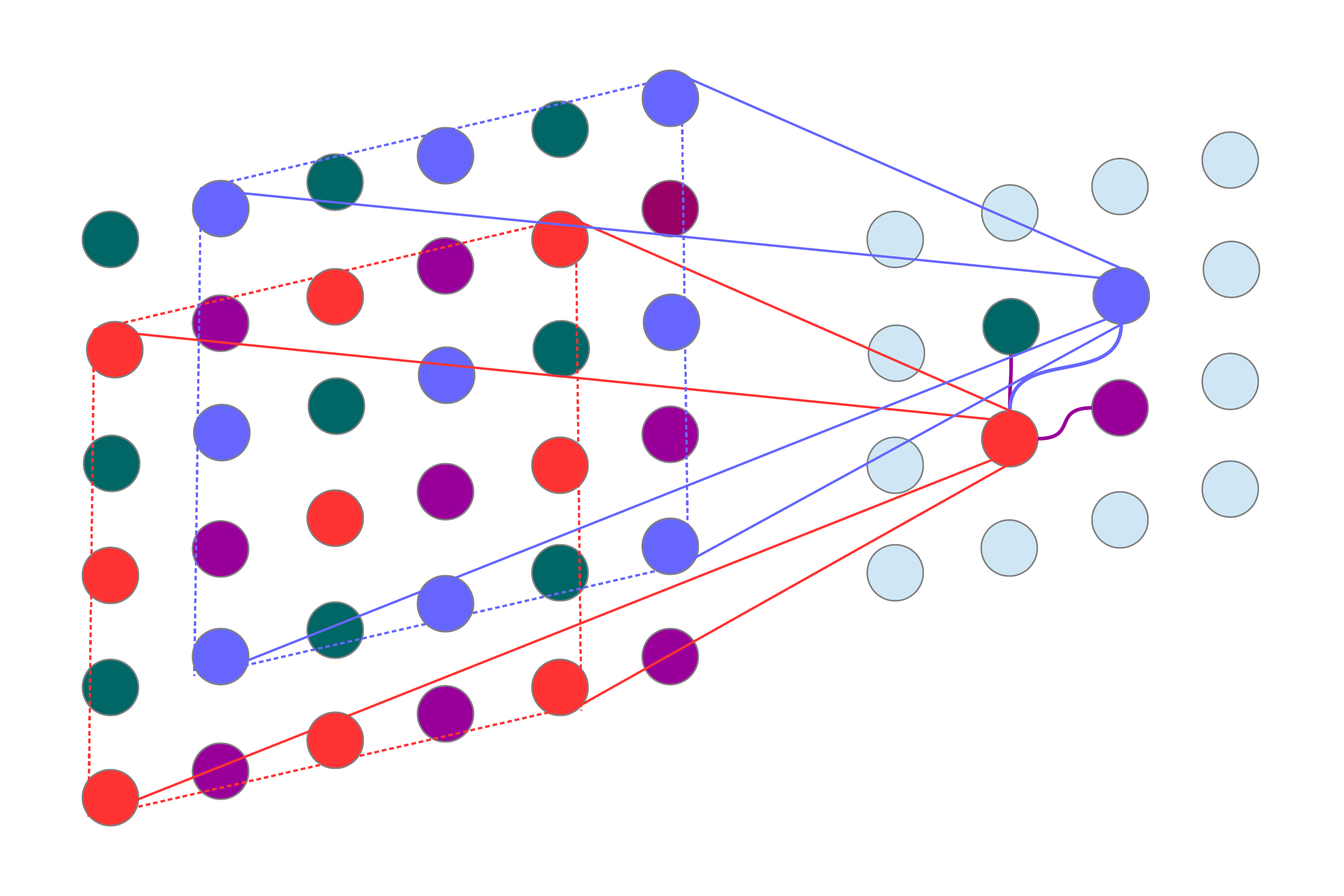}
		\caption{Dilated \ARMA}
		\label{fig:dilated-arma-2d}
	\end{subfigure}
\caption{{\bf Diagrams of receptive field.}
In \ARMA layer {\bf(b)}, each output neuron receives its neighbors' receptive field.
In {\bf(d)}, \ARMA's autoregression fills the gaps created by dilated convolution.}
\label{fig:cmp-2d}
\end{figure}
\vspace{-0.5em}

Our \ARMA layer is complementary to dilated convolutional layer, deformable convolutional layer, 
non-local attention block and encoder-decoder architecture, and can be combined with each of them.
For instance, {\em dilated \ARMA layer}, illustrated in \autoref{fig:dilated-arma-2d}, 
removes the gridding effect caused by dilated convolution ---
the \ARlong kernel can be interpreted as an anti-aliasing filter.  

The motivation of introducing \ARMA layer is to enlarge the effective input region
for each network output without increasing the filter size or network depth, 
thus avoiding the difficulties in training larger or deeper models.
As illustrated in \autoref{fig:cmp-2d}, 
each output neuron in a traditional convolutional layer (\autoref{fig:conv-2d}) 
only receives information from a small input region (the filter size).
However, an \ARMA layer enlarges the region from a local small one to a larger one (\autoref{fig:arma-2d}),
and enables an output neuron to receive information
from a faraway input neuron through the connections to its neighbors.
Now we formally introduce the concept of {\em \ERFlong} (\ERF) to characterize the effective input region. 
And we will provably show that an \ARMA network can have arbitrarily large \ERF 
with a single extra parameter at each layer in Theorem~\ref{thm:erf-arma} in the following subsection.

\subsection{\ERFLONG}
\label{sub:erf}

\ERFLong (\ERF)~\citep{luo2016understanding} measures 
the area of the input region that makes {\em substantial} contribution to an output neuron. 
In this section, we analyze the ERF size of an $L$-layers network 
with \ARMA layers v.s.\ traditional convolutional layers.
Formally, consider an output at location $(i_1, i_2)$, 
the impact from an input pixel at $(i_1 - p_1, i_2 - p_2)$ 
(i.e \  $L$ layers and $(p_1, p_2)$ pixels away) 
is measured by the amplitude of partial derivative
$g(i_1, i_2, p_1, p_2) = \left| \derivativeInline{\tensorInd{Y}{L}{i_1, i_2, t}} 
{\tensorInd{X}{1}{i_1 - p_1, i_2 - p_2, s}}\right|$ (where superscripts index the layers), 
i.e.\ how much the output changes as the input pixel is perturbed. 

\begin{definition} [{\bf \ERFLONG, \ERF}]
\label{def:erf}
Consider an $L$-layers network 
with an $S$-channels input $\tensorSup{X}{1} \in \R^{I_1 \times I_2 \times S}$
and a $T$-channels output $\tensorSup{Y}{L} \in \R^{I_1 \times I_2 \times T}$,
its \ERFlong is defined as the empirical distribution of the gradient maps:
$\text{ERF}(p_1, p_2) = {1 / (I_1 I_2 S T)} \cdot \sum_{s, t, i_1, i_2} \allowbreak
[{g(i_1, i_2, p_1, p_2)} \allowbreak / {\sum_{j_1, j_2} g(j_1, j_2, p_1, p_2)}]$,  
To measure the size of the \ERF, we define its radius $r(\textsl\ERF)$ 
as the standard deviation of the empirical distribution:
{\small
\begingroup
\setlength{\abovedisplayskip}{2pt}
\setlength{\belowdisplayskip}{0pt}
\begin{equation}
r^2\left( \textsl\ERF \right) = 
\sum_{p_1, p_2} \left( p_1^2 + p_2^2 \right) \textsl\ERF\left( p_1, p_2 \right)  -
\left[ \sum_{p_1, p_2} \sqrt{p_1^2 + p_2^2} ~ \textsl\ERF(p_1, p_2) \right]^2
\label{eq:erf-radius}
\end{equation}
\endgroup}%
\end{definition}

Notice that \ERF simultaneously depends on the model parameters 
and a specified input to the network, i.e.\ \ERF is both {\em model-dependent} and {\em data-dependent}. 
Therefore, it is generally intractable to compute the \ERF analytically for any practical neural network. 

We follow the original paper of \ERF~\citep{luo2016understanding} 
to estimate the radius with a simplified linear network.
The paper empirically verifies that such an estimation
is accurate and can be used to guide filter designs.

\begin{theorem}[{\bf \ERF of a linear \ARMA network with dilated convolutions}]
\label{thm:erf-arma}
Consider an $L$-layers linear network,
where the $\ell^{\text{th}}$ layer computes
$\vectorInd{y}{\ell}{i} - \scalarSup{a}{\ell} \vectorInd{y}{\ell}{i-1}
= \sum_{p = 0}^{\scalarSup{K}{\ell} -1} 
 [(1 - \scalarSup{a}{\ell}) / \scalarSup{K}{\ell}] \cdot \vectorInd{y}{\ell - 1}{i - \scalarSup{d}{\ell} p}$
(i.e.\ the \MAlong coefficients are uniform
with length $\scalarSup{K}{\ell}$ and dilation $\scalarSup{d}{\ell}$, 
and the \ARlong coefficients $\vectorInd{a}{\ell}{0} = 1, 
\vectorInd{a}{\ell}{1} = - \scalarSup{a}{\ell} $ has length $2$). 
Suppose $0 \leq \scalarSup{a}{\ell} < 1$ for $1 \leq \ell \leq L$, 
the  \ERF radius of such a linear \ARMA network is
{\small
\begingroup
\setlength{\abovedisplayskip}{2pt}
\setlength{\belowdisplayskip}{2pt}
\begin{equation}
r(\text{\ERF})_{\text{\ARMA}}^2 = \sum_{\ell = 1}^{L}
\left[ \frac{{\scalarSup{d}{\ell}}^2 \left({\scalarSup{K}{\ell}}^2 - 1\right)}{12} 
+ \frac{ \scalarSup{a}{\ell}}{\left(1 - \scalarSup{a}{\ell}\right)^2} \right]
\label{eq:erf-arma}
\end{equation}
\endgroup}%
When $\scalarSup{a}{\ell} = 0, \forall \ell \in [L]$,
the \ARMA layers reduce to (dilated) convolutional layers,
and the \ERF of the resulted linear \CNN has radius
$r(\text{\ERF})_{\text{\CNN}}^2 = \sum_{\ell = 1}^{L}
{\scalarSup{d}{\ell}}^2 ({\scalarSup{K}{\ell}}^2 - 1) / 12$.
\end{theorem}

Theorem~\ref{thm:erf-arma} is proved in \autoref{app:receptive}.
If the coefficients for different layers are identical,
the radius reduces to $r(\text{\ERF})_{\text{\ARMA}} = 
\sqrt{L} \cdot \sqrt{d^2 (K^2 - 1) / 12 + a / (1 - a)^2}$.

{\it Remarks:} 
\textbf{(1) Compared with a (dilated) \CNN, 
an \ARMA network can have arbitrarily large \ERF
with an extra parameter $a$ at each layer.} 
When the \ARlong coefficient $a$ is large (e.g.\ $a > 1 - 1 / (dK)$),
the second term $a / (1 - a)^2 $ dominates the radius, 
and the \ERF is substantially larger than that of a \CNN.
In particular, the radius tends to infinity as $a$ approaches $1$.
\textbf{(2) An \ARMA network can adaptively adjust its \ERF through learnable parameter $a$.}
As $a$ gets smaller (e.g.\ $a < 1 - 1 / (dK)$),
the second term is comparable to or smaller than the first term, 
and the effect of expanded \ERF diminishes.
In particular if $a = 0$, an \ARMA network reduces to a \CNN.

\textbf{Visualization of the \ERF.} 
In Theorem~\ref{thm:erf-arma},
we analytically show that the \ARMA's radius of \ERF increases with  
the network depth and magnitude of the \ARlong coefficients. 
We now verify our analysis by simulating linear \ARMA networks 
with a single extra parameter (an \ARlong coefficient $a$) in each layer 
under varying depths and magnitude of the \ARlong coefficient $a$.
Shown in \autoref{fig:ERF}, as the \ARlong coefficient get larger, the radius of the \ERF increases.
When the \ARlong coefficient is zero (i.e.\ $a = 0$), 
an \ARMA network reduces to a traditional convolutional network.  
The simulation results also indicate that the \ERF expands as the networks get deeper, 
and \ARMA's ability to expand the \ERF increases as the networks get deeper.
In conclusion, an \ARMA network can have a large \ERF even when the network is shallow, 
and its ability to expand the \ERF increases as the network gets deeper.

\vspace{-0.5em}
\begin{figure}[!htbp]
\centering
{\small
	\begin{tabular}{c c | c c}
	a & \hspace{0.4cm} \(L = 1\) \hfill \(L = 3\) \hfill \(L = 5\)  \hspace{0.4cm} 
	& a & \hspace{0.4cm} \(L = 1\) \hfill \(L = 3\) \hfill \(L = 5\) \hspace{0.4cm} \\
	\raisebox{0.8cm}{\begin{tabular}{c} \(0.0\) \\ (CNN) \end{tabular}} & \hspace{-0.2cm}
	\includegraphics[width=0.38\linewidth]{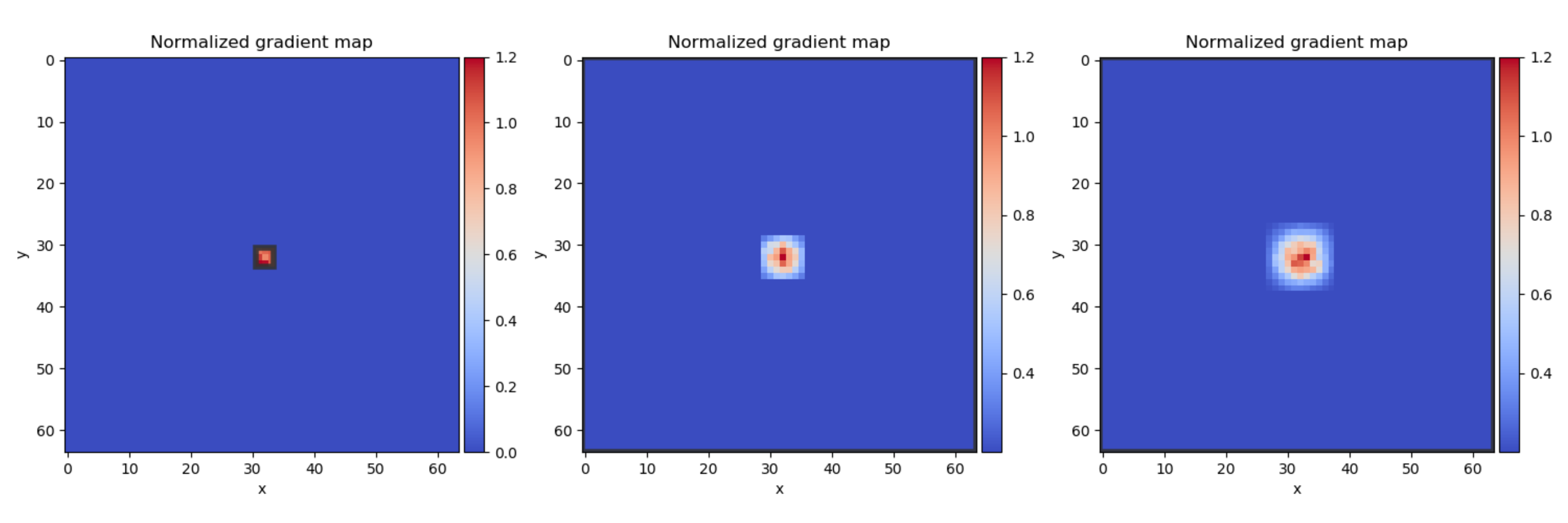}
	& \raisebox{0.8cm}{\(0.8\)} & \hspace{-0.2cm} 
	\includegraphics[width=0.38\linewidth]{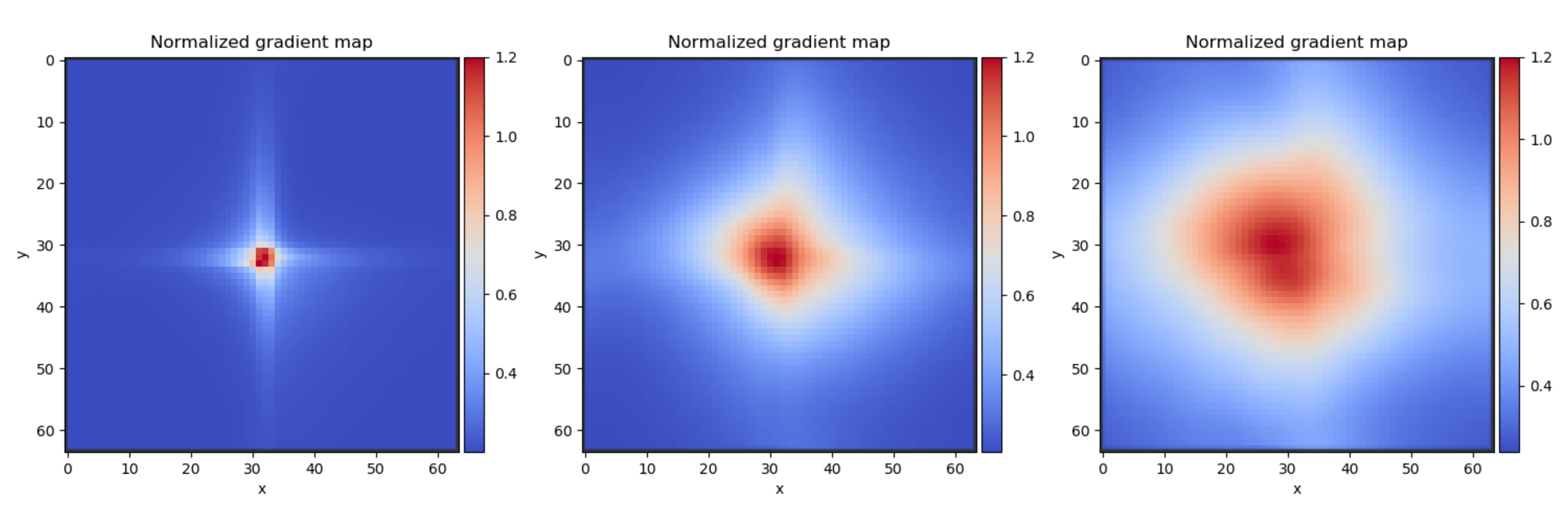} \\
	\raisebox{0.8cm}{\(0.6\)} & \hspace{-0.2cm} 
	\includegraphics[width=0.38\linewidth]{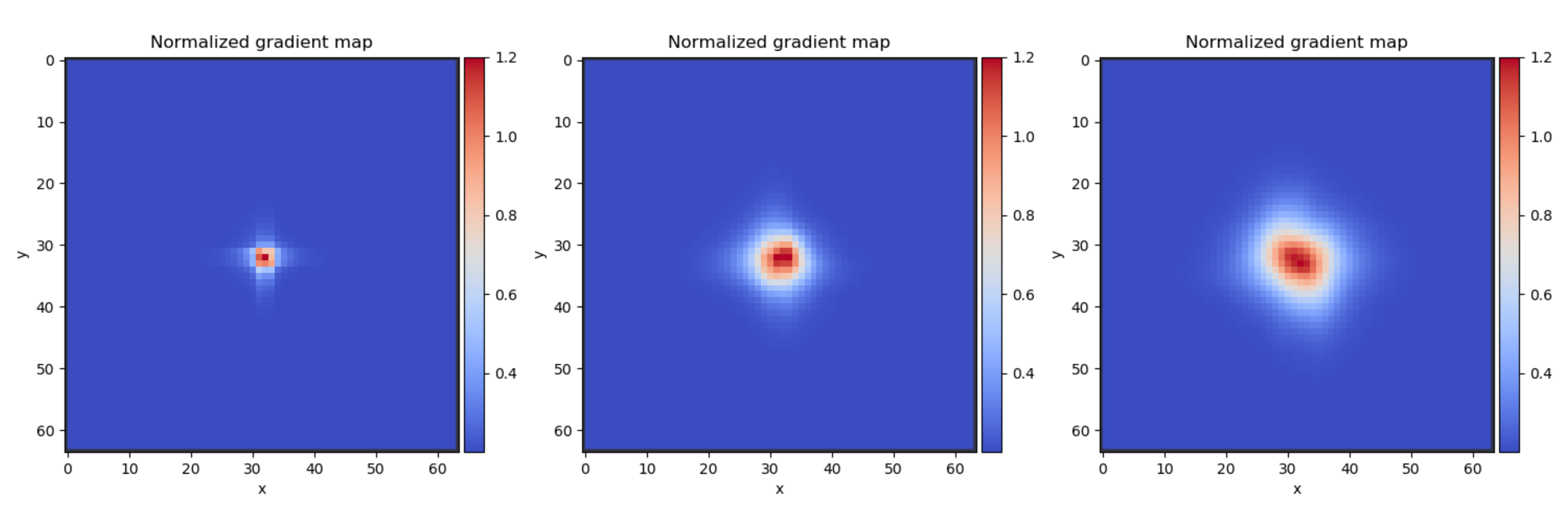} 
	& \raisebox{0.8cm}{\(0.9\)} & \hspace{-0.2cm} 
	\includegraphics[width=0.38\linewidth]{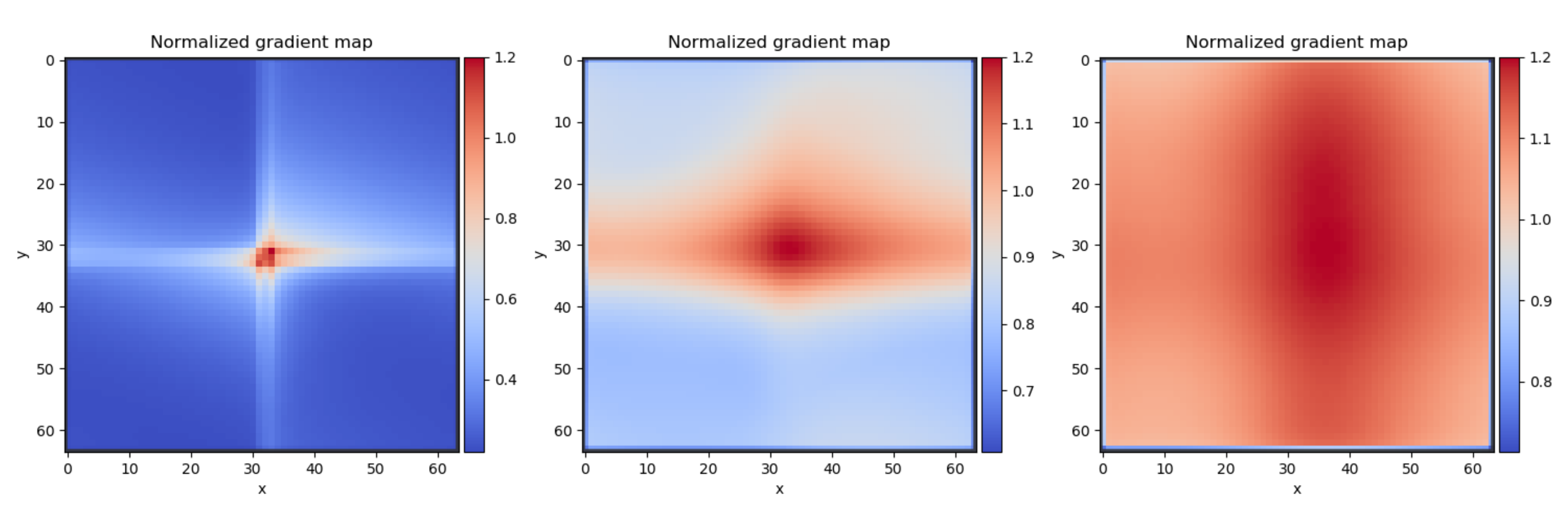} \\
	\end{tabular}
\caption{
Visualization of {\ERF} in linear {\ARMA} networks with a single extra parameter (an autoregressive coefficient $a$) in each layer, under different network depth $L=1,3,5$ and different magnitude of the autoregressive coefficient $a=0.0, 0.6, 0.8,0.9$.}
\label{fig:ERF}
}%
\end{figure}
\vspace{-1.5em}

\section{Prediction and Learning of \ARMA Layer}
\label{sec:arma-computation}

In the \ARMA layer, each neuron is influenced by its neighbors from all directions (see \autoref{fig:arma-2d}).
As a result, no neurons could be evaluated alone before evaluating any other neighboring neurons.
To compute \autoref{eq:arma-cnn}, we thus need to solve a system of linear equations to obtain all values simultaneously.
{\bf (1)} However, the standard solver using Gaussian elimination is too expensive to be practical, 
and therefore we need to seek for a more efficient solution.
{\bf (2)} Furthermore, the solver for the system of linear equations is typically not automatic differentiable,
and we have to derive the backward equations analytically.
{\bf (3)} Finally, we also need to devise an efficient algorithm to
compute the backpropagation equations efficiently.
In the section, we address these aforementioned problems.

\textbf{Decomposing \ARMA Layer.}
We decompose the \ARMA layer in \autoref{eq:arma-cnn} 
into a {\MAlong} (\MA) layer and an {\ARLong} layer, 
with $\mytensor{T} \in \R^{I_1 \times I_2 \times T}$ as an intermediate result:
{\small
\begingroup
\setlength{\abovedisplayskip}{2pt}
\setlength{\belowdisplayskip}{0pt}
\begin{equation}
\textsf{\MA Layer: } ~ 
\tensorSub{T}{\bm{:}, \bm{:}, t} = \sumIndex{s}{S}
\tensorSub{W}{\bm{:}, \bm{:}, t, s} \conv \tensorSub{X}{\bm{:}, \bm{:}, s}; \quad
\textsf{\AR Layer: } ~ 
\tensorSub{A}{\bm{:}, \bm{:}, t} \conv \tensorSub{Y}{\bm{:}, \bm{:}, t} 
= \tensorSub{T}{\bm{:}, \bm{:}, t}
\label{eq:arma-cnn-steps}
\end{equation}
\endgroup}%

\begin{wraptable}{r}{0.48\textwidth}
\vspace{-1.5em}
\setlength{\tabcolsep}{3pt}
\small{
	\begin{tabular}{c | c | c | c}
	Layer & \# params. & \# FLOPs & $r(\textsl{\ERF})^2$ \\
	\hline
	Conv. & $K_w^2 C^2$ & $O(I^2 K_w^2 C^2)$ & $O(L K_w^2)$ \\ 
	\hline
	\ARMA
	& $\begin{gathered} K_w^2 C^2 \\ + \textcolor{red}{K_a^2 C} \end{gathered}$
	& $\begin{gathered} O(K_w^2 I^2  C^2 + \\ \textcolor{red}{ I^2  \log(I)\ C}) \end{gathered}$
	& $\begin{gathered} O\big(L K_w^2 + \\ \textcolor{blue}{L \frac{a}{(1 - a)^2}} \big) \end{gathered}$
	\end{tabular}
	}
\captionof{table}{{\small{
\ARMA layer achieves large
\textcolor{blue}{gain} of \ERF radius through small \textcolor{red}{overhead} of extra \# of parameters and \# of FLOPs. Through a single extra parameter $a$ (thus $K_a = 2$), the \ERF radius can be arbitrarily large.
For notational simplicity, we assume the heights/widths are equal $I_1$$=$$I_2$$=$$I_1^{\prime}$$=$$I_2^{\prime}$$ =$$ I$,  and the input and output channels are the same $S $$=$$ T $$=$$ C$.}}}
\vspace{-2em}
\label{tab:cmp-arma-conv}
\end{wraptable}

\textbf{Difficulty in Computing the \AR Layer.} 
While the \MA layer is simply a traditional convolutional layer, 
it is nontrivial to solve the \AR layer.
Naively using Gaussian elimination, the linear equations in the \AR layer
can be solved in time cubic in dimension $O((I_1^2 + I_2^2) I_1I_2 T)$, which is too expensive.

\textbf{Solving the \AR Layer.}
We propose to use the frequency-domain division~\citep{lim1990two} 
to solve the {\em deconvolution} problem in the \AR layer.
Since the convolution in spatial domain leads to element-wise product in frequency domain, we first transform 
\(\mytensor{A}, \mytensor{T}\) into their frequency representations \(\mytensor{\Fourier{A}},\mytensor{\Fourier{T}}\), 
with which we compute \(\mytensor{\Fourier{Y}}\)
(the frequency representation of \(\mytensor{Y}\)) with element-wise division.
Then, we reconstruct the output \(\mytensor{Y}\) by an inverse Fourier transform of \(\mytensor{\Fourier{Y}}\).

\textbf{Computational Overhead.} \ARMA trades small overhead of 
extra number of parameters and computation for large gain of \ERF radius as shown in \autoref{tab:cmp-arma-conv}.
With {\em \FFTLONG} (\FFT), the FLOPS required by the extra autoregressive layer is 
$O(\log(\max(I_1, I_2)) I_1 I_2 T)$ (see \autoref{app:computation} for derivations).
Importantly, compared with non-local attention block~\citep{wang2018non}, the extra computation introduced in a \ARMA layer is smaller; a non-local attention block requires $O(I_1^2 I_2^2 T)$ FLOPS.

\textbf{Backpropagation.} 
Deriving the backpropagation for \autoref{eq:arma-cnn-steps} is nontrivial; 
although backpropagation rule for \MA layer is conventional, that of \AR layer is not. 
In Theorem~\ref{thm:ar-backprop} we show that
backpropagation of an \AR layer can be computed as two \ARMA models.
\begin{theorem} [{\bf Backpropagation of \ARMA layer}]
\label{thm:ar-backprop}
Given \(\tensorSub{A}{\bm{:}, \bm{:}, t} \ast \tensorSub{Y}{\bm{:}, \bm{:}, t} = \tensorSub{T}{\bm{:}, \bm{:}, t}\) 
and the gradient \(\gradientInline{\mytensor{Y}}\), the gradients \(\{\gradientInline{\mytensor{A}}, \gradientInline{\mytensor{X}}\}\) can be obtained by two \ARMA models:
{\small
\begingroup
\setlength{\abovedisplayskip}{2pt}
\setlength{\belowdisplayskip}{2pt}
\begin{equation}
\tensorSub{\adjoint{A}}{\bm{:}, \bm{:}, t} \conv \gradient{\tensorSub{A}{\bm{:}, \bm{:}, t}} =
- \tensorSub{\adjoint{Y}}{\bm{:}, \bm{:}, t} \conv \gradient{\tensorSub{Y}{\bm{:}, \bm{:}, t}}; \quad 
\tensorSub{\adjoint{A}}{\bm{:}, \bm{:}, t} \conv \gradient{\tensorSub{T}{\bm{:}, \bm{:}, t}} =
\gradient{\tensorSub{Y}{\bm{:}, \bm{:}, t}}
\label{eq:arma-cnn-backprop}
\end{equation}
\endgroup}%
where \(\tensorSub{\adjoint{A}}{\bm{:}, \bm{:}, t}\) and \(\tensorSub{\adjoint{Y}}{\bm{:}, \bm{:}, t}\) 
are the transposed images of \(\tensorSub{A}{\bm{:}, \bm{:}, t}\) and \(\tensorSub{Y}{\bm{:}, \bm{:}, t}\)
(e.g.\ \(\tensorSub{\adjoint{A}}{i_1, i_2, t} = \tensorSub{A}{-i_1, -i_2, t}\)).
\end{theorem}
Since the backpropagation is characterized by \ARMA models,
it can be evaluated efficiently using \FFT similar to \autoref{eq:arma-cnn-steps}.
The proof of Theorem~\ref{thm:ar-backprop} with its \FFT evaluation is given in \autoref{app:computation}.

\section{Stability of \ARMA Layers}
\label{sec:stability}

An \ARMA model with arbitrary coefficients is not always stable.
For example, the model \( \vectorSub{y}{i} - a \vectorSub{y}{i-1} = \vectorSub{x}{i} \) is unstable if \(|a| > 1\): 
Consider an input \(\myvector{x}\) with \(\vectorSub{x}{0} = 1\) and \(\vectorSub{x}{i} = 0, \forall i \neq 0\), 
the output \(\myvector{y}\) will recursively amplify itself as
\(\vectorSub{y}{0} = 1, \vectorSub{y}{1} = a, \cdots, \vectorSub{y}{i} = a^{i}\) and diverge to infinity.

\subsection{Stability Constraints for \ARMA layer}
\label{sub:arma-stability}

The key to guarantee stability of an \ARMA layer is 
to constrain its \ARlong coefficients, which
prevents the output from repeatedly amplifying itself.
To derive the constraints, we propose a special design, {\em separable \ARMA layer}
inspired by {\em separable filters}~\citep{lim1990two}.

\begin{definition}[\bf Separable \ARMA Layer]
\label{def:separable-arma-cnn}
A separable \ARMA layer is parameterized by a \MAlong kernel 
\(\mytensor{W} \in \R^{K_{w} \times K_{w} \times S \times T}\) and 
\(T \times Q\) sets of \ARlong filters
\(\{( \matrixInd{f}{q}{\bm{:}, t}, \matrixInd{g}{q}{\bm{:}, t})_{q = 1}^{Q}\}_{t = 1}^{T}\),.
It takes an input \(\mytensor{X} \in \R^{I_1 \times I_2 \times S}\) 
and returns an output \(\mytensor{Y} \in \R^{I^{\prime}_1 \times I^{\prime}_2 \times T}\) as
{\small
\begingroup
\setlength{\abovedisplayskip}{2pt}
\setlength{\belowdisplayskip}{2pt}
\begin{equation}
\left( \matrixInd{f}{1}{\bm{:}, t} \conv \cdots \conv \matrixInd{f}{Q}{\bm{:}, t} \right) \otimes 
\left( \matrixInd{g}{1}{\bm{:}, t} \conv \cdots \conv \matrixInd{g}{Q}{\bm{:}, t} \right) \conv 
\tensorSub{Y}{\bm{:}, \bm{:}, t} = \sumIndex{s}{S} \
\tensorSub{W}{\bm{:}, \bm{:}, t, s} \conv \tensorSub{X}{\bm{:}, \bm{:}, s}
\label{eq:separable-arma-cnn}
\end{equation}
\endgroup}%
where the filters \( \matrixInd{f}{q}{\bm{:}, t}, \matrixInd{g}{q}{\bm{:}, t} \in \R^{3} \) are length-$3$, 
and \(\otimes\) denotes outer product of two 1D-filters.
\end{definition}

{\it Remarks:}
Each autoregressive filter \(\tensorSub{A}{\bm{:}, \bm{:}, t}\) is designed to be separable, 
i.e.\ \(\tensorSub{A}{\bm{:}, \bm{:}, t} = \matrixSub{F}{\bm{:}, t} \otimes \matrixSub{G}{\bm{:}, t}\), thus it
can be characterized by 1D-filters \( \matrixSub{F}{\bm{:}, t}\) and \( \matrixSub{G}{\bm{:}, t}\). 
By the fundamental theorem of algebra~\citep{oppenheim2014discrete}, 
any 1D-filter can be represented as a composition of length-3 filters.
Therefore, \( \matrixSub{F}{\bm{:}, t}\) and \( \matrixSub{G}{\bm{:}, t}\) can further be factorized as
\( \matrixSub{F}{\bm{:}, t} = \matrixInd{f}{1}{\bm{:}, t} \conv
\matrixInd{f}{2}{\bm{:}, t} \cdots \conv \matrixInd{f}{Q}{\bm{:}, t}\) and 
\( \matrixSub{G}{\bm{:}, t} = \matrixInd{g}{1}{\bm{:}, t} \conv 
\matrixInd{g}{2}{\bm{:}, t} \cdots \conv \matrixInd{g}{Q}{\bm{:}, t}\).
In summary, each \(\tensorSub{A}{\bm{:}, \bm{:}, t}\) is characterized by
\(Q\) sets of length-$3$ \ARlong filters \(( \matrixInd{f}{q}{\bm{:}, t}, \matrixInd{g}{q}{\bm{:}, t})_{q=1}^{Q}\).

\vspace{-1em}
\begin{figure}[!htbp]
	\includegraphics[width=\textwidth]{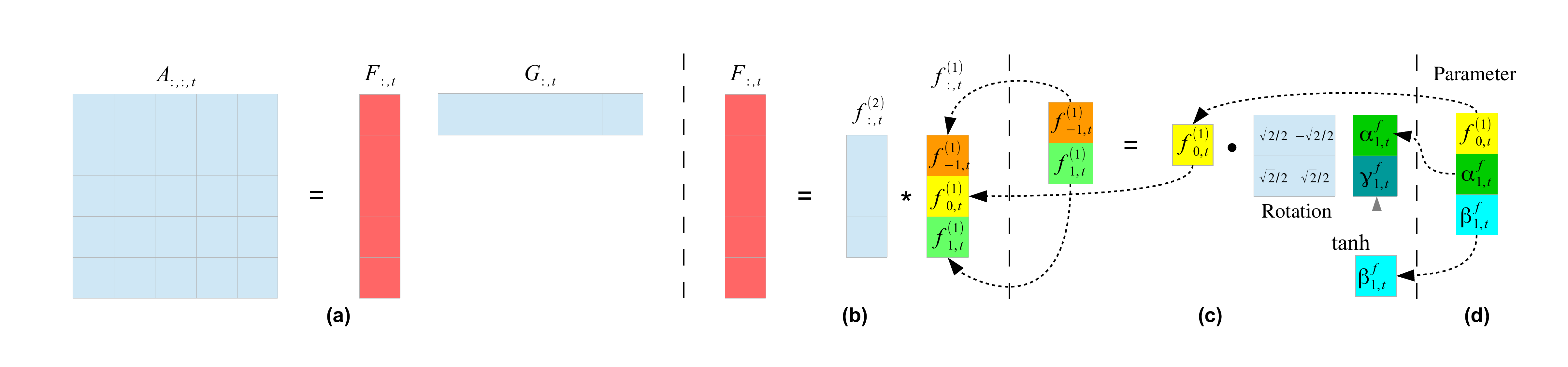}
	\vspace{-2em}
\captionof{figure}{For each channel $t$, 
\textbf{(a)} the two-dimensional filter $\tensorSub{A}{\bm{:}, \bm{:}, t}$ is parameterized 
 through an outer product of two 1D-filters $\matrixSub{F}{\bm{:}, t}$ and $\matrixSub{G}{\bm{:}, t}$;
\textbf{(b)} $\matrixSub{F}{\bm{:}, t}$ is parameterized through a convolution of 
$\matrixInd{f}{1}{\bm{:}, t} \conv \cdots \conv \matrixInd{f}{Q}{\bm{:}, t}$,
and similarly $\matrixSub{G}{\bm{:}, t}$ as a convolution of 
$\matrixInd{g}{1}{\bm{:}, t} \conv \cdots \conv \matrixInd{g}{Q}{\bm{:}, t}$;
\textbf{(c)} we re-parameterize each constrained $(\matrixInd{f}{q}{-1, t}, \matrixInd{f}{q}{1, t})$
to unconstrained $(\matrixSub{\alpha^{f}}{q, t}, \matrixSub{\beta^{f}}{q, t})$,
and similarly $(\matrixInd{g}{q}{-1, t}, \matrixInd{g}{q}{1, t})$ 
to $(\matrixSub{\alpha^{g}}{q, t}, \matrixSub{\beta^{g}}{q, t})$;
\textbf{(d)} final parameters for unconstrained optimization are 
$(\matrixSub{\alpha^{f}}{q, t}, \matrixSub{\beta^{f}}{q, t}, 
\matrixSub{\alpha^{g}}{q, t}, \matrixSub{\beta^{g}}{q, t})_{q = 1}^{Q}$.}
\end{figure}

\begin{theorem} [{\bf Constraints for Stable Separable \ARMA Layer}]
\label{thm:stability-arma}

A sufficient condition for the separable \ARMA layer 
(Definition~\ref{def:separable-arma-cnn}) 
to be stable (i.e.\ output be bounded for any bounded input) is:
{\small
\begingroup
\setlength{\abovedisplayskip}{2pt}
\setlength{\belowdisplayskip}{0pt}
\begin{equation}
\left| \matrixInd{f}{q}{-1, t} + \matrixInd{f}{q}{1, t} \right| < \matrixInd{f}{q}{0, t}, ~
\left| \matrixInd{g}{q}{-1, t} + \matrixInd{g}{q}{1, t} \right| < \matrixInd{g}{q}{0, t}, ~
\forall q \in [Q], t \in [T].
\label{eq:second-order-constraint}
\end{equation}
\endgroup}%
\end{theorem}
The proof is deferred to \autoref{app:stability}, 
which follows the standard techniques using Z-transform.

\subsection{Achieving stability via re-parameterization}
\label{sub:arma-reparameterization}

In principle, the constraints required for stability in a \ARMA layer 
as in Theorem~\ref{thm:stability-arma} 
could be enforced through constrained optimization.
However, constrained optimization algorithm, 
such as projected gradient descent~\citep{bertsekas2015convex}, 
is more expensive as it requires an extra projection step.
Moreover, it could be more difficult to achieve convergence.
In order to avoid the aforementioned challenges,
we introduce a {\em re-parameterization} mechanism 
to remove constraints needed to guarantee stability in ARMA layer. 

\begin{theorem}[Re-parameterization]
\label{thm:stability-param}
For a separable \ARMA layer in Definition~\ref{def:separable-arma-cnn},
if we re-parameterize each tuple 
\((\matrixInd{f}{q}{-1, t}, \matrixInd{f}{q}{1, t}, 
\matrixInd{g}{q}{-1, t}, \matrixInd{g}{q}{1, t})\) 
as learnable parameters 
\((\matrixSub{\alpha^{f}}{q, t}, \matrixSub{\beta^{f}}{q, t}, 
\matrixSub{\alpha^{g}}{q, t}, \matrixSub{\beta^{g}}{q, t})\): 
{\small
\setlength{\abovedisplayskip}{2pt}
\setlength{\belowdisplayskip}{2pt}
\begin{equation}
\begin{pmatrix} \matrixInd{f}{q}{-1, t} & \matrixInd{g}{q}{-1, t} \\ 
\matrixInd{f}{q}{1, t} & \matrixInd{g}{q}{1, t} \end{pmatrix} = 
\begin{pmatrix} \matrixInd{f}{q}{0, t} & 0 \\ 0 & \matrixInd{g}{q}{0, t} \end{pmatrix}
\begin{pmatrix} \sqrt{2} / 2 & -\sqrt{2} / 2 \\ \sqrt{2} / 2 & \sqrt{2} / 2 \end{pmatrix} 
\begin{pmatrix} \matrixSub{\alpha^{f}}{q, t} & \matrixSub{\alpha^{g}}{q, t} \\ 
\tanh(\matrixSub{\beta^{f}}{q, t}) & \tanh(\matrixSub{\beta^{g}}{q, t}) \end{pmatrix}
\label{eq:re-parameterization}
\end{equation}}%
then the layer is stable for \textbf{arbitrary} 
\(\{(\matrixInd{f}{q}{0, t}, \matrixInd{g}{q}{0, t}, \matrixSub{\alpha^{f}}{q, t}, \matrixSub{\beta^{f}}{q, t}, 
\matrixSub{\alpha^{g}}{q, t}, \matrixSub{\beta^{g}}{q, t})_{q=1}^{Q}\}_{t=1}^{T}\) with no constraints.
\end{theorem}

In practice, we can set \( \matrixSub{f^{q}}{0, t} = \matrixSub{g^{q}}{0, t} = 1\) 
(since the scale can be learned by the \MAlong kernel),
and only store and optimize over each tuple
\((\matrixSub{\alpha^{f}}{q, t}, \matrixSub{\beta^{f}}{q, t}, \matrixSub{\alpha^{g}}{q, t}, \matrixSub{\beta^{g}}{q, t})\).
In other words, each \ARlong filter \(\tensorSub{A}{\bm{:}, \bm{:}, t}\) is constructed from
\((\matrixSub{\alpha^{f}}{q, t}, \matrixSub{\beta^{f}}{q, t}, \matrixSub{\alpha^{g}}{q, t}, \matrixSub{\beta^{g}}{q, t})_{q = 1}^{Q}\) on the fly during training or inference.

\begin{wrapfigure}{r}{0.35\textwidth}
\vspace{-1em}
\begin{minipage}{0.35\textwidth}
\centering
	\resizebox{0.9\textwidth}{!}{\input{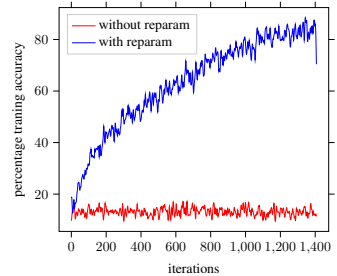}}
	\captionof{figure}{
	\small{Learning curves with and without re-parameterization on 
	an \ARMA network with a VGG-11 backbone on CIFAR-10.}}
\label{fig:stabilization-compare}
\end{minipage}
\vspace{-2em}
\end{wrapfigure}

\textbf{Experimental Demonstration of Re-parameterization.}
To verify the re-parameterization is essential for stable training,
we train a VGG-11 network~\citep{simonyan2014very} on CIFAR-10 dataset, 
where all convolutional layers are replaced by \ARMA layers 
with \ARlong coefficients initialized as zeros.
We compare the learning curves using the re-parameterization 
v.s.\ not using the re-parameterization in \autoref{fig:stabilization-compare}.
As we can see, the training quickly converges under our proposed re-parameterization mechanism 
with which the stability of the network is guaranteed.
However without the re-parameterization mechanism, 
a naive training of \ARMA network never converges and gets NaN error quickly. 
The experiment thus verifies that the theory in
Theorem~\ref{thm:stability-param} is effective in guaranteeing stability. 
\section{Experiments}
\label{sec:experiments}

We apply our \ARMA networks on two dense prediction problems
-- pixel-level video prediction and semantic segmentation to demonstrate effectiveness of \ARMA networks.
{\bf (1)} We incorporate our \ARMA layers 
in U-Net~\citep{ronneberger2015u, wang2020non} for semantic segmentation,
and in ConvLSTM network~\citep{xingjian2015convolutional, byeon2018contextvp} for video prediction.
\textbf{We show that the resulted \ARMA U-Net and \ARMA-LSTM models
uniformly outperform the baselines on both tasks.}
{\bf (2)} We then interpret the varying performance of \ARMA networks on different tasks
by visualizing the histograms of the learned \ARlong coefficients.
We include the detailed setups (datasets, model architectures, training strategies and evaluation metrics) 
and visualization in \autoref{app:supp_exp} for reproducibility purposes.

\textbf{Semantic Segmentation on Biomedical Medical Images.}
We evaluate our \ARMA U-Net on the lesion segmentation task 
in ISIC 2018 challenge~\citep{lesion-segmentation-isic-2018},
comparing against a baseline U-Net~\citep{ronneberger2015u} and
non-local U-Net~\citep{wang2020non} (U-Net augmented with non-local attention blocks).

\vspace{-1em}
\begin{table}[!htbp]
\caption{\small{\textbf{Semantic segmentation on ISIC dataset}. 
For all metrics (ACC, SE, SP, PC, F1 and JS), higher values indicates better performance. 
The reported numbers are an average of $10$ runs with different seeds.}}
\label{tab:eval-isic}
\setlength{\tabcolsep}{2pt}
\centering
\resizebox{\textwidth}{!}{
	\begin{tabular}{c| c | c c c c c c}
	\midrule
	Model & params. & ACC & SE & SP & PC & F1 & JS \\
	\midrule
	U-Net~\citep{ronneberger2015u} & 3.453M & 0.946 $\pm$ 0.003 & 0.884 $\pm$ 0.019 & 
	\textbf{0.977} $\pm$ 0.005 & 0.857 $\pm$ 0.020 & 0.842 $\pm$ 0.009 & 0.754 $\pm$ 0.011 \\
	NL U-Net~\citep{wang2020non}  & 4.403M & 0.945 $\pm$ 0.003 & 0.877 $\pm$ 0.017 & 
	0.973 $\pm$ 0.004 & 0.844 $\pm$ 0.014 & 0.831 $\pm$ 0.012 & 0.741 $\pm$ 0.013 \\
	\midrule
	\midrule
	ARMA U-Net & 3.455M & 0.955 $\pm$ 0.003 & 0.896 $\pm$ 0.011 & 
	0.972 $\pm$ 0.005 & \textbf{0.873} $\pm$ 0.011 & 0.861 $\pm$ 0.007 & 0.780 $\pm$ 0.009 \\
	NL \ARMA U-Net  & 4.405M & \textbf{0.960} $\pm$ 0.002 & \textbf{0.909} $\pm$ 0.009 & 
	0.968 $\pm$ 0.004 & 0.870 $\pm$ 0.011 & \textbf{0.870} $\pm$ 0.006 & \textbf{0.790} $\pm$ 0.008 \\
	\midrule
\end{tabular}}
\end{table}
\vspace{-0.5em}

\emph{\ARMA networks outperform both baselines in almost all metrics.}
As shown in Table~\ref{tab:eval-isic}, our (non-local) \ARMA U-Net outperform 
both U-Net and non-local U-Net except for specificity (SP).
Furthermore, we find that the synergy of non-local attention and \ARMA layers 
achieves best results among all.

\textbf{Pixel-level Video Prediction.}
We evaluate our \ARMA-LSTM network on 
the Moving-MNIST-2 dataset~\citep{moving-mnist-dataset} with different moving velocities, 
comparing against the baseline ConvLSTM network~\citep{xingjian2015convolutional, byeon2018contextvp} 
and its augmentation using dilated convolutions and non-local attention blocks~\citep{wang2018non}. As shown in the visualizations in \autoref{app:supp_exp}, the {dilated \ARMA-LSTM} does not have gridding artifacts as in dilated Conv-LSTM, that is \emph{\ARMA removes the gridding artifacts.}

\vspace{-1em}
\begin{table}[!htbp]
\caption{\small{10-frames \textbf{video prediction} on Moving-MNIST-2 with three different speeds (results averaged over 10 predicted frames). MA and AR denote the size of moving-average and autoregressive kernels respectively, and dil. denotes dilation in the moving-average kernel. Higher PSNR, SSIM values indicate better performance.}}
\label{tab:eval-mnist}
\setlength{\tabcolsep}{4pt}
\centering
	\begin{tabular}{c| c c c | c | c c c c c c}
	\midrule
	\multirow{2}{*}{Model} & \multirow{2}{*}{MA} & \multirow{2}{*}{AR} &\multirow{2}{*}{dil.} & \multirow{2}{*}{params.} 
	& \multicolumn{2}{c}{original speed} & \multicolumn{2}{c}{2X speed} & \multicolumn{2}{c}{3X speed} \\
	& & & & & PSNR & SSIM & PSNR & SSIM & PSNR & SSIM \\
	\midrule
	Conv-LSTM (size 3) & 3 & 1 & 1 & 0.887M & 18.24 & 0.867 & 16.62 & 0.827 & 15.81 & 0.810 \\
	Conv-LSTM (size 5) & 5 & 1 & 1 & 2.462M & 19.58 & 0.901 & 17.61 & 0.856 & 16.99 & 0.841 \\
	Dilated Conv-LSTM & 3 & 2 & 2 & 0.887M & 19.16 & 0.893 & 17.92 & 0.858 & 17.48 & 0.846 \\
	\midrule \midrule
	Dilated ARMA-LSTM & 3 & 3 & 2 & 0.893M &  {\bf 19.72} & {\bf 0.904} & 18.05 & 0.870 & 17.65 & 0.855 \\
	ARMA-LSTM (size 3) & 3 & 2 & 1 & 0.893M & {\bf 19.72} &  0.899 & {\bf 18.73} & {\bf 0.881} & {\bf 18.13} & {\bf 0.869} \\
	\midrule
	\end{tabular}
\end{table}
\vspace{-0.5em}

\emph{\ARMA networks outperforms larger networks:}
As shown in Table~\ref{tab:eval-mnist},
our \ARMA networks with kernel sizes \(3 \times 3\) outperform all baselines under all velocities
(at the original speed, our \ARMA network requires dilated convolutions to achieve the best performance).
Moreover, for videos with a higher moving speed, the advantage is more pronounced 
as expected due to \ARMA's ability to expand the \ERF.
The \ARMA networks improve the best baseline 
(Conv-LSTM with kernel size \(5 \times 5\)) in PSNR by \(6.36\%\) at 2X speed 
and by \(6.70\%\) at 3X speed, with {63.7\%} fewer parameters.

\emph{\ARMA networks outperforms non-local attention blocks:} 
As shown in Table~\ref{tab:eval-mnist-non-local},
our \ARMA-LSTM with kernel sizes \(3 \times 3\) outperforms 
the {Conv-LSTM}s augmented by non-local attention blocks. 
However, attention mechanism does not always improve the baselines or our models. 
When both \ARMA-LSTM and {Conv-LSTM} are combined with non-local attention blocks, 
our model achieves better performance compared to the non-\ARMA baselines.  

\begin{minipage}{0.54\textwidth}
\captionof{table}{\small{\textbf{Comparison with non-local attention blocks} on \textbf{video prediction}.} The original networks are the same as in Table~\ref{tab:eval-mnist}. Each non-local network additionally inserts two non-local blocks in the corresponding base network.}
\label{tab:eval-mnist-non-local}
\setlength{\tabcolsep}{4pt}
\centering
	\resizebox{\linewidth}{!}{
	\begin{tabular}{c | c c | c c}
	\midrule
    	\multirow{2}{*}{Model} & \multicolumn{2}{c}{Original} & \multicolumn{2}{c}{Non-local} \\
	& PSNR & SSIM & PSNR & SSIM \\
	\midrule
	ConvLSTM (size 3) & 18.24 & 0.867 & 19.45 & 0.895  \\
	ConvLSTM (size 5) & 19.58 & {\bf 0.901} & 19.18 & 0.891 \\
	\midrule \midrule
	\ARMA-LSTM (size 3) & {\bf 19.72} & 0.899 & {\bf 19.62} & {\bf 0.897}  \\
	\midrule
	\end{tabular}}%
\end{minipage}
\hfill
\begin{minipage}{0.44\textwidth}
\centering
	\includegraphics[width=\linewidth]{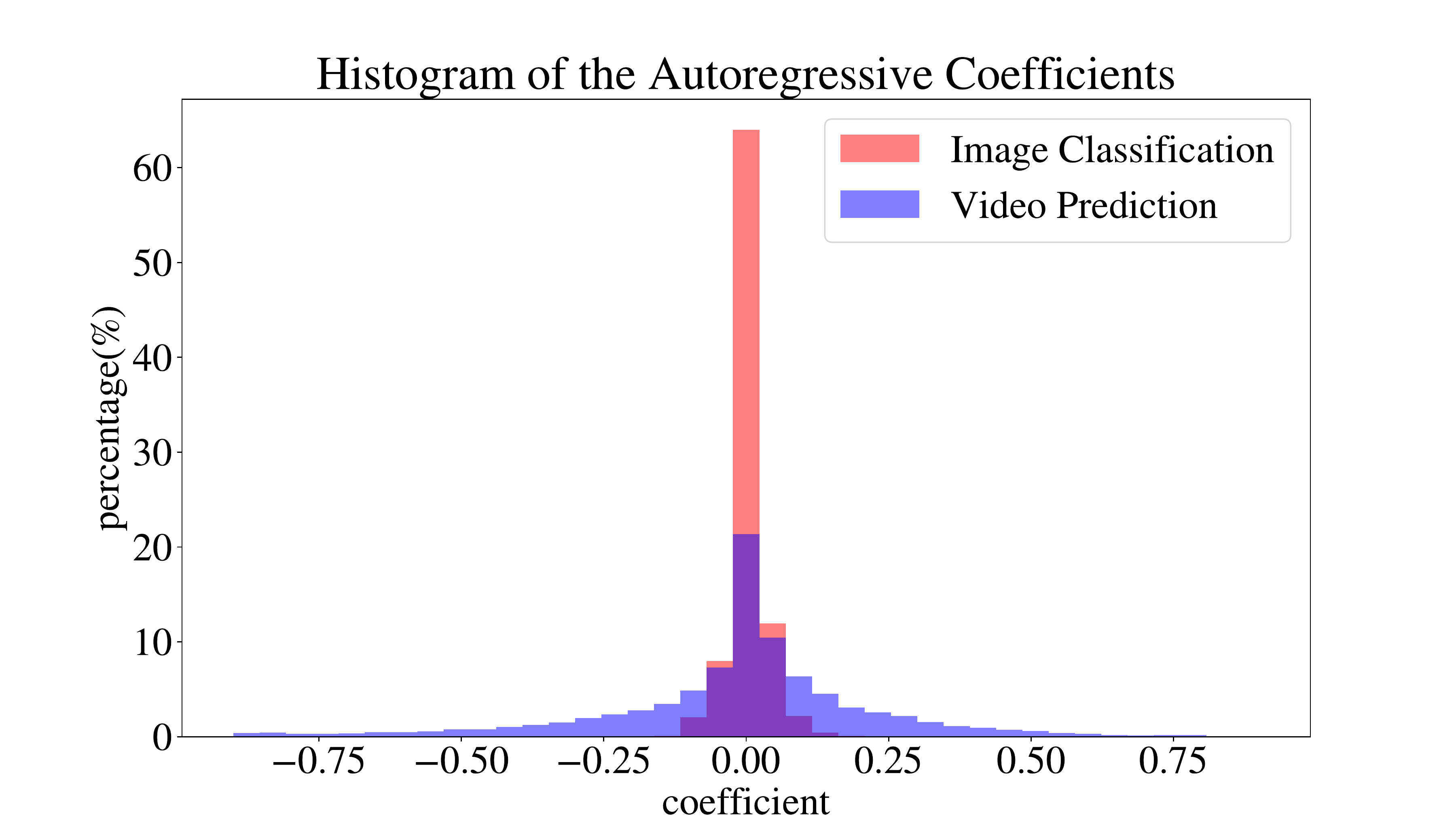}
	\captionof{figure}{\small{Histogram of the \ARlong coefficients in trained \ARMA networks}.}
\label{fig:histogram}
\end{minipage}

\textbf{Interpretation by \ARLong Coefficients.}
To explain why \ARMA networks achieve impressive performance in dense prediction, 
\autoref{fig:histogram} compares the histograms of the trained \ARlong coefficients 
between video prediction and image classification. 
(\autoref{sub:supp_exp_image_classification} demonstrates performance of image classifications 
when \ARMA layers are incorporated in VGG and ResNet.)
\begin{enumerate}[leftmargin=*, itemsep=0pt, topsep=0pt]
\item The histograms demonstrate how \ARMA networks adaptively learn \ARlong coefficients according to the tasks.
As motivated in the introduction, dense prediction such as video prediction requires each layer to have a large receptive field such that global information is captured.
\item The large \ARlong coefficients in video prediction model suggests the overall \ERF is significantly expanded.
In image classification model, global information is already aggregated 
by pooling (downsampling) layers and a fully-connected classification layer. 
Therefore, the \ARMA layers automatically learn nearly zero \ARlong coefficients.
\end{enumerate}

\section{Discussion}
\label{sec:discussion}

This paper proposes a novel {\em \ARMA} layer 
capable of expanding a network's \ERFlong adaptively. 
Our method is related to techniques in signal processing and machine learning. 
First, \ARMA layer is equivalent to a multi-channel 
{\em impulse response filter} in signal processing~\citep{oppenheim2014discrete}. 
Alternatively, we can interpret the \ARlong layer as a 
a learnable {\em spectral normalization}~\citep{miyato2018spectral} following the \MAlong layer.
Additionally, the \ARMA layer is an {\em linear recurrent neural network}, 
where the recurrent propagations are over the spatial domain (\autoref{sec:related}).

\bibliography{\bibhome/supp_bib}
\bibliographystyle{plain}

\newpage
\appendix
{\begin{center}{\bf \large Appendix of \mytitle}\end{center}}
\numberwithin{equation}{section}

\section{Supplementary Materials for Experiments}
\label{app:supp_exp}

In this section, we explain detailed setups 
(datasets, model architectures, learning strategies and evaluation metrics) 
of all experiments, and provide additional visualizations of the results.

\subsection{Visualization of Effective Receptive Field}
\label{sub:supp_exp_visualization_erf}

In the simulations in \autoref{sub:erf},
all linear networks have \(32\) channels and \(64 \times 64\) feature size at each layer.
The filter size for both \MAlong coefficients and \ARlong coefficients is set to \(3 \times 3\): 
each \MAlong kernel \(\mytensor{W}\) is initialized using Xavier's method, 
while the \ARlong kernel \(\mytensor{A}\) is initialized randomly within a stable region 
\(-a \leq \matrixInd{f}{q}{-1, t} + \matrixInd{f}{q}{1, t} \leq 0, 
-a \leq \matrixInd{g}{q}{-1, t} + \matrixInd{g}{q}{1, t} \leq 0, \forall t \in [T], q \in [Q] \)
(see \autoref{sec:stability} for details).
Each heat map in \autoref{fig:ERF} is computed as an average of \(32\) gradient maps from different channels. 

\subsection{Multi-frame Video Prediction}
\label{sub:supp_exp_video_prediction}

\paragraph{Datasets and Metrics}
The Moving-MNIST-2 dataset is generated by moving two digits of size 
$28 \times 28$ in MNIST dataset within a $64 \times 64$ black canvas~\citep{moving-mnist-dataset}. 
These digits are placed at a random initial location, 
and move with constant velocity in the canvas and bounce when they reach the boundary.
In addition to the default velocity in the public generator~\citep{moving-mnist-dataset},
we increase the velocity to \(2\times\) and \(3\times\) to test all models on videos with stronger motions.
For each velocity, we generate 10,000 videos for training set, 3,000 for validation set, and 5,000 for test set, 
where each video contains \(20\) frames.
All models are trained to the next 10 frames given 10 input frames, and we evaluate their performance 
based on the metrics of {\em mean square error} (MSE), {\em peak signal-noise ratio} (PSNR) 
and {\em structure similarity} (SSIM)~\citep{wang2004image}.

\paragraph{Model Architectures}
\textbf{(1) Baselines.}
The backbone architecture consists of a stack of 12 Conv-LSTM modules, 
and each module contains 32 units (channels).
Following~\citep{byeon2018contextvp},
two skip connections that perform channel concatenation  
are added between (3, 9) and (6, 12) module.
An additional traditional convolutional layer
is applied on top of all recurrent layers to compute the predicted frames.
The backbone architecture is illustrated in \autoref{fig:convlstm}.
In the baseline networks, we consider three different convolutions at each layer: 
{\bf(a)} Traditional convolution with filter size $3 \times 3$;
{\bf(b)} Traditional convolution with filter size $5 \times 5$; 
and {\bf(c)} $2$-dilated convolution with filter size $3 \times 3$.

\textbf{(2) \ARMA networks.}
Our \ARMA networks use the same backbone architecture as baselines, 
but replace their convolutional layers with \ARMA layers.
For all {\ARMA} models, we set the filter size for 
both \MAlong and \ARlong parts to $3 \times 3$. 
In the \ARMA networks, we consider two different convolutions each layer:
{\bf(a)} The \MAlong part is a traditional convolution;
{\bf(b)} We further consider using $2$-dilated convolution in the \MAlong part.

\textbf{(3) Non-local networks.}
In non-local networks, we additionally insert two non-local block in the backbone architecture,
as illustrated in \autoref{fig:convlstm_nl}. 
In each non-local block, we use embedded Gaussian as the non-local operation~\citep{wang2018non}, 
and we replace the batch normalization by instance normalization 
that is compatible to recurrent neural networks.
In non-local networks, we consider three types of convolutions at each layer: 
{\bf(1)(2)} Traditional convolutions with filter size $3 \times 3$ and $5 \times 5$; 
{\bf(3)} \ARMA layer with $3 \times 3$ \MAlong and \ARlong filters.

\begin{center}
\centering
	\begin{minipage}[c]{0.48\textwidth}
	\centering
		\includegraphics[trim={1.2cm 2.6cm 1.2cm 1.6cm}, clip, width=\linewidth]{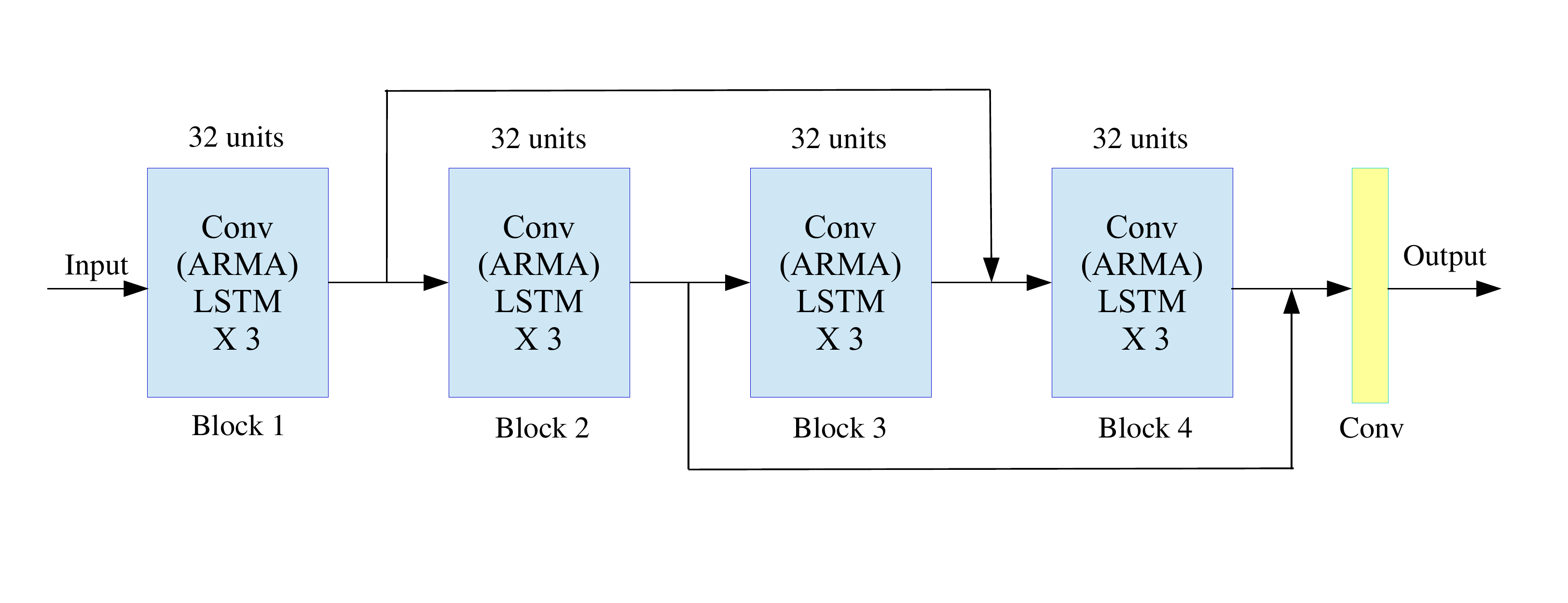}
		\captionof{figure}{Conv(ARMA)-LSTM.}
	\label{fig:convlstm}
	\end{minipage}
	\hfill
	\begin{minipage}[c]{0.5\textwidth}
	\centering
		\includegraphics[trim={1.2cm 2.6cm 1.2cm 1.6cm}, clip, width=\linewidth]{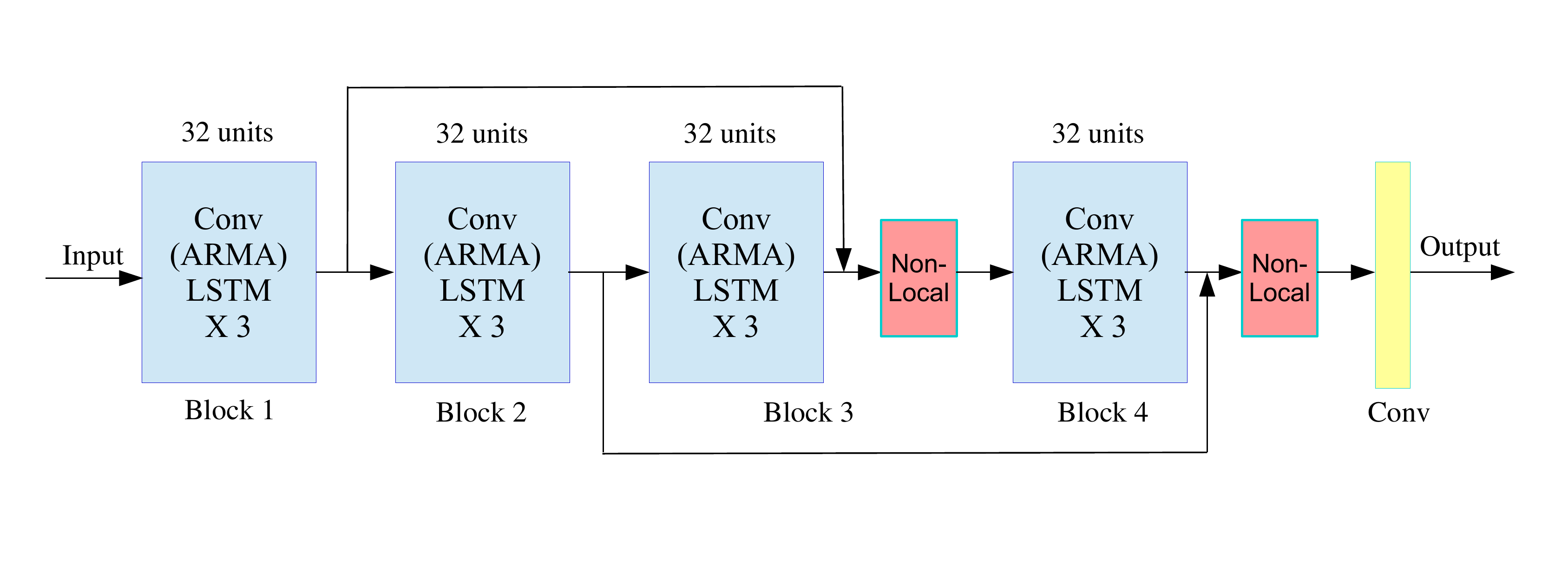}
		\captionof{figure}{Non-Local Conv(ARMA)-LSTM.}
	\label{fig:convlstm_nl}
	\end{minipage}
\end{center}

\paragraph{Training Strategy}
All models are trained using ADAM optimizer~\citep{kingma2014adam}, 
with $\mathcal{L}_{1} + \mathcal{L}_{2}$ loss and for $500$ epochs. 
We set the initial learning rate to $10^{-3}$, and the value for gradient clipping to $3$.
Learning rate decay and scheduled sampling~\citep{bengio2015scheduled} are used to ease training.
Scheduled sampling is started once the model does not improve in $20$ epochs 
(in term of validation loss), 
and the sampling ratio is decreased linearly by $4 \times 10^{-4}$ each epoch 
(i.e.\ scheduling sampling lasts for $250$ epochs).
Learning rate decay is further activated if the validation loss does not drop in 20 epochs, 
and the learning rate is decreased exponentially by $0.98$ every $5$ epochs.
All convolutional layers and \MAlong parts in \ARMA layers
are initialized by Xavier's normalized initializer~\citep{glorot2010understanding}, 
and \ARlong coefficients in \ARMA layers are initialized as zeros 
(i.e.\ each \ARMA layer is initialized as a traditional layer).

\paragraph{Visualization of the Predictions}
We visualize the predictions by different models 
under three moving velocites in \autoref{fig:visual-mnist-v1}, 
\autoref{fig:visual-mnist-v2} and \autoref{fig:visual-mnist-v3}
Notice that the gridding artifacts by dilated convolutions are removed by \ARMA layer:
since each neuron receives information from all pixels in a local region (\autoref{fig:dilated-arma-2d}), 
adjacent neurons are on longer computed from separate sets of pixels.
Moreover, for videos with a higher moving speed, the advantage of our \ARMA layer is more pronounced 
as expected due to \ARMA's ability to expand the \ERF.

\begin{figure}[!htbp]
  \centering
  \addtolength{\tabcolsep}{-5pt}
  \begin{tabular}{cc}
      input & ground truth (top) / predictions \\ 
      \scriptsize $t=1\;\;\;\;\;\;\;\;\;\;2\;\;\;\;\;\;\;\;\;3$ & \scriptsize $4\;\;\;\;\;\;\;\;\;\;\;5\;\;\;\;\;\;\;\;\;\;6\;\;\;\;\;\;\;\;\;\;\;7\;\;\;\;\;\;\;\;\;\;\;\;8\;\;\;\;\;\;\;\;\;9\;\;\;\;\;\;\;\;\;10\;\;\;\;\;\;\;\;\;11\;\;\;\;\;\;\;\;\;12\;\;\;\;\;\;\;\;\;\;13$ \\
      \includegraphics[trim={16.cm 4.7cm 0.cm 0.1cm},clip, height=1cm]{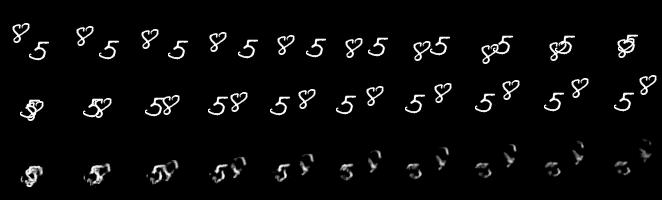} & 
      \includegraphics[trim={0.cm 2.5cm 0.cm 2.2cm},clip, height=1cm, width=10cm]{\fighome/per_frame_vis/w3/cmp_372_1248.jpg} \\ [-0.2em]
        \raisebox{0.4cm}{\scriptsize Traditional ($K_w = 3$)} & 
        \includegraphics[trim={0.cm 0 0.cm 4.7cm},clip, height=1cm, width=10cm]{\fighome/per_frame_vis/w3/cmp_372_1248.jpg} \\ [-0.2em]
       \raisebox{0.4cm}{\scriptsize $2$-dilated ($K_w = 3$)}  & \includegraphics[trim={0.cm 0 0.cm 4.7cm},clip, height=1cm, width=10cm]{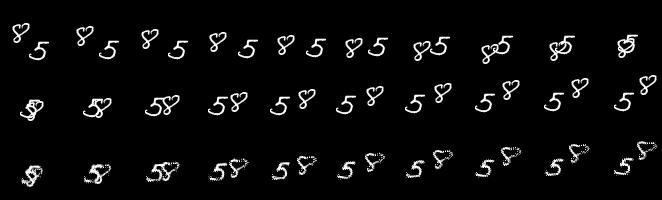} \\ [-0.2em]
      \raisebox{0.4cm}{\scriptsize Traditional ($K_w = 5$)}  & 
      \includegraphics[trim={0.cm 0 0.cm 4.7cm},clip, height=1cm, width=10cm]{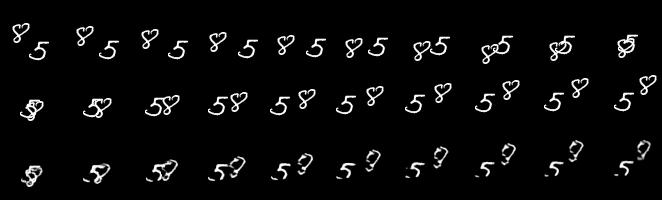} \\ [-0.2em]
      \raisebox{0.4cm}{\scriptsize $2$-dilated \ARMA ($K_w = 3, K_a = 3$)}  &  \includegraphics[trim={0.cm 0 0.cm 4.7cm},clip, height=1cm, width=10cm]{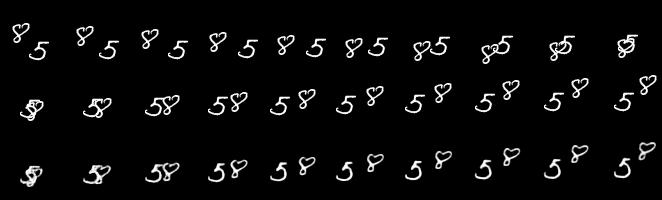} \\ [-0.2em]
      \raisebox{0.4cm}{\scriptsize \ARMA ($K_w = 3, K_a = 3$)} &
      \includegraphics[trim={0.cm 0 0.cm 4.7cm},clip, height=1cm, width=10cm]{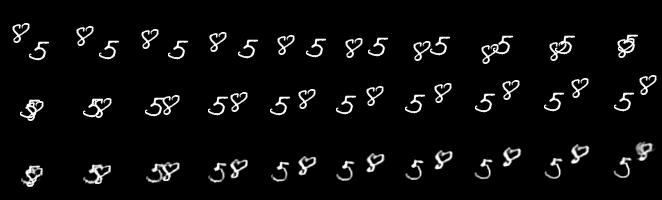} \\ [-0.2em]
  \end{tabular}
  \caption{{\bf Prediction on Moving-MNIST-2 (original speed)}. 
  The first row contains the last $3$ input frames and $10$ ground-truth frames for models to predict.}
  \label{fig:visual-mnist-v1}
\end{figure}
\begin{figure}[!htbp]
  \centering
  \addtolength{\tabcolsep}{-5pt}
  \begin{tabular}{cc}
      input & ground truth (top) / predictions \\ 
      \scriptsize $t=1\;\;\;\;\;\;\;\;\;\;2\;\;\;\;\;\;\;\;\;3$ & \scriptsize $4\;\;\;\;\;\;\;\;\;\;\;5\;\;\;\;\;\;\;\;\;\;6\;\;\;\;\;\;\;\;\;\;\;7\;\;\;\;\;\;\;\;\;\;\;\;8\;\;\;\;\;\;\;\;\;9\;\;\;\;\;\;\;\;\;10\;\;\;\;\;\;\;\;\;11\;\;\;\;\;\;\;\;\;12\;\;\;\;\;\;\;\;\;\;13$ \\
      \includegraphics[trim={16.cm 4.7cm 0.cm 0.1cm},clip, height=1cm]{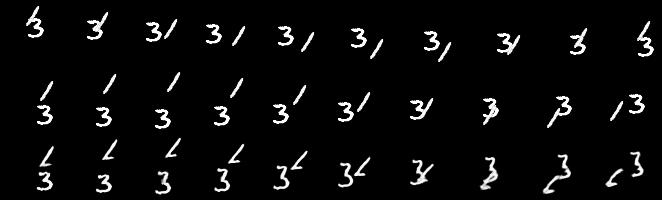} & 
      \includegraphics[trim={0.cm 2.5cm 0.cm 2.2cm},clip, height=1cm, width=10cm]{\fighome/per_frame_vis/v2_w3/cmp_559_2688.jpg} \\ [-0.2em]
        \raisebox{0.4cm}{\scriptsize Traditional ($K_w = 3$)} & 
        \includegraphics[trim={0.cm 0 0.cm 4.7cm},clip, height=1cm, width=10cm]{\fighome/per_frame_vis/v2_w3/cmp_559_2688.jpg} \\ [-0.2em]
       \raisebox{0.4cm}{\scriptsize $2$-dilated ($K_w = 3$)}  & \includegraphics[trim={0.cm 0 0.cm 4.7cm},clip, height=1cm, width=10cm]{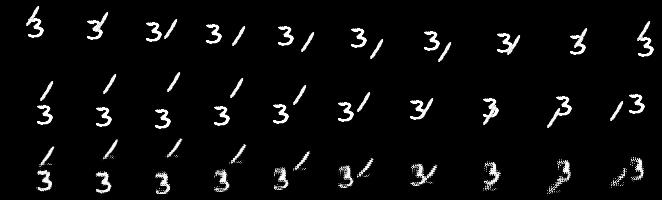} \\ [-0.2em]
      \raisebox{0.4cm}{\scriptsize Traditional ($K_w = 5$)}  & 
      \includegraphics[trim={0.cm 0 0.cm 4.7cm},clip, height=1cm, width=10cm]{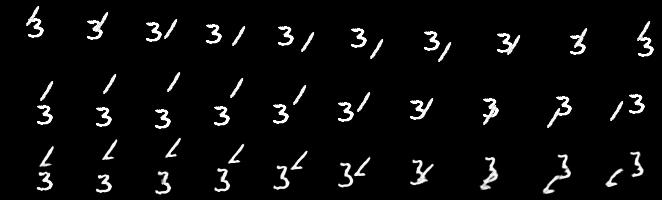} \\ [-0.2em]
      \raisebox{0.4cm}{\scriptsize $2$-dilated \ARMA ($K_w = 3, K_a = 3$)}  &  \includegraphics[trim={0.cm 0 0.cm 4.7cm},clip, height=1cm, width=10cm]{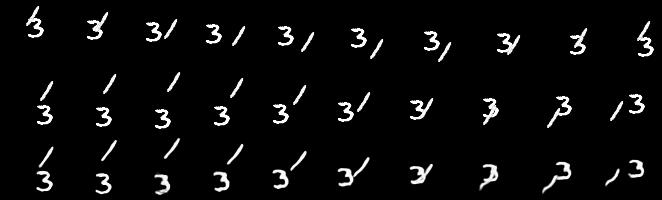} \\ [-0.2em]
      \raisebox{0.4cm}{\scriptsize \ARMA ($K_w = 3, K_a = 3$)} &
      \includegraphics[trim={0.cm 0 0.cm 4.7cm},clip, height=1cm, width=10cm]{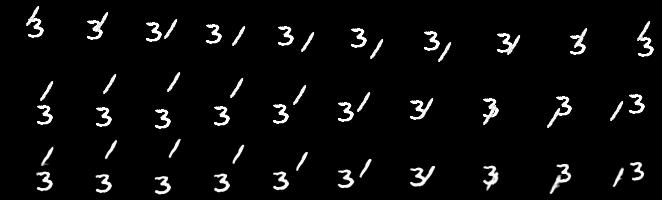} \\ [-0.2em]
  \end{tabular}
  \caption{{\bf Prediction on Moving-MNIST-2 ($2\times$ speed)}. 
  The first row contains the last $3$ input frames and $10$ ground-truth frames for models to predict.}
  \label{fig:visual-mnist-v2}
\end{figure}
\begin{figure}[!htbp]
  \centering
  \addtolength{\tabcolsep}{-5pt}
  \begin{tabular}{cc}
      input & ground truth (top) / predictions \\ 
      \scriptsize $t=1\;\;\;\;\;\;\;\;\;\;2\;\;\;\;\;\;\;\;\;3$ & \scriptsize $4\;\;\;\;\;\;\;\;\;\;\;5\;\;\;\;\;\;\;\;\;\;6\;\;\;\;\;\;\;\;\;\;\;7\;\;\;\;\;\;\;\;\;\;\;\;8\;\;\;\;\;\;\;\;\;9\;\;\;\;\;\;\;\;\;10\;\;\;\;\;\;\;\;\;11\;\;\;\;\;\;\;\;\;12\;\;\;\;\;\;\;\;\;\;13$ \\
      \includegraphics[trim={16.cm 4.7cm 0.cm 0.1cm},clip, height=1cm]{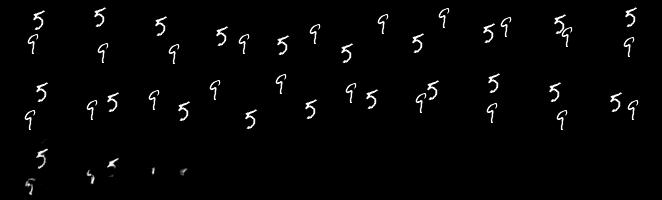} & 
      \includegraphics[trim={0.cm 2.5cm 0.cm 2.2cm},clip, height=1cm, width=10cm]{\fighome/per_frame_vis/v3_w3/cmp_373_672.jpg} \\ [-0.2em]
        \raisebox{0.4cm}{\scriptsize Traditional ($K_w = 3$)} & 
        \includegraphics[trim={0.cm 0 0.cm 4.7cm},clip, height=1cm, width=10cm]{\fighome/per_frame_vis/v3_w3/cmp_373_672.jpg} \\ [-0.2em]
       \raisebox{0.4cm}{\scriptsize $2$-dilated ($K_w = 3$)}  & \includegraphics[trim={0.cm 0 0.cm 4.7cm},clip, height=1cm, width=10cm]{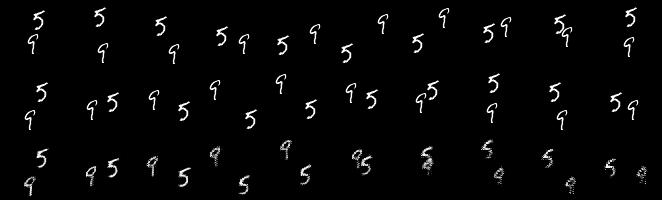} \\ [-0.2em]
      \raisebox{0.4cm}{\scriptsize Traditional ($K_w = 5$)}  & 
      \includegraphics[trim={0.cm 0 0.cm 4.7cm},clip, height=1cm, width=10cm]{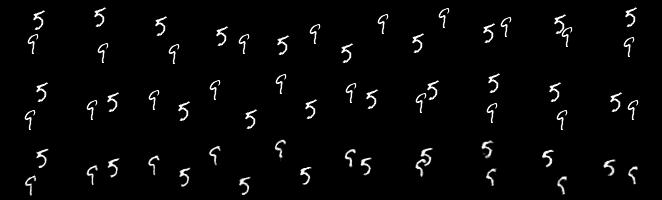} \\ [-0.2em]
      \raisebox{0.4cm}{\scriptsize $2$-dilated \ARMA ($K_w = 3, K_a = 3$)}  &  \includegraphics[trim={0.cm 0 0.cm 4.7cm},clip, height=1cm, width=10cm]{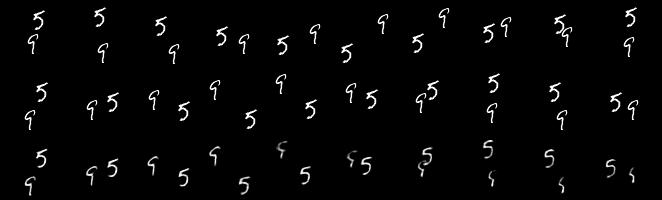} \\ [-0.2em]
      \raisebox{0.4cm}{\scriptsize \ARMA ($K_w = 3, K_a = 3$)} &
      \includegraphics[trim={0.cm 0 0.cm 4.7cm},clip, height=1cm, width=10cm]{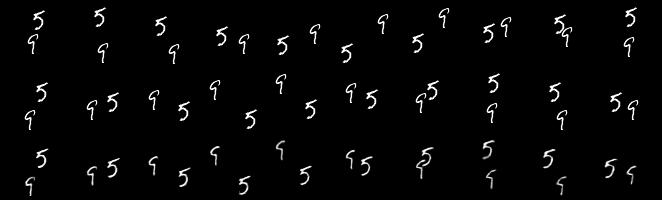} \\ [-0.2em]
  \end{tabular}
  \caption{{\bf Prediction on Moving-MNIST-2 ($3\times$ speed)}. 
  The first row contains the last $3$ input frames and $10$ ground-truth frames for models to predict.}
  \label{fig:visual-mnist-v3}
\end{figure}

\begin{figure}[!htbp]
\centering
	\begin{subfigure}[t]{0.32\linewidth}
		\centering
		\resizebox{\linewidth}{!}{
\begin{tikzpicture}

\definecolor{color0}{rgb}{0.75,0,0.75}
\definecolor{color1}{rgb}{0.75,0.75,0}

\begin{axis}[
legend cell align={left},
legend style={fill opacity=0.8, draw opacity=1, text opacity=1, at={(0.03,0.97)}, anchor=north west, draw=white!80!black},
tick align=outside,
tick pos=both,
x grid style={white!69.0196078431373!black},
xlabel={Time steps},
xmin=0.55, xmax=10.45,
xtick style={color=black},
y grid style={white!69.0196078431373!black},
ylabel={MSE},
ymin=0.00380430939156016, ymax=0.0289598512892555,
ytick style={color=black}
]
\addplot [thick, blue]
table {%
1 0.00652123496872235
2 0.0100063383508518
3 0.0132419686372673
4 0.0160665456615404
5 0.0186731539532828
6 0.0209115678661289
7 0.0230470243637927
8 0.0247791083897477
9 0.0264471098170253
10 0.027816417566633
};
\addlegendentry{Conv (K=3)}
\addplot [thick, color0]
table {%
1 0.00494774311418268
2 0.00739772343112836
3 0.00986781598005796
4 0.0122389285152822
5 0.0144563691053816
6 0.0166966374791428
7 0.0187928312854306
8 0.0206600659461271
9 0.0226683907617271
10 0.0245250804112609
};
\addlegendentry{Conv (K=5)}
\addplot [thick, color1]
table {%
1 0.00600437373640249
2 0.00822394542510202
3 0.0105353858431828
4 0.0126951618585482
5 0.0148393449078336
6 0.0168844166317052
7 0.0188244288656458
8 0.0205146512115556
9 0.0223131540676512
10 0.0239127576792333
};
\addlegendentry{Dilated-Conv}
\addplot [thick, red]
table {%
1 0.00508357862224958
2 0.00735330684695312
3 0.00966384278035658
4 0.0118009843080445
5 0.0138756871553045
6 0.0157729146885524
7 0.0176930602251208
8 0.0193938144318312
9 0.0211894592769858
10 0.0227429178498808
};
\addlegendentry{Dialted ARMA}
\end{axis}

\end{tikzpicture}}
		\vspace{-1em}
		\caption{MSE}
	\end{subfigure}
	\begin{subfigure}[t]{0.32\linewidth}
		\centering
		\resizebox{\linewidth}{!}{
\begin{tikzpicture}

\definecolor{color0}{rgb}{0.75,0,0.75}
\definecolor{color1}{rgb}{0.75,0.75,0}

\begin{axis}[
legend cell align={left},
legend style={fill opacity=0.8, draw opacity=1, text opacity=1, draw=white!80!black},
tick align=outside,
tick pos=both,
x grid style={white!69.0196078431373!black},
xlabel={Time steps},
xmin=0.55, xmax=10.45,
xtick style={color=black},
y grid style={white!69.0196078431373!black},
ylabel={PSNR},
ymin=15.4072208511922, ymax=25.0694480014747,
ytick style={color=black}
]
\addplot [thick, blue]
table {%
1 23.0705585712325
2 20.9342232130136
3 19.5555415182316
4 18.5805068948829
5 17.8301527295829
6 17.264340572185
7 16.7900789961815
8 16.429560895351
9 16.1031652448299
10 15.8464129943869
};
\addlegendentry{Conv (K=3)}
\addplot [thick, color0]
table {%
1 24.6302558582801
2 22.5656354816221
3 21.1575843321174
4 20.1251358947131
5 19.3346894019546
6 18.6163704474188
7 18.031332337355
8 17.5676386685637
9 17.1029723712642
10 16.7126598776118
};
\addlegendentry{Conv (K=5)}
\addplot [thick, color1]
table {%
1 23.1533019994024
2 21.7997738572248
3 20.6767796663404
4 19.7840349321441
5 19.0088100131386
6 18.3603312814825
7 17.8389664787466
8 17.4146548873448
9 17.0023456402438
10 16.6582364344761
};
\addlegendentry{Dilated-Conv}
\addplot [thick, red]
table {%
1 24.2809909873136
2 22.4880512672361
3 21.212902663083
4 20.2465412812261
5 19.4855024133467
6 18.856767248827
7 18.2967598395822
8 17.840709569055
9 17.4064881895702
10 17.0579534705116
};
\addlegendentry{Dialted ARMA}
\end{axis}

\end{tikzpicture}}
		\vspace{-1em}
		\caption{PSNR}
	\end{subfigure}
	\begin{subfigure}[t]{0.32\linewidth}
		\centering
		\resizebox{\linewidth}{!}{
\begin{tikzpicture}

\definecolor{color0}{rgb}{0.75,0,0.75}
\definecolor{color1}{rgb}{0.75,0.75,0}

\begin{axis}[
legend cell align={left},
legend style={fill opacity=0.8, draw opacity=1, text opacity=1, draw=white!80!black},
tick align=outside,
tick pos=both,
x grid style={white!69.0196078431373!black},
xlabel={Time steps},
xmin=0.55, xmax=10.45,
xtick style={color=black},
y grid style={white!69.0196078431373!black},
ylabel={SSIM},
ymin=0.800043787876827, ymax=0.97101302696089,
ytick style={color=black}
]
\addplot [thick, blue]
table {%
1 0.950169888239155
2 0.925337039539072
3 0.903302491182124
4 0.884441045515051
5 0.867380759509769
6 0.85262812360876
7 0.839006238108125
8 0.827731689466968
9 0.816874623946026
10 0.807815116926103
};
\addlegendentry{Conv (K=3)}
\addplot [thick, color0]
table {%
1 0.963241697911614
2 0.946857660019486
3 0.931599894440768
4 0.917261508766772
5 0.904240201587725
6 0.891329641273691
7 0.879514257956326
8 0.869196186742036
9 0.858121331048786
10 0.848116917040197
};
\addlegendentry{Conv (K=5)}
\addplot [thick, color1]
table {%
1 0.954125131500462
2 0.939170499894388
3 0.924214891920886
4 0.9101697422044
5 0.896473999494524
6 0.883508909100867
7 0.871345549545344
8 0.860715095168098
9 0.849181027444626
10 0.838756306901115
};
\addlegendentry{Dilated-Conv}
\addplot [thick, red]
table {%
1 0.961737568692073
2 0.94676358214146
3 0.932512397446549
4 0.919292693932525
5 0.906946648270627
6 0.895800173932082
7 0.884792931718824
8 0.875086907722459
9 0.864897565024478
10 0.856050435361102
};
\addlegendentry{Dialted ARMA}
\end{axis}

\end{tikzpicture}}
		\vspace{-1em}
		\caption{SSIM}
	\end{subfigure}
\caption{{\bf Per-frame performance comparison} of our \ARMA and our dilated \ARMA networks v.s.\ 
the Conv-LSTM, dilated Conv-LSTM baselines for Moving-MNIST-2 (original speed).
Lower MSE values (in $10^{-3}$) or higher PSNR/SSIM values indicate better performance.
}\label{fig:curves_mnist_v1}
\end{figure}
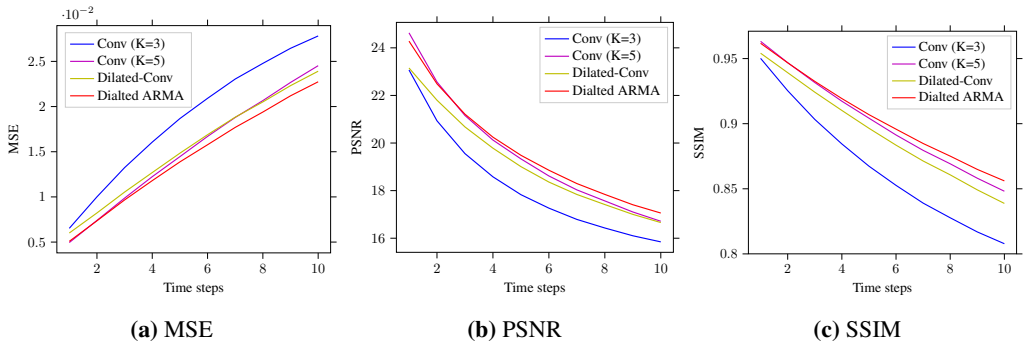

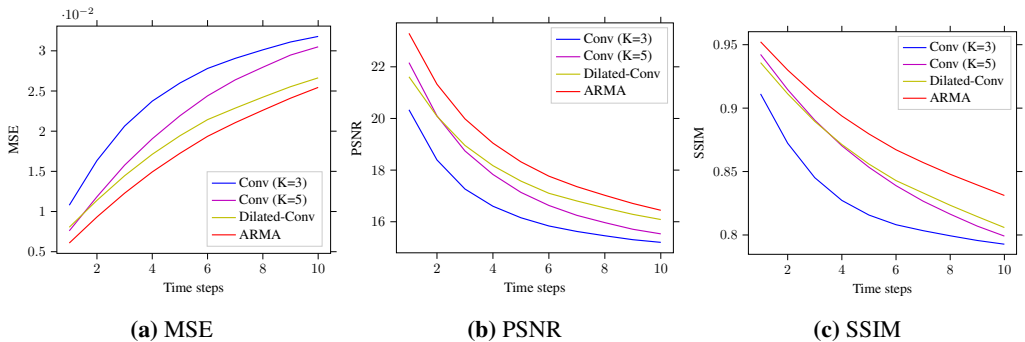
\begin{figure}[!htbp]
\centering
	\begin{subfigure}[t]{0.32\linewidth}
		\centering
		\resizebox{\linewidth}{!}{
\begin{tikzpicture}

\definecolor{color0}{rgb}{0.75,0,0.75}
\definecolor{color1}{rgb}{0.75,0.75,0}

\begin{axis}[
legend cell align={left},
legend style={fill opacity=0.8, draw opacity=1, text opacity=1, at={(0.97,0.03)}, anchor=south east, draw=white!80!black},
tick align=outside,
tick pos=both,
x grid style={white!69.0196078431373!black},
xlabel={Time steps},
xmin=0.55, xmax=10.45,
xtick style={color=black},
y grid style={white!69.0196078431373!black},
ylabel={MSE},
ymin=0.00479341910138081, ymax=0.0330781578033324,
ytick style={color=black}
]
\addplot [thick, blue]
table {%
1 0.0107864300201257
2 0.0163437034951267
3 0.0206329448847203
4 0.0237399560308938
5 0.0259927651699426
6 0.027805294510943
7 0.0290622339846832
8 0.0301248002927173
9 0.0311007138378259
10 0.0317924878623346
};
\addlegendentry{Conv (K=3)}
\addplot [thick, color0]
table {%
1 0.00759613484444936
2 0.0118371152314211
3 0.0157689095042328
4 0.0190548595656603
5 0.0219161871188707
6 0.0243969522031346
7 0.0263841840045438
8 0.0279586293631245
9 0.0294734768209546
10 0.0304991062206873
};
\addlegendentry{Conv (K=5)}
\addplot [thick, color1]
table {%
1 0.00804321915718528
2 0.0113923678289673
3 0.0144524018440108
4 0.0171240772054645
5 0.0194281918540442
6 0.0214363818768275
7 0.0228681370648302
8 0.0242333712177016
9 0.0255353227499321
10 0.0266375791447152
};
\addlegendentry{Dilated-Conv}
\addplot [thick, red]
table {%
1 0.00607908904237861
2 0.00932206836502122
3 0.0122817539004983
4 0.014928843082256
5 0.0172411009643878
6 0.0193584155936797
7 0.0210581403616266
8 0.0226055498148544
9 0.0241043629797249
10 0.0254462825106797
};
\addlegendentry{ARMA}
\end{axis}

\end{tikzpicture}}
		\vspace{-1em}
		\caption{MSE}
	\end{subfigure}
	\begin{subfigure}[t]{0.32\linewidth}
		\centering
		\resizebox{\linewidth}{!}{
\begin{tikzpicture}

\definecolor{color0}{rgb}{0.75,0,0.75}
\definecolor{color1}{rgb}{0.75,0.75,0}

\begin{axis}[
legend cell align={left},
legend style={fill opacity=0.8, draw opacity=1, text opacity=1, draw=white!80!black},
tick align=outside,
tick pos=both,
x grid style={white!69.0196078431373!black},
xlabel={Time steps},
xmin=0.55, xmax=10.45,
xtick style={color=black},
y grid style={white!69.0196078431373!black},
ylabel={PSNR},
ymin=14.7970663791851, ymax=23.7013304100852,
ytick style={color=black}
]
\addplot [thick, blue]
table {%
1 20.3358046238318
2 18.3929349717916
3 17.2679565678697
4 16.6007290012955
5 16.1547019210025
6 15.8366146526201
7 15.6240623742996
8 15.4571878592229
9 15.3052527941432
10 15.2018056533169
};
\addlegendentry{Conv (K=3)}
\addplot [thick, color0]
table {%
1 22.1616766930303
2 20.0991387464632
3 18.7428139550369
4 17.8309665173337
5 17.137566605582
6 16.6307146813541
7 16.2477716491106
8 15.9661645224622
9 15.708841844896
10 15.5305516631727
};
\addlegendentry{Conv (K=5)}
\addplot [thick, color1]
table {%
1 21.6133536531686
2 20.0856310926027
3 18.9556556463418
4 18.1720189530899
5 17.5746974981893
6 17.1019604925028
7 16.8047943367105
8 16.5348566161144
9 16.2857750630565
10 16.0859548884329
};
\addlegendentry{Dilated-Conv}
\addplot [thick, red]
table {%
1 23.2965911359534
2 21.3256140573266
3 19.9851945610225
4 19.0448180761787
5 18.3202670858559
6 17.7651925627277
7 17.3592677394995
8 17.023414442605
9 16.7077897980479
10 16.4440579295277
};
\addlegendentry{ARMA}
\end{axis}

\end{tikzpicture}}
		\vspace{-1em}
		\caption{PSNR}
	\end{subfigure}
	\begin{subfigure}[t]{0.32\linewidth}
		\centering
		\resizebox{\linewidth}{!}{
\begin{tikzpicture}

\definecolor{color0}{rgb}{0.75,0,0.75}
\definecolor{color1}{rgb}{0.75,0.75,0}

\begin{axis}[
legend cell align={left},
legend style={fill opacity=0.8, draw opacity=1, text opacity=1, draw=white!80!black},
tick align=outside,
tick pos=both,
x grid style={white!69.0196078431373!black},
xlabel={Time steps},
xmin=0.55, xmax=10.45,
xtick style={color=black},
y grid style={white!69.0196078431373!black},
ylabel={SSIM},
ymin=0.784756252257109, ymax=0.960242465761391,
ytick style={color=black}
]
\addplot [thick, blue]
table {%
1 0.911199975610055
2 0.872310577953333
3 0.845176211498488
4 0.827252473471397
5 0.815713948491785
6 0.808093607659721
7 0.803402992787796
8 0.799369508545441
9 0.79567408862985
10 0.792732898325486
};
\addlegendentry{Conv (K=3)}
\addplot [thick, color0]
table {%
1 0.942320336572277
2 0.914597646574124
3 0.890484473597789
4 0.870323873415163
5 0.853367040117396
6 0.838889723383443
7 0.826656029898529
8 0.816412207524252
9 0.806988944849334
10 0.799125655788883
};
\addlegendentry{Conv (K=5)}
\addplot [thick, color1]
table {%
1 0.935858477464469
2 0.911253287672169
3 0.889514104097924
4 0.871310735742828
5 0.85590322706479
6 0.842870712059604
7 0.833132496680099
8 0.823423278840363
9 0.814507659581583
10 0.805867933615178
};
\addlegendentry{Dilated-Conv}
\addplot [thick, red]
table {%
1 0.952265819693014
2 0.930008265827574
3 0.910572830945683
4 0.893804403032574
5 0.87964296912722
6 0.867169939893576
7 0.857118263248236
8 0.847849047278918
9 0.839442659539744
10 0.831152231123047
};
\addlegendentry{ARMA}
\end{axis}

\end{tikzpicture}}
		\vspace{-1em}
		\caption{SSIM}
	\end{subfigure}
\caption{{\bf Per-frame performance comparison} of our \ARMA and our dilated \ARMA networks v.s.\ 
the Conv-LSTM, dilated Conv-LSTM baselines for Moving-MNIST-2 ($2\times$ speed).
Lower MSE values (in $10^{-3}$) or higher PSNR/SSIM values indicate better performance.
}
\label{fig:curves_mnist_v2}
\end{figure}

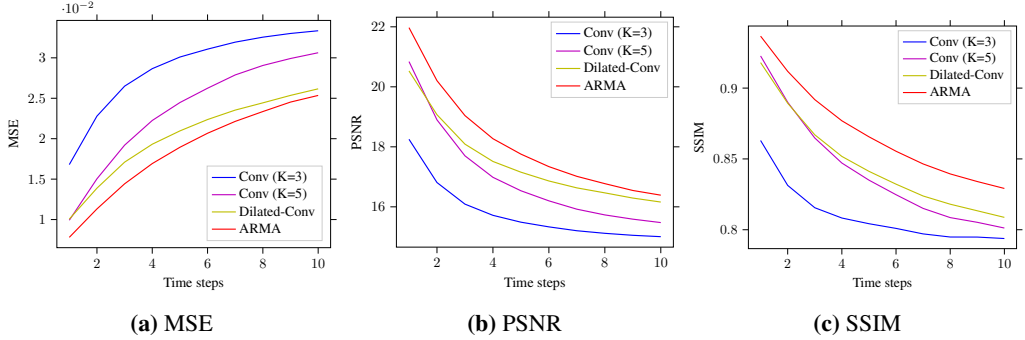
\begin{figure}[!htbp]
\centering
	\begin{subfigure}[t]{0.32\linewidth}
		\centering
		\resizebox{\linewidth}{!}{
\begin{tikzpicture}

\definecolor{color0}{rgb}{0.75,0,0.75}
\definecolor{color1}{rgb}{0.75,0.75,0}

\begin{axis}[
legend cell align={left},
legend style={fill opacity=0.8, draw opacity=1, text opacity=1, at={(0.97,0.03)}, anchor=south east, draw=white!80!black},
tick align=outside,
tick pos=both,
x grid style={white!69.0196078431373!black},
xlabel={Time steps},
xmin=0.55, xmax=10.45,
xtick style={color=black},
y grid style={white!69.0196078431373!black},
ylabel={MSE},
ymin=0.00652216773710135, ymax=0.0346271411125301,
ytick style={color=black}
]
\addplot [thick, blue]
table {%
1 0.0167978574017571
2 0.0227942365823597
3 0.026497364808371
4 0.0286612972740561
5 0.0300988418311176
6 0.0310755494853093
7 0.0319403999405057
8 0.0325564120982272
9 0.0330144927279039
10 0.0333496423227379
};
\addlegendentry{Conv (K=3)}
\addplot [thick, color0]
table {%
1 0.00991549792097928
2 0.0150402957441724
3 0.019199850976699
4 0.022253208693118
5 0.0244702396936723
6 0.0262419756390367
7 0.0278854666767273
8 0.0290558518245545
9 0.0299176172884129
10 0.0306297875639882
};
\addlegendentry{Conv (K=5)}
\addplot [thick, color1]
table {%
1 0.0100508113013425
2 0.0138826643363984
3 0.0170996106747709
4 0.0193185223270004
5 0.020963238326451
6 0.022369192965504
7 0.023536807084826
8 0.0244331834243906
9 0.0253558707757718
10 0.0261662578186909
};
\addlegendentry{Dilated-Conv}
\addplot [thick, red]
table {%
1 0.00779966652689357
2 0.0113153960074166
3 0.0144200199455957
4 0.0169226314208128
5 0.0189386353884064
6 0.0206744597573721
7 0.0221400156010635
8 0.0233577235422015
9 0.0245290607738243
10 0.0253507604429687
};
\addlegendentry{ARMA}
\end{axis}

\end{tikzpicture}}
		\vspace{-1em}
		\caption{MSE}
	\end{subfigure}
	\begin{subfigure}[t]{0.32\linewidth}
		\centering
		\resizebox{\linewidth}{!}{
\begin{tikzpicture}

\definecolor{color0}{rgb}{0.75,0,0.75}
\definecolor{color1}{rgb}{0.75,0.75,0}

\begin{axis}[
legend cell align={left},
legend style={fill opacity=0.8, draw opacity=1, text opacity=1, draw=white!80!black},
tick align=outside,
tick pos=both,
x grid style={white!69.0196078431373!black},
xlabel={Time steps},
xmin=0.55, xmax=10.45,
xtick style={color=black},
y grid style={white!69.0196078431373!black},
ylabel={PSNR},
ymin=14.6568729848007, ymax=22.3207602181737,
ytick style={color=black}
]
\addplot [thick, blue]
table {%
1 18.2515250645959
2 16.8097527473094
3 16.0903088325427
4 15.7167136802862
5 15.4883133834478
6 15.3319498232596
7 15.2030607673073
8 15.117130431583
9 15.0510161218357
10 15.0052314954086
};
\addlegendentry{Conv (K=3)}
\addplot [thick, color0]
table {%
1 20.8409999118963
2 18.8986417044652
3 17.7000403209397
4 16.9837078045373
5 16.535747431295
6 16.1996072670313
7 15.9214776241282
8 15.7303837494095
9 15.5888782456233
10 15.4749361348785
};
\addlegendentry{Conv (K=5)}
\addplot [thick, color1]
table {%
1 20.528999123788
2 19.0703059130758
3 18.0897314679223
4 17.5124377109706
5 17.1550336720475
6 16.8618908697655
7 16.6317018535291
8 16.4653382533782
9 16.2921393142623
10 16.1593246477387
};
\addlegendentry{Dilated-Conv}
\addplot [thick, red]
table {%
1 21.9724017075659
2 20.2051884932366
3 19.0410725736255
4 18.2707889716345
5 17.7570862302428
6 17.3400272964425
7 17.0159746742423
8 16.772532999916
9 16.5452088090973
10 16.3884493298197
};
\addlegendentry{ARMA}
\end{axis}

\end{tikzpicture}}
		\vspace{-1em}
		\caption{PSNR}
	\end{subfigure}
	\begin{subfigure}[t]{0.32\linewidth}
		\centering
		\resizebox{\linewidth}{!}{
\begin{tikzpicture}

\definecolor{color0}{rgb}{0.75,0,0.75}
\definecolor{color1}{rgb}{0.75,0.75,0}

\begin{axis}[
legend cell align={left},
legend style={fill opacity=0.8, draw opacity=1, text opacity=1, draw=white!80!black},
tick align=outside,
tick pos=both,
x grid style={white!69.0196078431373!black},
xlabel={Time steps},
xmin=0.55, xmax=10.45,
xtick style={color=black},
y grid style={white!69.0196078431373!black},
ylabel={SSIM},
ymin=0.786657608165842, ymax=0.943826740286725,
ytick style={color=black}
]
\addplot [thick, blue]
table {%
1 0.863012854788052
2 0.831316892523012
3 0.815543921174206
4 0.808277178885945
5 0.804254455267408
6 0.800950147936577
7 0.797058624048922
8 0.794859674299331
9 0.794820767669791
10 0.793801659625882
};
\addlegendentry{Conv (K=3)}
\addplot [thick, color0]
table {%
1 0.922703053392108
2 0.890007753704784
3 0.864780679998058
4 0.847095551101655
5 0.835016525920503
6 0.824733650238965
7 0.815059594048915
8 0.808567135956327
9 0.805235347156133
10 0.801201960566016
};
\addlegendentry{Conv (K=5)}
\addplot [thick, color1]
table {%
1 0.91799735844923
2 0.889329989729781
3 0.867036754052553
4 0.851775079975178
5 0.84124007945244
6 0.832181576287618
7 0.823823988078459
8 0.818030439605132
9 0.813384076807662
10 0.808728758546808
};
\addlegendentry{Dilated-Conv}
\addplot [thick, red]
table {%
1 0.936682688826685
2 0.911846329869225
3 0.891870562001967
4 0.876925164257258
5 0.865617060845095
6 0.855419149976617
7 0.846450950974051
8 0.839409896112894
9 0.834025502556032
10 0.829149402522451
};
\addlegendentry{ARMA}
\end{axis}

\end{tikzpicture}}
		\vspace{-1em}
		\caption{SSIM}
	\end{subfigure}
\caption{{\bf Per-frame performance comparison} of our \ARMA and our dilated \ARMA networks v.s.\ 
the Conv-LSTM, dilated Conv-LSTM baselines for Moving-MNIST-2 ($3\times$ speed).
Lower MSE values (in $10^{-3}$) or higher PSNR/SSIM values indicate better performance.
}
\label{fig:curves_mnist_v3}
\end{figure}

\subsection{Medical Image Segmentation}
\label{sub:supp_exp_semantic_segmenation}

To demonstrate \ARMA networks' applicability to image segmentation, 
we evaluate it on a challenging medical image segmentation problem. 

\paragraph{Dataset and Metrics}
For all experiments, we use a dataset from ISIC 2018: 
Skin Lesion Analysis Towards Melanoma Detection~\citep{codella2018skin}, 
which can be downloaded online\footnote{\url{https://challenge2018.isic-archive.com/task1/training/}}.
In this task, a model aims to predict a binary mask 
that indicates the location of the primary skin lesion for each input image.
The dataset consists of $2594$ images, and we resize each image to $224 \times 224$. 
We split dataset into training set, validation set and test set with ratios of $0.7$, $0.1$ and $0.2$ respectively.
 All models are evaluated using the following metrics:
$AC = (TP + TN) / (TP + TN + FP + FN)$, $SE = TP / (TP + FN)$, $SP = TN / (TN + FP)$, $PC = TP / (TP+FP)$, $F1 = 2 PC \cdot SE / (PC + SE)$ and $JS = |GT \cap SR| / |GT \cup SR|$, 
where TP stands for true positive, TN for true negative, FP for false positive, FN for false negative, 
GT for ground truth mask and SR for predictive mask.
 
\paragraph{Model Architectures} 
\textbf{(1) Baselines.} 
We use U-Net~\citep{ronneberger2015u} and non-local U-Net~\citep{wang2020non} as baseline models.
U-Net has a contracting path to capture context and a symmetric expanding path enables precise localization.
The network architecture is illustrated in \autoref{fig:unet_architecture}.
Non-local U-Net is equipped with global aggregation blocks based on the self-attention operator 
to aggregate global information without a deep encoder for biomedical image segmentation,
which is illustrated in \autoref{fig:nonlocal_unet_architecture}.
\textbf{(2) Our architectures.} 
We replace all traditional convolution layers with \ARMA layers in U-Net and non-local U-Net.

 \begin{figure}[!htbp]
\centering
	\begin{subfigure}[b]{0.46\textwidth}
	\centering
		\includegraphics[width=\linewidth]{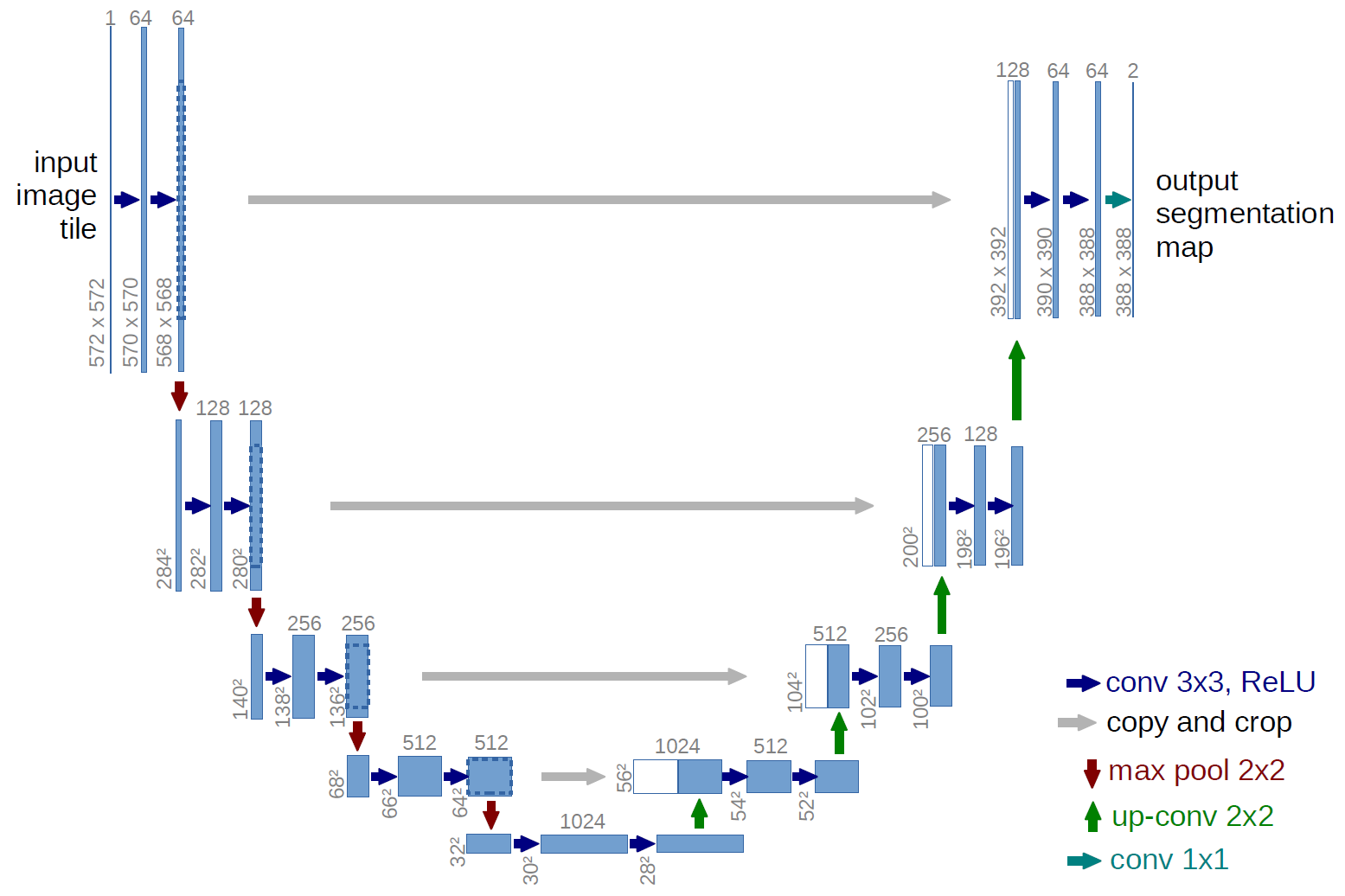}
		\caption{U-Net architecture.}
        \label{fig:unet_architecture}
    	\end{subfigure}
	\begin{subfigure}[b]{0.46\textwidth}
	\centering
		\includegraphics[width=\linewidth]{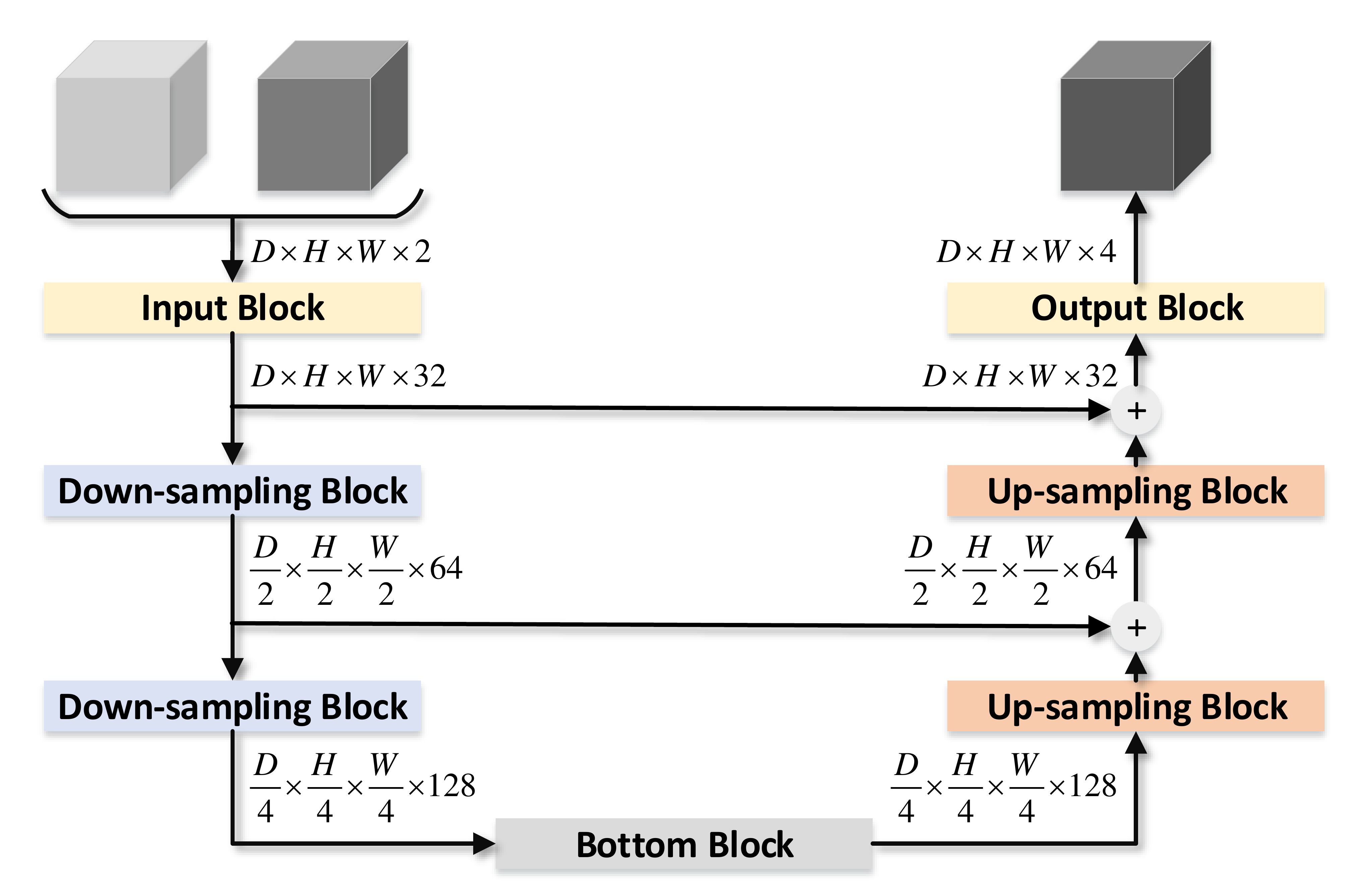}
		\caption{Non-local U-Net architecture.}
	\label{fig:nonlocal_unet_architecture}
	\end{subfigure}
\caption{Backbone architectures.}
\label{fig:unet_att_architecture}
\end{figure}
 
\paragraph{Training Strategy.} 
All models are trained using ADAM optimizer~\citep{kingma2014adam} with binary cross entropy (BCE) loss.
For initial learning rate, we search from $10^{-2}$ to $10^{-5}$ and 
choose $10^{-3}$ for U-Net and $10^{-2}$ for non-local U-Net.
The learning rate is decayed by $0.98$ every epoch.
During training, each image is randomly augmented by rotation, cropping, shifting, color jitter and normalization following the public source code\footnote{\url{https://github.com/LeeJunHyun/Image_Segmentation}}.

\begin{figure}[!htbp]
\centering
	\begin{subfigure}[b]{0.16\textwidth}
	\centering
		\includegraphics[width=\linewidth]{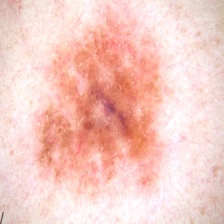}
		\caption*{}
    	\end{subfigure}
	\hfill
	\begin{subfigure}[b]{0.16\textwidth}
    	\centering
		\includegraphics[width=\linewidth]{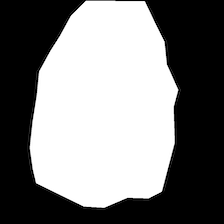}
		\caption*{}
	\end{subfigure}
	\hfill
	\begin{subfigure}[b]{0.16\textwidth}
	\centering
		\includegraphics[width=\linewidth]{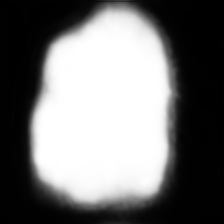}
		\caption*{}
	\end{subfigure}
	\hfill
	\begin{subfigure}[b]{0.16\textwidth}
	\centering
		\includegraphics[width=\linewidth]{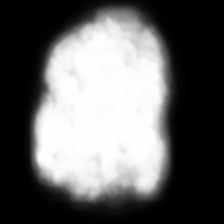}
		\caption*{}
	\end{subfigure}
	\hfill
	\begin{subfigure}[b]{0.16\textwidth}
	\centering
		\includegraphics[width=\linewidth]{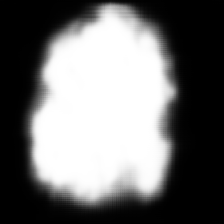}
		\caption*{}
	\end{subfigure}
	\hfill
	\begin{subfigure}[b]{0.16\textwidth}
	\centering
		\includegraphics[width=\linewidth]{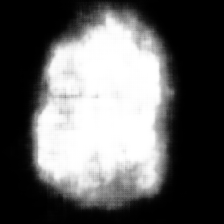}
		\caption*{}
	\end{subfigure}
	\begin{subfigure}[b]{0.16\textwidth}
	\centering
		\includegraphics[width=\linewidth]{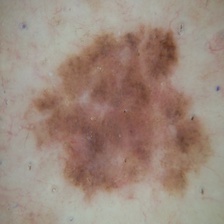}
		\caption*{{\small \begin{tabular}{c} Input \\ image \end{tabular}}}
	\end{subfigure}
	\hfill
	\begin{subfigure}[b]{0.16\textwidth}
	\centering
		\includegraphics[width=\linewidth]{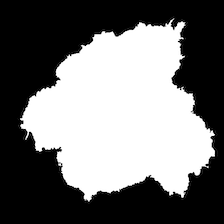}
		\caption*{{\small \begin{tabular}{c} Ground \\ truth \end{tabular}}}
	\end{subfigure}
	\hfill
	\begin{subfigure}[b]{0.16\textwidth}
	\centering
		\includegraphics[width=\linewidth]{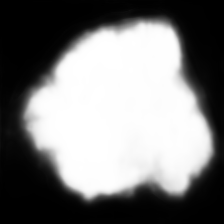}
		\caption*{{\small \begin{tabular}{c} U-net \\ {\bf with \ARMA} \end{tabular}}}
	\end{subfigure}
	\hfill
	\begin{subfigure}[b]{0.16\textwidth}
	\centering
		\includegraphics[width=\linewidth]{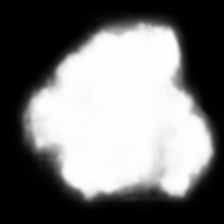}
		\caption*{{\small \begin{tabular}{c} U-net \\ without \ARMA \end{tabular}}}
	\end{subfigure}
	\begin{subfigure}[b]{0.16\textwidth}
	\centering
		\includegraphics[width=\linewidth]{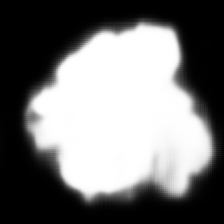}
		\caption*{{\small \begin{tabular}{c} Non-Local U-net \\ {\bf with \ARMA} \end{tabular}}}
	\end{subfigure}
	\hfill
	\begin{subfigure}[b]{0.16\textwidth}
	\centering
		\includegraphics[width=\linewidth]{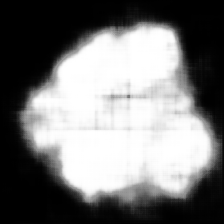}
		\caption*{{\small \begin{tabular}{c} Non-Local U-net \\ without \ARMA \end{tabular}}}
	\end{subfigure}
\caption{Predictive results of U-Net and Non-local U-Net with/without {\ARMA} layers.}
\label{fig:isic-results}
\end{figure}

\subsection{Image Classification}
\label{sub:supp_exp_image_classification}

\paragraph{Model Architectures and Datasets}
We replace the traditional convolutional layers by \ARMA layers 
in three benchmarking architectures for image classification:
AlexNet~\citep{NIPS2012_4824}, VGG-11~\citep{simonyan2014very}, and ResNet-18~\citep{he2016deep}.
We apply our proposed \ARMA networks on CIFAR10 and CIFAR100 datasets.
Both datasets have $50000$ training examples and $10000$ test examples, 
and we use $5000$ examples from the training set for validation (and leave $45000$ examples for training).

\paragraph{Training Strategy}
All models are trained using cross-entropy loss and SGD optimizer
with batch size $128$, learning rate $0.1$, weight decay $0.0005$ and momentum $0.9$.
For CIFAR10, the models are trained for $300$ epochs and we half the learning rate every $30$ epochs.
For CIFAR100, the models are trained for $200$ epochs and we divide the learning rate by 5 
at the $60^{\text{th}}$, $120^{\text{th}}$, $160^{\text{th}}$ epochs.

\paragraph{Results}
The experimental results are summarized in Table~\ref{table:image_classification_results}.
Our results show that \ARMA models achieve comparable or slightly better results than 
the benchmarking architectures.
Replacing traditional convolutional layer with our proposed \ARMA layer 
slightly boosts the performance of VGG-11 and ResNet-18 by 0.01\%-0.1\% in accuracy.
\textbf{Since image classifications tasks do not require convolutional layers to have large receptive field,
the learned \ARlong coefficients are highly concentrated around 0 as shown in \autoref{fig:histogram}.}
Consequently, \ARMA networks effectively reduce to traditional convolutional neural networks 
and therefore achieve comparable results.

\begin{table*}[!htbp]
    \centering
    \setlength\aboverulesep{0pt}\setlength\belowrulesep{0pt}
    \setcellgapes{3pt}\makegapedcells
    \scalebox{0.8}{
    \begin{tabular}{c| cc cc cc}
    \hline
     &
     \multicolumn{2}{c}{AlexNet} &
     \multicolumn{2}{c}{VGG-11}  & 
     \multicolumn{2}{c}{ResNet-18} \\
    \cmidrule(r){2-3} \cmidrule(r){4-5} \cmidrule(r){6-7} 
     &
    Conv. & \ARMA & Conv. & \ARMA & Conv. & \ARMA \\
    \hline
    CIFAR10 & $86.30 \pm 0.29$ & $85.67 \pm 0.19$ & $91.57 \pm 0.59$ & $91.57 \pm 0.73$ & $95.01 \pm 0.15$ & $95.07 \pm 0.13$ \\
    CIFAR100& $58.99 \pm 0.37$ & $57.43 \pm 0.24$ & $68.25\pm 0.11$ & $68.36 \pm 1.67$ & $73.71 \pm 0.23$ & $73.72 \pm 0.52$ \\
    \hline
    \end{tabular}}
    \caption{{\bf Image classification} on CIFAR10 and CIFAR100. 
    The reported accuracies (\%) and their standard deviations are computed from $10$ runs with different seeds. 
    Since image classifications tasks do not require convolutional layers to have large receptive field,
the learned \ARlong coefficients are highly concentrated around $0$ as shown in \autoref{fig:histogram}.
Consequently, \ARMA networks effectively reduce to traditional {\CNN}s
and therefore achieve comparable results.}
    \label{table:image_classification_results}
\end{table*}

\section{Analysis of \ERFLONG (\ERF)}
\label{app:receptive}

In this section, we prove the Theorem~\ref{thm:erf-arma} in \autoref{sub:erf}. Throughout this section, 
we use $\vectorSup{a}{\ell} = \{\cdots, \vectorInd{a}{\ell}{-1}, \vectorInd{a}{\ell}{0}, \vectorInd{a}{\ell}{1}, \cdots \}$ 
to denote the $\ell^{\text{th}}$ layer's \ARlong coefficients, and $\vectorSup{w}{\ell} = \{\cdots, \vectorInd{w}{\ell}{-1}, \vectorInd{w}{\ell}{0}, \vectorInd{w}{\ell}{1}, \cdots \}$ the $\ell^{\text{th}}$ layer's \ARlong coefficients.

\subsection{\ERF of General Linear Convolutional Networks}
\label{sub:}

The proof of is Theorem~\ref{thm:erf-arma} based on the following theorem on
linear convolutional networks~\citep{luo2016understanding}, which includes both
\CNN and \ARMA networks as special cases.

\begin{theorem}[{\bf \ERF of linear convolutional networks with infinite horizon}]
\label{thm:erf-cnn}
Consider an $L$-layer linear convolutional network
(without activation and pooling),
where its $\ell^\text{th}$-layer computes a weighted-sum of its input
$\vectorInd{y}{\ell}{i} = \sum_{p = -\infty}^{+\infty} \vectorInd{w}{\ell}{p} \vectorInd{y}{\ell - 1}{i - p}$.
Suppose the weights are non-negative $\vectorInd{w}{\ell}{p} \geq 0$ and normalized at each layer
$\sum_{p = -\infty}^{+\infty} \vectorInd{w}{\ell}{p} = 1, 1 \leq \ell \leq [L]$,
the network has an \ERF radius as
\begin{equation}
r^2(\text{ERF})= \sum_{\ell = 1}^{L} \left[ \sum_{p = -\infty}^{+\infty} p^2  \vectorInd{w}{\ell}{p} - 
\left( \sum_{p = -\infty}^{+\infty} p \vectorInd{w}{\ell}{p} \right)^2 \right]
\label{eq:erf-cnn}
\end{equation}
Furthermore, the \ERF converges to a Gaussian density function when $L$ tends to infinity.
\end{theorem}

\begin{proof}
In this linear convolutional network, the gradient maps can be computed with chain rule as
\begin{equation}
\vectorSub{g}{i, \bm{:}} = {\vectorSup{w}{1}}^{\top} \ast {\vectorSup{w}{2}}^{\top} 
\cdots \ast {\vectorSup{w}{L}}^{\top}, ~ \forall i \in \Z
\label{eq:gradient-map-analytic}
\end{equation}
where ${w}^{\ell \top}$ denotes the reversed sequence of $\vectorSup{w}{\ell}$. 
Notice that 
{\bf (1)} The gradient maps do not depend on the input, i.e.\ they are data-independent;
{\bf (2)} The gradient maps are identical across different locations in the output.
Consequently, the ERF is equal to any one gradient map above
\begin{equation}
\text{ERF} = {\vectorSup{w}{1}}^{\top} \ast {\vectorSup{w}{2}}^{\top} 
\cdots \ast {\vectorSup{w}{L}}^{\top}
\label{eq:erf-cnn-analytic}
\end{equation}
The remaining part of the proof makes use of {\em probabilistic method}, 
which interprets the operation at each layer as a discrete random variable. 
Since the weights at each layer are non-negative and normalized, 
they can be treated as values of a probability mass function.
Concretely, we construct $L$ independent random variables $\{\scalarSup{W}{1}, \cdots, \scalarSup{W}{L}\}$
such that $\P[\scalarSup{W}{\ell} = p] = \vectorInd{w}{\ell}{-p}$.
Similarly, we introduce a random variable $\scalarSup{S}{L}$ to represent the \ERF, 
i.e.\ $\P[\scalarSup{S}{L} = p] = \vectorSub{\text{ERF}}{p}$. 
As a result, the ERF radius is equal to 
standard deviation of $\scalarSup{S}{L}$, or equivalently $r^2(\text{ERF}) = \V[\scalarSup{S}{L}]$. 

Recall that {\em addition of independent random variables results in 
convolution of their probability mass functions}, \autoref{eq:erf-cnn-analytic}
implies that $\scalarSup{S}{L}$ is an addition of all $\scalarSup{W}{\ell}$'s, i.e.\
$\scalarSup{S}{L} = \sum_{\ell = 1}^{L} \scalarSup{W}{\ell}$.
Therefore, the variance of $\scalarSup{S}{L}$ is equal to
a summation of the variances for $\scalarSup{W}{\ell}$'s. Thus,
\begin{align}
\V[\scalarSup{S}{L}] & = \sum_{\ell = 1}^{L} \V[\scalarSup{W}{\ell}] 
= \sum_{\ell = 1}^{L} \left[ \E[(\scalarSup{W}{\ell})^2] - \E[\scalarSup{W}{\ell}]^2 \right] \\
& = \sum_{\ell = 1}^{L} \left[ \sum_{p = -\infty}^{+\infty} p^2  \vectorInd{w}{\ell}{p} -
\left( \sum_{p = -\infty}^{+\infty} p \vectorInd{w}{\ell}{p} \right)^2 \right]
\end{align}
which proves the \autoref{eq:erf-cnn}.
Furthermore, the {\em Lyapunov central limit theorem} shows that 
$(\scalarSup{S}{L} - \E[\scalarSup{S}{L}]) / \V[\scalarSup{S}{L}]$ converges to
a standard normal random variable if $L$ tends to infinity
\begin{equation}
\frac{\scalarSup{S}{L} - \E[\scalarSup{S}{L}]}{\sqrt{\V[\scalarSup{S}{L}]}} 
= \frac{\sum_{\ell = 1}^{L} \left(\scalarSup{W}{\ell} - \E[\scalarSup{W}{\ell}] \right)}
{\sqrt{\sum_{\ell = 1}^{L} \V[\scalarSup{W}{\ell}]}}\xrightarrow[]{D} \mathcal{N}(0, 1)
\end{equation}
that is, the \ERF function is approximately Gaussian
when the number of layers $L$ is large enough.
\end{proof}

\subsection{\ERF of Traditional {\CNN}s ($\vectorInd{a}{\ell}{1} = -\scalarSup{a}{\ell}  = 0$)}
As a warmup, we first provide a proof for 
the special case of traditional \CNN where $\scalarSup{a}{\ell} = 0$ for all layers.
For reference, we list the first two cases of {\em Faulhaber's formula}:
\begin{subequations}
\begin{gather}
\sum_{p = 0}^{K - 1} p = \frac{K(K-1)}{2}
\label{eq:faulhaber-sum-1} \\
\sum_{p = 0}^{K - 1} p^2 = \frac{K(K-1)(2K-1)}{6} 
\label{eq:faulhaber-sum-2}
\end{gather}
\end{subequations}
\label{sub:erf-conv}

\begin{proof}
In this special case, the \ERF radius can be obtained by plugging $\vectorInd{w}{\ell}{p} = 1 / \scalarSup{K}{\ell}$ 
for $p = 0, \scalarSup{d}{\ell}, \cdots, \scalarSup{d}{\ell} (\scalarSup{K}{\ell} - 1)$ into \autoref{eq:erf-cnn}.
\begin{align}
r^2(\text{ERF})
& = \sum_{\ell = 1}^{L} \left[ \sum_{p = 0}^{\scalarSup{K}{\ell} - 1} \frac{\left(p \scalarSup{d}{\ell} \right)^2}{\scalarSup{K}{\ell}} 
- \left( \sum_{p = 0}^{\scalarSup{K}{\ell} - 1} \frac{p \scalarSup{d}{\ell}}{\scalarSup{K}{\ell}} \right)^2 \right] \\
& = \sum_{\ell = 1}^{L} \left[ \frac{{\scalarSup{d}{\ell}}^2}{\scalarSup{K}{\ell}} \frac{\scalarSup{K}{\ell}\left(\scalarSup{K}{\ell} - 1\right)\left(2\scalarSup{K}{\ell}-1\right)}{6} - \left( \frac{\scalarSup{d}{\ell}}{\scalarSup{K}{\ell}} \frac{\scalarSup{K}{\ell}\left(\scalarSup{K}{\ell}-1\right)}{2} \right)^2 \right] \\
& = \sum_{\ell = 1}^{L} \frac{{{\scalarSup{d}{\ell}}^2}\left({\scalarSup{K}{\ell}}^2 - 1 \right)}{12}
\label{eq:erf-conv}
\end{align}
where the infinite series are computed using
\autoref{eq:faulhaber-sum-1} and \autoref{eq:faulhaber-sum-2}. 
Taking square root on both sides completes the proof for the special case of {\CNN}s.
\end{proof}

\textbf{\ERF Analysis of {\CNN}s.}
If we further assume that all layers are identical, 
i.e.\ $\scalarSup{K}{\ell} = K, \scalarSup{d}{\ell} = d$ for $1 \leq \ell \leq L$, 
we can simplify \autoref{eq:erf-conv} as
\begin{equation}
r(\text{ERF}) = \sqrt{L} \cdot \sqrt{\frac{d^2 (K^2 - 1)}{12}} = O\left(d K \sqrt{L} \right)
\end{equation}
That is, the \ERF radius grows linearly with kernel size $K$ and dilation $d$, 
but sub-linearly with the number of layers $L$ in the linear network.

\subsection{\ERF of \ARMA networks ($\vectorInd{a}{\ell}{1} = - \scalarSup{a}{\ell} \leq 0$)}
In the part, we provide a proof for general {\ARMA} networks where $\scalarSup{a}{\ell} \leq 0$.
In sketch, the proof consists of three steps:
{\bf (1)} we introduce {\em inverse convolution}
and convert each {\ARMA} model to a {\MAlong} model:
\(\myvector{a} \ast \myvector{y} 
= \myvector{w} \ast \myvector{x} \implies \myvector{y} 
= \myvector{f} \ast \myvector{x}\),
where $\myvector{f}$ represents a convolution with {\em infinite} number of coefficients, 
$\myvector{x}$ and $\myvector{y}$ are the input and output of the model respectively.
{\bf (2)} We derive the {\em \MGFlong} (\MGF) of the {\MAlong} coefficients from the first step,
and use the functions to compute the first and second moments.
{\bf (3)} We plug the moments from the second step 
into \autoref{eq:erf-cnn} to obtain \autoref{eq:erf-arma}.

\begin{definition}[{\bf Inverse convolution}]
\label{def:inverse-convolution}
Given a convolution (with coefficients) $\myvector{a}$,
its inverse convolution $\inverse{\myvector{a}}$ is defined such that
$\myvector{a} \ast \inverse{\myvector{a}} = \inverse{\myvector{a}} \ast \myvector{a} = \myvector{\delta}$ 
is an identical mapping, i.e.\
\begin{equation}
\sum_{p = -\infty}^{+\infty} \vectorSub{a}{i - p} \vectorSub{\inverse{a}}{p} 
= \vectorSub{\delta}{i} = \left\{
	\begin{aligned}
		1 \qquad i = 0 \\
		0 \qquad i \ne 0
	\end{aligned}
\right.
\label{eq:inverse-convolution}
\end{equation}
\end{definition}

{\em Remark:} The inverse convolution does not exist for any convolution $\myvector{a}$.
A necessary and sufficient condition for invertibility of $\myvector{a}$ is 
that its Fourier transform is non-zero everywhere~\citep{oppenheim2014discrete}.

\begin{definition}[{\bf Moments and \MGFLONG, \MGF}]
\label{def:moment-generating-function}
Given a convolution (with coefficients) $\myvector{f}$, its $i^{\text{th}}$ moment is defined as
\begin{equation}
\scalarSub{M}{i}(\myvector{f}) = \sum_{p = -\infty}^{+\infty} \vectorSub{f}{p} p^i
\label{eq:moments}
\end{equation}
Furthermore, we define the {\MGFlong} of the coefficients $\myvector{f}$ as
\begin{equation}
\scalarSub{M}{\myvector{f}}(\lambda) = \sum_{p = -\infty}^{+\infty} \vectorSub{f}{p} e^{\lambda p}
\label{eq:moment-generating-function}
\end{equation}
The name ``moment generating'' comes from the fact that
\begin{equation}
\scalarSub{M}{i}(\myvector{f}) = 
\left. \frac{d^{i} \scalarSub{M}{\myvector{f}}(\lambda)}{{d \lambda}^{i}} \right\rvert_{\lambda = 0}
\label{eq:mgf-to-moments}
\end{equation}
\end{definition}

{\em Remark:} Since {\MGFlong} (\MGF) could be interpreted as a real-valued {\DTFTlong} (\DTFT), 
the properties of {\MGF} are very similar to the ones of {\DTFT}.
In particular, the convolution theorem also holds for {\MGF},
i.e.\ $\scalarSub{M}{\myvector{f} \ast \myvector{g}}(\lambda) =
\scalarSub{M}{\myvector{f}}(\lambda) \scalarSub{M}{\myvector{g}}(\lambda)$.
If two convolutions $\myvector{a}$ and $\inverse{\myvector{a}}$ are inverse to each other,
we have $\scalarSub{M}{\myvector{a}}(\lambda) \scalarSub{M}{\inverse{\myvector{a}}}(\lambda) = 1$.

Now we are ready to prove Theorem~\ref{thm:erf-arma} using Theorem~\ref{thm:erf-cnn} and
Definitions~\ref{def:inverse-convolution} and~\ref{def:moment-generating-function}.
\begin{proof}
Let $\vectorSup{f}{\ell} = \vectorSup{\inverse{a}}{\ell} \ast \vectorSup{w}{\ell}$, we have
\begin{equation}
\begin{aligned}
\vectorSup{y}{\ell} & = \myvector{\delta} \ast \vectorSup{y}{\ell} 
= \left( \vectorSup{\inverse{a}}{\ell} \ast \vectorSup{a}{\ell} \right) \ast \vectorSup{y}{\ell}  
= \vectorSup{\inverse{a}}{\ell} \ast \left( \vectorSup{a}{\ell} \ast \vectorSup{y}{\ell} \right) \\
& = \vectorSup{\inverse{a}}{\ell} \ast \left( \vectorSup{w}{\ell} \ast \vectorSup{y}{\ell - 1} \right)
= \left( \vectorSup{\inverse{a}}{\ell} \ast \vectorSup{w}{\ell} \right) \ast \vectorSup{y}{\ell - 1}
= \vectorSup{f}{\ell} \ast \vectorSup{y}{\ell - 1}
\end{aligned}
\end{equation}
where each $\vectorSup{f}{\ell}$ has infinite number of coefficients. 
We denote the {\MGF} of $\vectorSup{f}{\ell}$ as $\scalarSub{M}{\vectorSup{f}{\ell}}$, 
and its first and second moments as $\scalarSub{M}{1}(\vectorSup{f}{\ell})$ and $\scalarSub{M}{2}(\vectorSup{f}{\ell})$.
With the moments of $\vectorSup{f}{\ell}$, we can rewrite \autoref{eq:erf-cnn} in Theorem~\ref{thm:erf-cnn} as
\begin{equation}
r^2(\text{ERF}) = \sum_{\ell = 1}^{L} 
\left[ \scalarSub{M}{2}(\vectorSup{f}{\ell}) - \left( \scalarSub{M}{1}(\vectorSup{f}{\ell}) \right)^2 \right] 
\label{eq:erf-var-app}
\end{equation}
The remaining part is to compute $\scalarSub{M}{\vectorSup{f}{\ell}}$ for each $\vectorSup{f}{\ell}$, 
with which $\scalarSub{M}{1}(\vectorSup{f}{\ell})$ and $\scalarSub{M}{2}(\vectorSup{f}{\ell})$ are generated.
Notice that $\vectorSup{f}{\ell} = \vectorSup{\inverse{a}}{\ell} \ast \vectorSup{w}{\ell}$ 
is a convolution between $\vectorSup{\inverse{a}}{\ell}$ and $\vectorSup{w}{\ell}$, we have
\begin{align}
\scalarSub{M}{\vectorSup{f}{\ell}}(\lambda) &
= \scalarSub{M}{\vectorSup{\inverse{a}}{\ell}}(\lambda) \scalarSub{M}{\vectorSup{w}{\ell}}(\lambda)
= \frac{\scalarSub{M}{\vectorSup{w}{\ell}}(\lambda)}{\scalarSub{M}{\vectorSup{a}{\ell}}(\lambda)} \\
& = \frac{1}{1 - \scalarSup{a}{\ell} e^{\lambda}} \sum_{p = 0}^{\scalarSup{K}{\ell} - 1} 
\frac{1 - \scalarSup{a}{\ell}}{\scalarSup{K}{\ell}} e^{\lambda p \scalarSup{d}{\ell}}
\end{align}
where the first equation uses the property that $\scalarSub{M}{\vectorSup{a}{\ell}}(\lambda) 
\scalarSub{M}{\vectorSup{\inverse{a}}{\ell}}(\lambda) = 1$ for any $\lambda$.
The first moment $\scalarSub{M}{1}(\vectorSup{f}{\ell})$ is therefore
\begin{align}
& \scalarSub{M}{1}(\vectorSup{f}{\ell}) =
\left. \frac{d \scalarSub{M}{\vectorSup{f}{\ell}}(\lambda)}{d \lambda} \right\rvert_{\lambda = 0} \\
= ~ & \bigg\{ \frac{\scalarSup{a}{\ell}}{\left(1 - \scalarSup{a}{\ell} \lambda \right)^2} 
\sum_{p = 0}^{\scalarSup{K}{\ell}} \frac{1 - \scalarSup{a}{\ell}}{\scalarSup{K}{\ell}} e^{\lambda p \scalarSup{d}{\ell}} + 
\frac{1}{1 - \scalarSup{a}{\ell} e^{\lambda}} \sum_{p = 0}^{\scalarSup{K}{\ell} - 1}
\frac{1 - \scalarSup{a}{\ell}}{\scalarSup{K}{\ell}} p \scalarSup{d}{\ell} e^{\lambda p \scalarSup{d}{\ell}} \bigg\}_{\lambda = 0} \\
= ~ & \frac{\scalarSup{a}{\ell}}{1 - \scalarSup{a}{\ell}} + \frac{\scalarSup{d}{\ell}}{\scalarSup{K}{\ell}}
\left( \sum_{p = 0}^{\scalarSup{K}{\ell} - 1} p \right) \\
= ~ & \frac{\scalarSup{a}{\ell}}{1 - \scalarSup{a}{\ell}} + \frac{\scalarSup{d}{\ell} \left( \scalarSup{K}{\ell} - 1\right)}{2}
\label{eq:first-moment}
\end{align}
where the last equation makes use of \autoref{eq:faulhaber-sum-1}.
Similarly, the second moment $\scalarSub{M}{2}(\vectorSup{f}{\ell})$ is
\begin{align}
& \scalarSub{M}{2}(\vectorSup{f}{\ell}) = 
\left. \frac{d^2 \scalarSub{M}{\vectorSup{f}{\ell}}(\lambda)}{{d \lambda}^2} \right\rvert_{\lambda = 0} \\
= ~ & \begin{aligned} \bigg\{ \frac{{\scalarSup{a}{\ell}}^2}{\left( 1 - \scalarSup{a}{\ell} \right)^3} & 
\sum_{p = 0}^{\scalarSup{K}{\ell}} \frac{1 - \scalarSup{a}{\ell}}{\scalarSup{K}{\ell}} e^{\lambda p \scalarSup{d}{\ell}}
+ \frac{2 \scalarSup{a}{\ell}}{\left(1 - \scalarSup{a}{\ell} e^{\lambda} \right)^2} \sum_{p = 0}^{\scalarSup{K}{\ell} - 1}
\frac{1 - \scalarSup{a}{\ell}}{\scalarSup{K}{\ell}} p \scalarSup{d}{\ell} e^{\lambda p \scalarSup{d}{\ell}} \\
& \quad + \frac{1}{1 - \scalarSup{a}{\ell} e^{\lambda}} \sum_{p = 0}^{\scalarSup{K}{\ell} - 1}
\frac{1 - \scalarSup{a}{\ell}}{\scalarSup{K}{\ell}} \left( p \scalarSup{d}{\ell} \right)^2 e^{\lambda p \scalarSup{d}{\ell}} \bigg\}_{\lambda = 0} \end{aligned} \\
& = \left( \frac{\scalarSup{a}{\ell}}{1 - \scalarSup{a}{\ell}}\right)^2 + 
\frac{2 \scalarSup{a}{\ell}}{1 - \scalarSup{a}{\ell}} \frac{\scalarSup{d}{\ell}}{\scalarSup{K}{\ell}} \left( \sum_{p = 0}^{\scalarSup{K}{\ell} - 1} p \right) + \frac{{\scalarSup{d}{\ell}}^2}{\scalarSup{K}{\ell}} \left( \sum_{p = 0}^{\scalarSup{K}{\ell} - 1} p^2 \right) \\
& =   \left( \frac{\scalarSup{a}{\ell}}{1 - \scalarSup{a}{\ell}}\right)^2 + \frac{2 \scalarSup{a}{\ell}}{1 - \scalarSup{a}{\ell}} \frac{\scalarSup{d}{\ell} \left( \scalarSup{K}{\ell} - 1\right)}{2} + \frac{{\scalarSup{d}{\ell}}^2 \left(\scalarSup{K}{\ell} - 1\right) \left(2 \scalarSup{K}{\ell} - 1\right)}{6}
\label{eq:second-moment}
\end{align}
Plugging \autoref{eq:first-moment} and \autoref{eq:second-moment} into \autoref{eq:erf-var-app}, we have
\begin{equation}
r^2(\text{ERF}) = \sum_{\ell = 1}^{L} \left[ \frac{{\scalarSup{d}{\ell}}^2 \left( \scalarSup{K}{\ell} - 1 \right)^2}{12}  
+ \frac{ \scalarSup{a}{\ell}}{\left(1 - \scalarSup{a}{\ell}\right)^2} \right]
\label{eq:erf-arma-app}
\end{equation}
Taking square root on both sides completes the proof.
\end{proof}

\textbf{\ERF Analysis of \ARMA Networks.} If we assume all layers are identical, 
i.e.\ $\scalarSup{K}{\ell} = K, \scalarSup{d}{\ell} = d, \scalarSup{a}{\ell} = a$ 
for $1 \leq \ell \leq L$, we can simplify \autoref{eq:erf-arma-app} as
\begin{equation}
r(\text{ERF}) = \sqrt{L} \cdot \sqrt{\frac{d^2 (K^2 - 1)}{12} + \frac{a}{(1 - a)^2}} = O\left( \sqrt{L} \max\left(d K, \frac{\sqrt{a}}{1 - a} \right) \right)
\end{equation}
The \ERF radius is dominated by the \AR coefficient when $a \lessapprox 1$ regardless of kernel size $K$ and dilation $d$.
The radius still grows sub-linearly with the number of layers $L$ in the linear network.

\section{Computation of {\ARMA} Layers}
\label{app:computation}

In the section, we first derive the backpropagation rules in Theorem~\ref{thm:ar-backprop}.
We then show how to efficiently compute both forward and backward passes 
in \ARMA layer using {\em \FFTLONG}.

\subsection{Backpropagation in {\ARMA} models}
\label{sub:backprop}

In this part, we will prove a general theorem for backpropagation in \ARMA models. 
To keep the notations simple, we derive the backpropagation equations
for \ARMA models with one dimension input/output and one channel.
However, the techniques in the proof can be trivially extended to general \ARMA models 
with high-dimensional input/output with multiple channels.

\begin{theorem} [{\bf Backpropagation in an \ARMA model}]
\label{thm:arma-backprop}
Consider an \ARMA model $\myvector{a} \ast \myvector{y} = \myvector{w} \ast \myvector{x}$,
where $\myvector{a}$ and $\myvector{w}$ are the sequences 
of moving-average and autoregressive coefficients respectively,
the gradients $\{\gradientInline{\myvector{x}}, \gradientInline{\myvector{w}}, \gradientInline{\myvector{a}} \}$ 
can be computed from $\gradientInline{\myvector{y}}$ with the following equations:
\begin{subequations}
\begin{align}
\adjoint{\myvector{a}} \ast \gradient{\myvector{x}} & = 
\adjoint{\myvector{w}} \ast \gradient{\myvector{y}}
\label{eq:arma-model-backprop-x} \\
- \adjoint{\myvector{a}} \ast \gradient{\myvector{a}} & = 
\adjoint{\myvector{y}} \ast \gradient{\myvector{y}} 
\label{eq:arma-model-backprop-a} \\
\adjoint{\myvector{a}} \ast \gradient{\myvector{w}} & = 
\adjoint{\myvector{x}} \ast \gradient{\myvector{y}}
\label{eq:arma-model-backprop-w} 
\end{align}
\end{subequations}
where $\adjoint{\myvector{a}}$, $\adjoint{\myvector{w}}$ and $\adjoint{\myvector{y}}$ 
denote the reversed sequences of $\myvector{a}$, $\myvector{w}$ and $\myvector{y}$ respectively.
\end{theorem}
Notice that Theorem~\ref{thm:ar-backprop} is special case of Theorem~\ref{thm:arma-backprop}: 
the first equation in \autoref{eq:arma-cnn-backprop} is proved by \autoref{eq:arma-model-backprop-a}, 
and the second equation is proved by \autoref{eq:arma-model-backprop-x}.

We provide two different proofs of Theorem~\ref{thm:arma-backprop}.
{\bf (1)} The analysis in our first proof is based on real numbers, 
and applicable to arbitrary types of convolution.
{\bf (2)} If the convolution is {\em circular} (as in the implementation of this paper), 
we provide a simpler proof using Fourier transform (therefore complex numbers). 
The second proof also suggests an {\FFT}-based algorithm to compute the backpropagation 
in \autoref{eq:arma-cnn-backprop} efficiently. 

\subsubsection{Proof in Real Numbers $\mathbb{R}$}
Before we prove the theorem, we first prove
a useful lemma on the inverse of transposed convolution. 

\begin{lemma} [{\bf Inverse of transposed convolution}]
Given a convolution (with coefficients) $\myvector{a}$, the operations of 
inversion and transposition are exchangeable,
\label{lem:inverse-transposed-convolution}
\begin{equation}
\inverse{\adjoint{\myvector{a}}} = \inverse{\adjoint{\myvector{a}}}
\label{eq:inverse-transposed-convolution}
\end{equation}
that is, the inverse transposed convolution is equal to the transposed inverse convolution.  
\end{lemma}
\begin{proof}
The lemma is an immediate result of the definitions of inverse and transposed convolutions.
\begin{equation}
\sum_{p = -\infty}^{+\infty} \vectorSub{\adjoint{a}}{p} \vectorSub{\adjoint{\inverse{a}}}{i - p} 
= \sum_{p = -\infty}^{+\infty} \vectorSub{a}{-p} \vectorSub{\inverse{a}}{p - i} 
= \vectorSub{\delta}{-i} = \vectorSub{\delta}{i} \quad \forall i
\end{equation}
which shows the inverse of $\adjoint{\myvector{a}}$,
i.e.\ $\inverse{\adjoint{\myvector{a}}}$, is equal to $\adjoint{\inverse{\myvector{a}}}$.
\end{proof}

Now we are ready to prove Theorem~\ref{thm:arma-backprop} at the beginning of this section.
\begin{proof}
To begin with, we write the {\ARMA} model
$\myvector{a} \ast \myvector{y} = \myvector{w} \ast \myvector{x}$ in its weighted-sum form: 
\begin{equation}
\sum_{q = -\infty}^{+\infty} \vectorSub{a}{q} \vectorSub{y}{i - q} 
= \sum_{p = -\infty}^{+\infty} \vectorSub{w}{p} \vectorSub{x}{i - p}, ~\forall i
\end{equation}
Taking derivative w.r.t.\ $\vectorSub{a}{r}$ on both sides, and since the right side
is a constant w.r.t.\ $\vectorSub{a}{r}$, we have
\begin{equation}
\derivative{\left(\sum_{q = -\infty}^{+\infty} \vectorSub{a}{q} \vectorSub{y}{q} \right)}{a_r} = 0, ~\forall i, r
\end{equation}
By {\em implicit function theorem}, the left hand side can be further expanded as
\begin{align}
&  \derivative{\left(\sum_{q = -\infty}^{+\infty} \vectorSub{a}{q} \vectorSub{y}{i - q} \right)}{\vectorSub{a}{r}}
= \sum_{q = -\infty}^{+\infty} \derivative{\left( \vectorSub{a}{q} \vectorSub{y}{i - q} \right)}{\vectorSub{a}{r}} \\
= ~ & \sum_{q \ne r} \vectorSub{a}{q} \derivative{ \vectorSub{y}{i - q} }{\vectorSub{a}{r}}
+ \left( \vectorSub{y}{i - r} + \vectorSub{a}{r} \derivative{\vectorSub{y}{i - r}}{ \vectorSub{a}{r}} \right) \\
= ~ & \sum_{q = -\infty}^{+\infty} \vectorSub{a}{q} \derivative{ \vectorSub{y}{i - q}}{ \vectorSub{a}{r}} 
+ \vectorSub{y}{i - r} = 0, ~\forall i, r
\end{align}
Rearranging the equation above, we have
\begin{subequations}
\begin{equation}
- \sum_{q = -\infty}^{+\infty} \vectorSub{a}{q} \derivative{\vectorSub{y}{i - q}}{\vectorSub{a}{r}} 
= \vectorSub{y}{i - r}, ~\forall i, r
\label{eq:arma-derivative-a}  
\end{equation}
Repeating the procedure twice for the derivatives w.r.t.\ $\vectorSub{w}{r}$ and $\vectorSub{x}{r}$, 
we have two similar equations:
\begin{align}
\sum_{q = -\infty}^{+\infty} \vectorSub{a}{q} \derivative{\vectorSub{y}{i - q}}{\vectorSub{w}{r}} 
& = \vectorSub{x}{i - r}, ~\forall i, r
\label{eq:arma-derivative-w} \\
\sum_{q = -\infty}^{+\infty} \vectorSub{a}{q} \derivative{\vectorSub{y}{i - q}}{\vectorSub{x}{r}} 
& = \vectorSub{a}{i - r}, ~\forall i, r 
\label{eq:arma-derivative-x}
\end{align}
\end{subequations}
Since \autoref{eq:arma-derivative-a}, \autoref{eq:arma-derivative-w} 
and \autoref{eq:arma-derivative-x} take the same form, 
we only precede with \autoref{eq:arma-derivative-w}
and obtain $\gradientInline{\myvector{w}}$. 
The other two can be derived using the same arguments.

Notice that \autoref{eq:arma-derivative-w} can be rewritten as 
\begin{equation}
\sum_{q = -\infty}^{+\infty} \vectorSub{a}{i - q} \derivative{\vectorSub{y}{q}}{\vectorSub{w}{r}} 
= \vectorSub{x}{i - r}, ~\forall i, r 
\label{eq:arma-derivative-w-2}
\end{equation}
by changing variable $q$ to $i - q$. 
Since \autoref{eq:arma-derivative-w-2} holds for any $i$, 
we further introduce a new index $l$ and change $i$ to $i - l$ on both hand sides:
\begin{equation}
\sum_{q = -\infty}^{+\infty} \vectorSub{a}{i - q - l} \derivative{\vectorSub{y}{q}}{\vectorSub{w}{r}} 
= \vectorSub{x}{i - r - l}, ~\forall i, r, l
\end{equation}
Now we convolve both hand sides with $\inverse{\myvector{a}}$, the inverse of $\myvector{a}$. 
Then for all $i$ and $r$, we have
\begin{align}
\sum_{l = -\infty}^{+\infty} \vectorSub{\inverse{a}}{l} 
\left( \sum_{q = -\infty}^{+\infty} \vectorSub{a}{i - q - l} \derivative{\vectorSub{y}{q}}{\vectorSub{w}{r}} \right)
& = \sum_{l = -\infty}^{+\infty} \vectorSub{\inverse{a}}{l} \vectorSub{x}{i - r- l} \\
\sum_{q = -\infty}^{+\infty} \left(\sum_{l = -\infty}^{+\infty} \vectorSub{\inverse{a}}{l} \vectorSub{a}{i - q - l} \right) 
\derivative{\vectorSub{y}{q}}{\vectorSub{w}{r}}
& = \sum_{l = -\infty}^{+\infty} \vectorSub{\inverse{a}}{l} \vectorSub{x}{i - r - l} \\
\derivative{\vectorSub{y}{i}}{\vectorSub{w}{r}}
= \sum_{q = -\infty}^{+\infty} \vectorSub{\delta}{i -q} \derivative{\vectorSub{y}{q}}{\vectorSub{w}{r}} 
& = \sum_{l = -\infty}^{+\infty} \vectorSub{\inverse{a}}{l} \vectorSub{x}{i - r - l}
\end{align}
Subsequently, we apply the chain rule to obtain $\gradientInline{\vectorSub{w}{r}}$
\begin{equation}
\gradient{w_r} = \sum_{i = -\infty}^{+\infty} \derivative{\vectorSub{y}{i}}{\vectorSub{w}{r}} \gradient{\vectorSub{y}{i}}
= \sum_{i = -\infty}^{+\infty} \sum_{l = -\infty}^{+\infty} \vectorSub{\inverse{a}}{l} \vectorSub{x}{i - r - l} \gradient{\vectorSub{y}{i}}, ~\forall r
\end{equation}
Finally, we convolve both hand sides with $\adjoint{\myvector{a}}$, 
the transpose of $\myvector{a}$, to obtain the {\ARMA} form of backpropagation rule.
\begin{align}
\sum_{r = -\infty}^{+\infty} \vectorSub{\adjoint{a}}{j - r} \gradient{\vectorSub{w}{r}}
& = \sum_{r = -\infty}^{+\infty} \vectorSub{\adjoint{a}}{j - r}
\left(\sum_{i = -\infty}^{+\infty} \sum_{l = -\infty}^{+\infty} \vectorSub{\inverse{a}}{l} 
\vectorSub{x}{i - r - l} \gradient{\vectorSub{y}{i}} \right) \\
& = \sum_{r = -\infty}^{+\infty} \sum_{i = -\infty}^{+\infty} \sum_{l = -\infty}^{+\infty} 
\vectorSub{\adjoint{a}}{j - r} \vectorSub{\inverse{a}}{l} \vectorSub{x}{i - r - l} \gradient{\vectorSub{y}{i}} \\
& = \sum_{r = -\infty}^{+\infty} \sum_{i = -\infty}^{+\infty} \sum_{l = -\infty}^{+\infty} 
\vectorSub{\adjoint{a}}{j - r} \vectorSub{\inverse{a}}{l - r} \vectorSub{x}{i - l} \gradient{\vectorSub{y}{i}} \\
& = \sum_{r = -\infty}^{+\infty} \sum_{i = -\infty}^{+\infty} \sum_{l = -\infty}^{+\infty}
\vectorSub{\adjoint{a}}{j - r} \vectorSub{\adjoint{\inverse{a}}}{r - l} \vectorSub{\adjoint{x}}{l - i} \gradient{\vectorSub{y}{i}} \\
& = \sum_{i = -\infty}^{+\infty} \sum_{l = -\infty}^{+\infty} \left(\sum_{r = -\infty}^{+\infty} 
\vectorSub{\adjoint{a}}{j-r} \vectorSub{\adjoint{\inverse{a}}}{r-l} \right) \vectorSub{\adjoint{x}}{l-i} \gradient{\vectorSub{y}{i}} \\
& = \sum_{i = -\infty}^{+\infty} \sum_{l = -\infty}^{+\infty} \vectorSub{\delta}{j - l} \vectorSub{\adjoint{x}}{l-i} \gradient{\vectorSub{y}{i}} \\
& = \sum_{i = -\infty}^{+\infty} \vectorSub{\adjoint{x}}{j - i} \gradient{\vectorSub{y}{i}}, ~\forall j
\end{align}
where the second last equality uses Lemma~\ref{lem:inverse-transposed-convolution}.
Therefore, we prove 
$\adjoint{\myvector{a}} \ast \gradientInline{\myvector{w}} = \adjoint{\myvector{x}} \ast \gradientInline{\myvector{y}}$, 
i.e.\ \autoref{eq:arma-model-backprop-w} in the theorem.
\autoref{eq:arma-model-backprop-a} and \autoref{eq:arma-model-backprop-x} can be proved similarly. 
\end{proof}

\subsubsection{Proof in Complex Numbers $\mathbb{C}$}
In this part, we provide an alternative proof of 
Theorem~\ref{thm:arma-backprop} using Fourier transform.

\begin{proof} 
If both convolutions in $\myvector{a} \ast \myvector{y} = \myvector{w} \ast \myvector{x}$ are circular with period $N$,
the celebrated {\em convolution theorem} relates the {\DFTlong} of $\myvector{a}$, $\myvector{y}$, $\myvector{w}$ and $\myvector{x}$ with
\begin{equation}
\vectorSub{A}{l} \vectorSub{Y}{l} = \vectorSub{W}{l} \vectorSub{X}{l}
\quad \left\{
	\begin{aligned} 
		\vectorSub{A}{l}  = \sum_{n=0}^{N-1} \vectorSub{a}{n} \rootunity{N}{n l}, \quad & 
		\vectorSub{Y}{l}  = \sum_{n=0}^{N-1} \vectorSub{y}{n} \rootunity{N}{n l}  \\
		\vectorSub{W}{l} = \sum_{n=0}^{N-1} \vectorSub{w}{n} \rootunity{N}{n l}, \quad & 
		\vectorSub{X}{l}  = \sum_{n=0}^{N-1} \vectorSub{x}{n} \rootunity{N}{n l} 
                \end{aligned} \right.
\end{equation}
where $\rootunity{N}{} = \exp(- \im 2 \pi / N)$ is the $N$-th root of unity.
For brevity, we only prove the most difficult equation
$ - \adjoint{\myvector{a}} \ast \gradientInline{\myvector{a}}
= \adjoint{\myvector{y}} \ast \gradientInline{\myvector{y}}$ (\autoref{eq:arma-model-backprop-a}) here,
and the proofs for the other two equations can be obtained with minor modification. 

Taking derivative w.r.t.\ $\vectorSub{A}{k}$ on both hand sides, we have
\begin{equation}
\left\{ \begin{aligned}
\vectorSub{A}{l} \derivative{\vectorSub{Y}{l}}{\vectorSub{A}{k}} & = 0, \quad l \ne k \\
\vectorSub{A}{l} \derivative{\vectorSub{Y}{l}}{\vectorSub{A}{k}} + \vectorSub{Y}{k} & = 0, \quad l = k
\end{aligned} \right.
\end{equation}
Since $\vectorSub{A}{l} \ne 0, \forall l$, the equation can be simplified as
\begin{equation}
\derivative{\vectorSub{Y}{l}}{\vectorSub{A}{k}} =
\left\{ \begin{aligned}
	0, & \quad l \ne k \\
	- \frac{\vectorSub{Y}{k}}{\vectorSub{A}{k}}, & \quad l = k
\end{aligned} \right.
\end{equation}
Then we apply chain rule to obtain the gradient of $\vectorSub{A}{k}$, which yields
\begin{equation}
\gradient{\vectorSub{A}{k}} = \sum_{l = 0}^{N - 1} \gradient{\vectorSub{Y}{l}} \derivative{\vectorSub{Y}{l}}{\vectorSub{A}{k}} 
= - \frac{\vectorSub{Y}{k}}{\vectorSub{A}{k}} \gradient{Y_k}
\end{equation}
Again, since $\vectorSub{A}{k} \ne 0, \forall k$, we can simplify the equation as
\begin{equation}
\vectorSub{A}{k} \gradient{\vectorSub{A}{k}} = - \vectorSub{Y}{k} \gradient{\vectorSub{Y}{k}}
\label{eq:backprop-complex-key}
\end{equation}
(Notice that the equation above suggests an efficient algorithm 
to evaluate the equation using {\FFT}.)
To precede, we apply the chain rule one more time to obtain the derivatives 
w.r.t.\ $\vectorSub{a}{n}$ and $\vectorSub{y}{n}$.
\begin{subequations}
\begin{align}
\gradient{\vectorSub{a}{n}} & = \sum_{k = 0}^{N - 1} 
\gradient{\vectorSub{A}{k}} \derivative{\vectorSub{A}{k}}{\vectorSub{a}{n}}
= \sum_{k = 0}^{N - 1} \gradient{\vectorSub{A}{k}} \rootunity{N}{k n} \\
\gradient{\vectorSub{y}{n}} & = \sum_{k = 0}^{N - 1} 
\gradient{\vectorSub{Y}{k}} \derivative{\vectorSub{Y}{k}}{\vectorSub{y}{n}}
= \sum_{k = 0}^{N - 1} \gradient{\vectorSub{Y}{k}} \rootunity{N}{k n} 
\end{align}
\end{subequations}
With the equations above, the convolution between 
$\adjoint{\myvector{a}}$ and $\gradientInline{\myvector{a}}$ can be rewritten as
\begin{align}
& \sum_{n = 0}^{N-1} \vectorSub{\adjoint{a}}{i - n} \gradient{\vectorSub{a}{n}} 
= \sum_{n = 0}^{N-1} \vectorSub{a}{n - i} \gradient{\vectorSub{a}{n - i}} \\
= ~ & \sum_{n = 0}^{N-1} a_{n - i} \left( \sum_{k = 0}^{N-1}
\gradient{\vectorSub{A}{k}} \rootunity{N}{kn} \right) \\
= ~ & \sum_{k = 0}^{N - 1} \left(\sum_{n = 0}^{N-1} 
\vectorSub{a}{n - i} \rootunity{N}{k(n - i)} \right) \gradient{\vectorSub{A}{k}} \rootunity{N}{ki} \\
= ~ &  \sum_{k = 0}^{N-1} \vectorSub{A}{k} \gradient{\vectorSub{A}{k}} \rootunity{N}{ki}
\end{align}
With identical arguments, we can rewrite the convolution between 
$\adjoint{\myvector{y}}$ and $\gradientInline{\myvector{y}}$ as
\begin{equation}
\sum_{n = 0}^{N - 1} \vectorSub{\adjoint{y}}{i - n} \gradient{\vectorSub{y}{n}} = 
\sum_{k = 0}^{N - 1} \vectorSub{Y}{k} \gradient{\vectorSub{Y}{k}} \rootunity{N}{ki}
\end{equation}
Recall the relation in \autoref{eq:backprop-complex-key}, we have
\begin{equation}
- \sum_{n = 0}^{N-1} \vectorSub{\adjoint{a}}{i - n} \gradient{\vectorSub{a}{n}} = 
\sum_{n = 0}^{N - 1} \vectorSub{\adjoint{y}}{i - n} \gradient{\vectorSub{y}{n}}
\end{equation}
i.e.\ $- \adjoint{\myvector{a}} \ast \gradientInline{\myvector{a}} = \adjoint{\myvector{y}} \ast \gradientInline{\myvector{y}}$, 
which completes the proof.
\end{proof}

\subsection{Efficient Computation using {\FFTLONG}}
\label{sub:fft}

The key to speeding up both forward and backward passes in \ARMA layers
is the {\em \DFTLONG} (\DFT), along with the {\em \FFTLONG} (\FFT) algorithm.

\begin{definition} [{\bf \DFTLONG, \DFT}]
\label{def:dft}
Given a third-order tensor \(\mytensor{T} \in \R^{I_1 \times I_2 \times C}\), 
we define its \DFT of over the spatial coordinates 
as \(\mytensor{\Fourier{T}} \in \C^{I_1 \times I_2 \times C}\).
\begin{equation}
\tensorSub{\Fourier{T}}{k_1, k_2, c} = \sum_{i_1 = 0}^{I_1 - 1} \sum_{i_2 = 0}^{I_2 - 1} 
\tensorSub{T}{i_1, i_2, c} ~ \rootunity{I_1}{- i_1 k_1} ~ \rootunity{I_2}{- i_2 k_2}
\label{eq:arma-cnn-dft}
\end{equation}
where \(\rootunity{I}{} = \exp({2\pi}/{I})\) is the \(I^{th}\) root of unity. 
Given the transformed tensor \(\Fourier{\mytensor{T}} \in \C^{I_1 \times I_2 \times C}\), 
the original tensor \(\mytensor{T}\) can be recovered by inverse {\DFT} (\IDFT) as
\begin{equation}
\tensorSub{T}{i_1, i_2, c} = \frac{1}{I_1 I_2} \sum_{k_1 = 0}^{I_1 - 1} \sum_{k_2 = 0}^{I_2 - 1} 
\tensorSub{\Fourier{T}}{i_1, i_2, c} \rootunity{I_1}{i_1 k_1} \rootunity{I_2}{i_2 k_2}
\label{eq:arma-cnn-idft}
\end{equation}
\end{definition}
The definition above can be extended to convolutional kernels \(\mytensor{A}\)
by first zero-padding \(\mytensor{A}\) to be \(\R^{I_1 \times I_2 \times C}\).
With {\DFT}, the \ARlong layer in \autoref{eq:arma-cnn-steps} can be computed as
\begin{equation}
\tensorSub{\Fourier{A}}{k_1, k_2, t} \tensorSub{\Fourier{Y}}{k_1, k_2, t}
 = \tensorSub{\Fourier{T}}{k_1, k_2, t}
\label{eq:arma-cnn-step-2-fourier}
\end{equation}
where \(\mytensor{\Fourier{A}}, \mytensor{\Fourier{T}}\) are computed from 
\(\mytensor{A}, \mytensor{T}\) with \autoref{eq:arma-cnn-dft},
and \(\mytensor{Y}\) is recovered from \(\mytensor{\Fourier{Y}}\) by \autoref{eq:arma-cnn-idft}.
Similarly, the backpropagation in \autoref{eq:arma-cnn-backprop} can be solved as
\begin{subequations}
\begin{align}
\gradient{\tensorSub{\Fourier{A}}{k_1, k_2, t}}
& = - \frac{\tensorSub{\Fourier{Y}}{k_1, k_2, t}}{\tensorSub{\Fourier{A}}{k_1, k_2, t}}
\cdot \gradient{\tensorSub{\Fourier{Y}}{k_1, k_2, t}} \\
\gradient{\tensorSub{\Fourier{A}}{k_1, k_2, t}} & = \frac{1}{\tensorSub{\Fourier{A}}{k_1, k_2, t}} \cdot \gradient{\tensorSub{\Fourier{Y}}{k_1, k_2, t}}
\end{align}
\end{subequations}
If every {\DFT} is evaluated using {\FFT}, the computational complexity of either forward or backward pass reduces to  $O(\log(\max(I_1, I_2)) I_1 I_2 T)$, compared to $O((I_1^2 + I_2^2) I_1I_2 T)$ using Gaussian elimination. 
\section{Stability of {\ARMA} Layers}
\label{app:stability}

In this section, we will prove the main Theorem~\ref{thm:stability-arma} in \autoref{sec:stability}.
The section is organized in three subsections: 
{\bf (1)} In \autoref{sub:stability-algebra}, 
we formally define the concept of {\em \BIBO stability}, 
and prove a lemma that relates the stability of 
a complicated model to the ones of its submodules;
{\bf (2)} In \autoref{sub:arma-decomposition}, we repeatedly apply the lemma and deduce
the stability of an {\ARMA} layer to from the one of length-$3$ filters;
{\bf (3)} Lastly in \autoref{sub:stabillity-2nd}, we prove a theorem on the stability of a length-$3$ filter.

\subsection{Algebra of \BIBO stability}
\label{sub:stability-algebra}

To analyze the stability of an {\ARMA} model, we adopt the traditional notion of 
{\em \BIBOLONG} (\BIBO) stability~\citep{oppenheim2014discrete} 
that characterizes stability of linear systems.

\begin{definition} [{\bf \BIBO stability}]
\label{def:BIBO-stability} 
An input \( \myvector{x} \) (or an output \(  \)) is bounded if 
\(| \vectorSub{x}{i} | < B_1, \forall i \in \Z\) for some \(B_1 > 0\) 
(or \(| \vectorSub{y}{i} < B_2, \forall i \in \Z\) for some \(B_2 > 0\)).
A model is \BIBO stable if the output \( \myvector{y} \) is bounded 
given any bounded input \( \myvector{x} \), that is
\begin{equation}
\forall \myvector{x}, ~ (\exists B_1 > 0, | \vectorSub{x}{i} | < B_1, \forall i \in \Z)
\implies (\exists B_2 > 0, | \vectorSub{y}{i} | < B_2, \forall i \in \Z)  
\label{eq:BIBO-stability}
\end{equation}
\end{definition}

The following lemma presents that the \BIBO stability is preserved under 
simple algebraic operations of {\em cascade}, {\em addition} and {\em concatenation}. 
This lemma allows us to reduce the stability analysis of a complex model into its simpler submodules.

\begin{lemma} [\bf{Preserved \BIBO Stability}]
\label{lem:stability-algebra}
\BIBO stability is preserved under the operations of cascade, addition and concatenation.
Suppose $f$ and $g$ are two \BIBO stable models, 
and consider three compound models:
{\bf (1)} $h_1 = g \circ f$ is a cascaded model
$\myvector{y} = h_1(\myvector{x}) = g(f(\myvector{x}))$,
{\bf (2)} $h_2 = f + g$ is a parallel model
$\myvector{y} = h_2(\myvector{x}) = f(\myvector{x}) + g(\myvector{x})$, 
{\bf (3)} $h_3 = f \otimes g$ is a concatenated model
$\myvector{y} = [ \vectorSup{y}{1}, \vectorSup{y}{2} ] = 
h_3([\vectorSup{x}{1}, \vectorSup{x}{2}]) = [f(\vectorSup{x}{1}), g(\vectorSup{x}{2})]$,
 $h_1$, $h_2$ and $h_3$ are all \BIBO stable.
\end{lemma}

\begin{proof}
(1) {\em Cascaded model $h_1 = g \circ f$: $y = h_1(x) = f(g(x))$}. 
Let $t = h(x)$ denote the intermediate result returned by the model $f$. 
Since $f$ is \BIBO stable, we have
\begin{subequations}
\begin{equation}
(\exists B_1 > 0, | \vectorSub{x}{i} | < B_1, \forall i \in \Z)
\implies (\exists B_0 > 0, | \vectorSub{t}{i} | < B_0, \forall i \in \Z)  
\label{eq:BIBO-cascade-step-1}
\end{equation}
Similarly, since $g$ is \BIBO stable, we further have
\begin{equation}
(\exists B_0 > 0, | \vectorSub{t}{i} | < B_0, \forall i \in \Z)
\implies (\exists B_2 > 0, | \vectorSub{y}{i} | < B_2, \forall i \in \Z)  
\label{eq:BIBO-cascade-step-2}
\end{equation}
\end{subequations}
Combining both \autoref{eq:BIBO-cascade-step-1} 
and \autoref{eq:BIBO-cascade-step-2}, we achieve 
\begin{equation}
(\exists B_1 > 0, | \vectorSub{t}{i} | < B_1, \forall i \in \Z)
\implies (\exists B_2 > 0, | \vectorSub{y}{i} | < B_2, \forall i \in \Z)  
\label{eq:BIBO-cascade}
\end{equation}
which is the definition of \BIBO stability for model $h_1$.

{\em (2) Parallel model $h_2 = f + g$: 
$\myvector{y} = h_2(\myvector{x}) = f(\myvector{x}) + g(\myvector{x})$.}
Let $\myvector{u} = f(\myvector{x})$ and $\myvector{v} = g(\myvector{x})$ 
be the outputs of $f$ and $g$.
Since both $f$ and $g$ are \BIBO stable, we have the following two relations:
\begin{subequations}
\begin{gather}
(\exists B_1 > 0, | \vectorSub{x}{i} | < B_1, \forall i \in \Z)
\implies (\exists B_{21} > 0, | \vectorSub{u}{i} | < B_{21}, \forall i \in \Z)  
\label{eq:BIBO-parallel-1} \\
(\exists B_1 > 0, | \vectorSub{x}{i} | < B_1, \forall i \in \Z)
\implies (\exists B_{22} > 0, | \vectorSub{v}{i} | < B_{22}, \forall i \in \Z)  
\label{eq:BIBO-parallel-2}
\end{gather}
\end{subequations}
Combining both \autoref{eq:BIBO-parallel-1}
and \autoref{eq:BIBO-parallel-2}, we have 
\begin{equation}
(\exists B_1 > 0, | \vectorSub{t}{i} | < B_0, \forall i \in \Z)
\implies (| \vectorSub{y}{i} | < B_2 = B_{21} + B_{22}, \forall i \in \Z)  
\label{eq:BIBO-parallel}
\end{equation}
We achieve the definition \BIBO stability for model $h_2$.

{\em (3) Concatenated model $\myvector{y} = \myvector{f} \otimes \myvector{g}$: 
$\myvector{y} = [ \vectorSup{y}{1}, \vectorSup{y}{2}] 
= h([\vectorSup{x}{1}, \vectorSup{x}{2}]) = [f(\vectorSup{x}{1}), g(\vectorSup{x}{2})]$:} 
Since $f$ and $g$ are both \BIBO stable, we have the following relations:
\begin{subequations}
\begin{gather}
(\exists B_1 > 0, | \vectorSub{x}{i} | < B_1, \forall i \in \Z)
\implies (\exists B_{21} > 0, | \vectorInd{y}{1}{i} | < B_{21}, \forall i \in \Z)  
\label{eq:BIBO-concat-1} \\
(\exists B_1 > 0, | \vectorSub{x}{i} | < B_1, \forall i \in \Z)
\implies (\exists B_{22} > 0, | \vectorInd{y}{2}{i} | < B_{22}, \forall i \in \Z)  
\label{eq:BIBO-concat-2}
\end{gather}
\end{subequations}
Again, combining both equations we have
\begin{equation}
(\exists B_1 > 0, | \vectorSub{x}{i} | < B_1, \forall i \in \Z)
\implies (| \vectorInd{y}{2}{i} | < B_2 = \max(B_{21}, B_{22}), \forall i \in \Z)  
\label{eq:BIBO-concat}
\end{equation}
And we achieve the \BIBO stability for model $h_3$. 
\end{proof}

\subsection{Reduction of an {\ARMA} layer}
\label{sub:arma-decomposition}

In what follows, we repeatedly use Lemma~\ref{lem:stability-algebra} 
to decompose an {\ARMA} layer into simpler submodules 
until the stability analysis for the submodule is tractable.

{\bf From {\ARMA} model to {\AR} model.}
In \autoref{sec:arma-computation}, we show that an
 {\ARMA} layer can be decomposed into 
a {\em cascade} of a {\em traditional convolutional layer}
and an {\em \ARlong layer} in \autoref{eq:arma-cnn-steps}.
Since the traditional convolutional layer
is always \BIBO stable,
it is sufficient to guarantee the stability of the \ARlong layer:
\begin{equation}
\tensorSub{A}{\bm{:}, \bm{:}, t} \ast \tensorSub{Y}{\bm{:}, \bm{:}, t} 
= \tensorSub{T}{\bm{:}, \bm{:}, t}, ~\forall t
\label{eq:arma-cnn-step-2-app}
\end{equation}

{\bf From multiple channels to a single channel.} 
Note that the \ARlong layer in \autoref{eq:arma-cnn-step-2-app} 
is a {\em concatenation} of $T$ channels of \ARMA models, 
therefore it is sufficient to guarantee the stability of each \ARMA model.
For simplicity, we drop the subscript $t$ and denote $\mytensor{A}, \mytensor{Y}, \mytensor{T}$
as $\mymatrix{A}, \mymatrix{Y}, \mymatrix{T}$.
Our goal now reduces to finding a sufficient condition for the stability of
\begin{equation}
\mymatrix{A} \ast \mymatrix{Y} = \mymatrix{T} ~ \Longleftrightarrow ~
\sum_{p_1, p_2} \matrixSub{A}{p_1, p_2} \matrixSub{Y}{i_1 - p_1, i_2 - p_2}
= \matrixSub{T}{i_1, i_2}, ~\forall i_1, i_2
\label{eq:arma-cnn-step-2-channel}
\end{equation}
where $\mymatrix{A} \in \R^{K_a \times K_a}$
and $\mymatrix{T}, \mymatrix{Y} \in \R^{I_1 \times I_2}$.

{\bf From separable 2D-filter to two 1D-filters.}
In a {\em separable {\ARMA} layer} (\autoref{eq:separable-arma-cnn}),
each filter $\mymatrix{A}$ in \autoref{eq:arma-cnn-step-2-channel} is separable, 
i.e.\ $\mymatrix{A} = \myvector{f} \otimes \myvector{g}$ is outer product of two
1D-filters $\myvector{f} \in \R^{I_1}, \myvector{g} \in \R^{I_2}$: 
\begin{equation}
\matrixSub{A}{p_1, p_2} = \vectorSub{f}{p_1} \vectorSub{g}{p_2}, ~\forall p_1, p_2
\label{eq:separable-filter}
\end{equation}
Given the factorization, the model in \autoref{eq:arma-cnn-step-2-channel}
can be written as a cascade of two submodules: 
\begin{subequations}
\begin{align}
\sum_{p_1} \vectorSub{f}{p_1} \matrixSub{S}{i_1 - p_1, i_2} & = \matrixSub{T}{i_1, i_2}, ~\forall i_2
\label{eq:separable-fitlers-step-1} \\
\sum_{p_2} \vectorSub{g}{p_2} \matrixSub{Y}{i_1, i_2 - p_2} & = \matrixSub{S}{i_1, i_2}, ~\forall i_1
\label{eq:separable-filters-step-2}
\end{align}
\end{subequations}
where $\mymatrix{S} \in \R^{I_1 \times I_2}$ is an intermediate result.
Notice that \autoref{eq:separable-fitlers-step-1} is a concatenation of $I_2$ submodules, 
each of which operates on a column of $\mymatrix{T}$.
Similarly, \autoref{eq:separable-filters-step-2} can be decomposed into a concatenation of $I_1$ submodules, 
and each submodule operates on a row of $\mymatrix{S}$.
According to Lemma~\ref{lem:stability-algebra}, 
it is sufficient to guarantee the stability of $\myvector{f}$ and $\myvector{g}$ individually. 
For simplicity, we denote both $\myvector{f}$ and $\myvector{g}$ as $\myvector{a}$, 
and rewrite each submodule in \autoref{eq:separable-fitlers-step-1} or \autoref{eq:separable-filters-step-2} as
\begin{equation}
\myvector{a} \ast \myvector{y} = \myvector{x} ~ \Longleftrightarrow ~ 
\sum_{p} \vectorSub{a}{p} \vectorSub{y}{i - p} = \vectorSub{x}{i}, \forall i
\label{eq:arma-cnn-1d}
\end{equation}

{\bf From general 1D-filter to composition of length-3 filters. }
By {\em the fundamental theorem of algebra}, 
any one-dimensional filter can be decomposed
as a composition of shorter filters~\citep{oppenheim2014discrete}.
Specifically, suppose $\myvector{a} \in \R^{K}$ is a filter of length-$K$,
it can be factorized into a composition of $Q = (K - 1)/2$ length-$3$ filters such that
\begin{equation}
\myvector{a} = \vectorSup{a}{1} \ast \vectorSup{a}{2} \cdots \ast \vectorSup{a}{Q}
\label{eq:higher-to-second}
\end{equation}
where each filter $\vectorSup{a}{q} \in \R^{3}$ has three coefficients.
By the decomposition, the model in \autoref{eq:arma-cnn-1d} is a cascade of $Q$ submodules
\begin{equation}
\vectorSup{a}{1} \ast \left( \vectorSup{a}{2} \ast \cdots \left( \vectorSup{a}{Q} \ast \myvector{y} \right) \right) = \myvector{x}
\label{eq:higher-to-second-steps}
\end{equation}
Therefore, we only need to guarantee the stability for each $\vectorSup{a}{q}$ individually.
In the next subsection, we will further drop the superscript $q$ and assume $\myvector{a}$ itself is a length-3 filter. 

\subsection{Stability of a length-3 1D-filter}
\label{sub:stabillity-2nd}

Without loss of generality, we assume the filter 
$\myvector{a}$ is centered at $0$ with $\vectorSub{a}{0} = 1$
(otherwise we can rescale the moving-average coefficients). 
The model at consideration can be written as
\begin{equation}
\vectorSub{a}{1} \vectorSub{y}{i - 1} + \vectorSub{y}{i} + \vectorSub{a}{-1} \vectorSub{y}{i + 1} = \vectorSub{x}{i}
\label{eq:arma-cnn-1d-2}
\end{equation}
 
The stability analysis of this model follows the standard approach of 
Z-transform~\citep{oppenheim2014discrete}.
To begin with, we review the concepts of {\em Z-transform},
{\em \ROCLONG} (\ROC) and their relationships to \BIBO stability 
of a linear model.

\begin{definition} [{\bf Z-transform and \ROC}]
\label{def:z-transform}
Given a one-dimensional sequence $\myvector{h}$, 
the Z-transform maps the sequence to a complex function on the complex plain $\C$
\begin{equation}
H(z) = \sum_{i = -\infty}^{+\infty} \vectorSub{h}{i} z^{-i}
\label{eq:z-transform}
\end{equation}
Notice that the infinite series does not necessarily converge for any $z \in \C$, 
and the transformation exists only if the summation is convergent.
The region in the complex plane that the Z-transform exists is known as 
the \ROC for the sequence $\myvector{h}$.
\end{definition}

\begin{lemma} [{\bf \ROC and \BIBO stability}]
\label{lem:ROC-BIBO}
Consider a linear model $\myvector{y} = \myvector{h} \ast \myvector{x}$, 
and let $H$ denote the Z-transform of $\myvector{h}$,
then a necessary and sufficient condition for the model being \BIBO stable is that 
the unit circle belongs to the \ROC, i.e.\ the infinite series
\begin{equation}
H(e^{\im \omega}) = \sum_{i = -\infty}^{+\infty} \vectorSub{h}{i} e^{-\im \omega i}
\label{eq:discrete-time-fourier-transform}
\end{equation}
converges for any frequency $\omega \in \R$,
i.e.\ {\em \DTFTlong} (\DTFT) exists for $\myvector{h}$.
\end{lemma}

\begin{lemma} [{\bf \ROC of length-3 \AR model}]
\label{lem:ROC-AR}
Consider an length-3 \AR model $\myvector{a} \ast \myvector{y} = \myvector{x}$, i.e.\
$\vectorSub{a}{-1} \vectorSub{y}{i - 1} + \vectorSub{y}{i} + \vectorSub{a}{1} \vectorSub{y}{i - 1} = x_i$,
the Z-transform of $\myvector$ is a length-3 complex polynomial
$A(z) = \vectorSub{a}{-1} z + 1 + \vectorSub{a}{1} z^{-1}$ with two zeros $z_1$ and $z_2$. 
Then the Z-transform of its inverse convolution $\inverse{\myvector{a}}$ is
\begin{equation}
\inverse{A}(z) = \frac{1}{A(z)} = \frac{z}{\vectorSub{a}{-1} z^2 + z + \vectorSub{a}{1}}
\end{equation}
with the corresponding \ROC $|z_1| < z < |z_2|$ as a ring.
Since the model can be written as $\myvector{y} = \inverse{\myvector{a}} \ast \myvector{x}$, 
it is \BIBO stable if $|z_1| < 1 < |z_2|$ according to Lemma~\ref{lem:ROC-BIBO}.
\end{lemma}

With the lemmas above, we are ready to prove Theorem~\ref{thm:stability-arma}.
\begin{proof}
Since the coefficients in $\myvector{a}$ are real numbers, 
the zeros of $F(z) = z A(z) = \vectorSub{a}{-1} z^2 + z + \vectorSub{a}{1}$ are conjugate to each other:
{\bf(1)} Both zeros lie on the real axis, i.e.\ $z_1$ and $z_2$ are real numbers; 
and {\bf(2)} $z_1$ and $z_2$ are complex conjugate to each other, i.e.\ $z_1^{*} = z_2$. 

Notice that {\bf(2)} also implies $|z_1| = |z_2|$. However, Lemma~\ref{lem:ROC-AR} shows that 
$|z_1| < 1 < |z_2|$ is required for \BIBO stability, and therefore the second distribution is not feasible.

If both zeros are real, the inequality $|z_1| < 1 < |z_2|$ is equivalent to 
$F(1) \cdot F(-1) < 0$, i.e.\
\begin{gather}
(\vectorSub{a}{-1} + 1 + \vectorSub{a}{1}) (\vectorSub{a}{-1} - 1 + \vectorSub{a}{1}) < 0 \\
(\vectorSub{a}{-1} + \vectorSub{a}{1})^ 2 - 1 < 0 \implies |\vectorSub{a}{-1} + \vectorSub{a}{1}| < 1
\end{gather}
which completes the proof.
\end{proof}

The constrain $|\vectorSub{a}{-1} + \vectorSub{a}{1}| < 1$ can be removed 
by re-parameterizing $(\vectorSub{a}{-1}, \vectorSub{a}{1})$ into \((\alpha, \beta)\):
\begin{equation}
\begin{pmatrix} \vectorSub{a}{-1} \\ \vectorSub{a}{1} \end{pmatrix} = 
\begin{pmatrix} \sqrt{2} / 2 & -\sqrt{2} / 2 \\ \sqrt{2} / 2 & \sqrt{2} / 2 \end{pmatrix} 
\begin{pmatrix} \alpha \\ \tanh(\beta) \end{pmatrix}
\label{eq:re-parameterization-app}
\end{equation}
where the learnable parameters \((\alpha, \beta)\) have no constrain. 
The transform in re-parameterization is illustrated in the following figure.

\begin{figure}[hbtp]
\centering
	\includegraphics[width=0.8\linewidth]{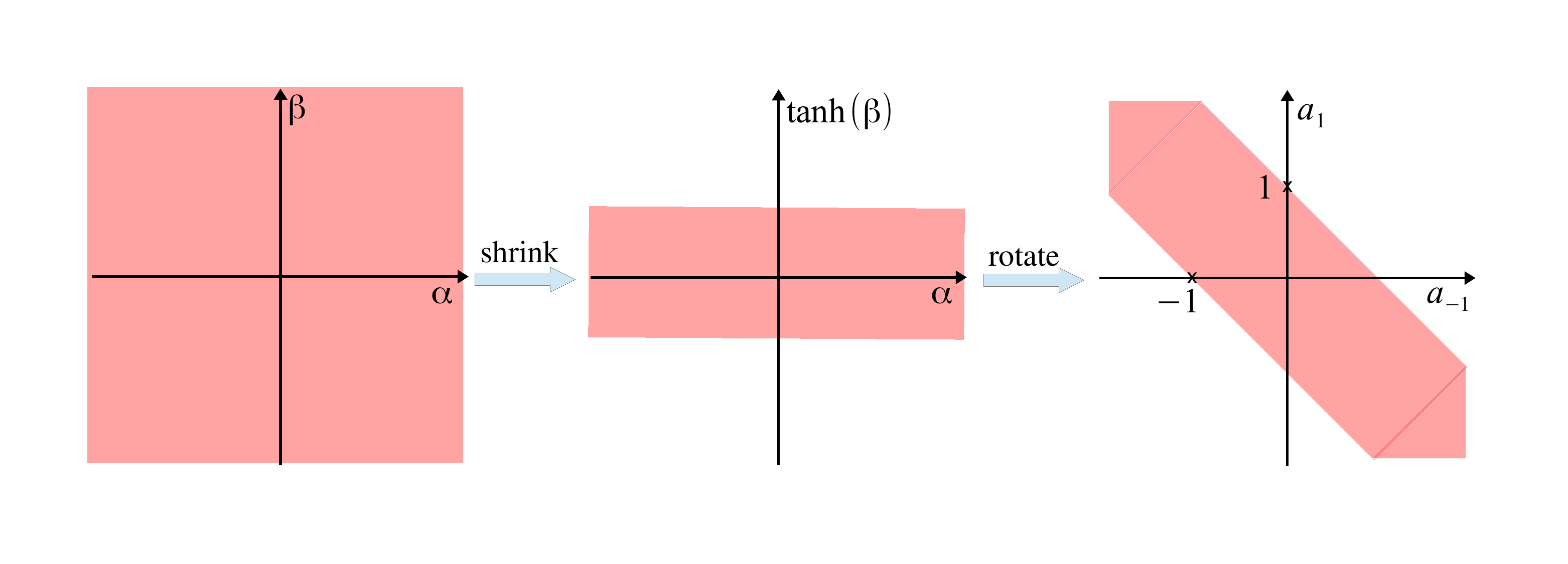}
	\caption{Visualization of the re-parameterization in \autoref{eq:re-parameterization-app}.}
\label{fig:stability}
\end{figure}

\end{document}